\documentclass[accepted]{uai2025}

\usepackage[american]{babel}

\usepackage{natbib}
\usepackage{mathtools}
\usepackage{booktabs}
\usepackage{tikz}
\usepackage[utf8]{inputenc}
\usepackage[T1]{fontenc}
\usepackage{amsthm}
\usepackage{amsfonts}
\usepackage{color}
\usepackage{mathtools}
\usepackage{mathrsfs}
\usepackage{hyperref}
\usepackage{booktabs}
\usepackage{enumitem}
\usepackage{nicefrac}
\usepackage{microtype}
\usepackage{xcolor}
\usepackage{comment}
\usepackage{bbm}
\usepackage{svg}
\usepackage{algorithm}
\usepackage{algorithmic}
\usepackage{enumitem}
\usepackage{times}
\usepackage{dsfont}
\usepackage{caption,color}
\usepackage{subcaption}
\usepackage{ulem}
\usepackage{units}

\newcommand{\cA}{\mathcal{A}}
\newcommand{\cB}{\mathcal{B}}
\newcommand{\cC}{\mathcal{C}}

\newcommand{\cF}{\mathcal{F}}
\newcommand{\cG}{\mathcal{G}}

\newcommand{\cM}{\mathcal{M}}
\newcommand{\cN}{\mathcal{N}}
\newcommand{\cO}{\mathcal{O}}
\newcommand{\ctO}{\tilde{\mathcal{O}}}
\newcommand{\cP}{\mathcal{P}}
\newcommand{\cQ}{\mathcal{Q}}
\newcommand{\cR}{\mathcal{R}}
\newcommand{\cS}{\mathcal{S}}
\newcommand{\cT}{\mathcal{T}}

\newcommand{\cY}{\mathcal{Y}}
\newcommand{\cZ}{\mathcal{Z}}

\newcommand{\bE}{\mathbb{E}}
\newcommand{\bN}{\mathbb{N}}
\newcommand{\bP}{\mathbb{P}}

\newcommand{\bR}{\mathbb{R}}

\newcommand{\ust}{^{\star}}
\newcommand{\lst}{_{\star}}
\newcommand{\uc}[1]{^{(#1)}}
\newcommand{\up}{^{\prime}}
\newcommand{\upp}{^{\prime\prime}}

\newcommand{\te}{\theta}

\newcommand{\eps}{\epsilon}
\newcommand{\teps}{\tilde{\epsilon}}
\newcommand{\lm}{\lambda}

\newcommand{\ts}{\tilde{s}}
\newcommand{\ta}{\tilde{a}}

\newcommand{\algo}{\texttt{ZoRL}}

\newcommand{\deff}{d_{\text{eff.}}}

\newcommand{\norm}[1]{\left\lVert#1\right\rVert}
\newcommand{\abs}[1]{\left|#1\right|}

\newcommand{\paren}[1]{\left(#1\right)}
\newcommand{\br}[1]{\left(#1\right)}
\newcommand{\flbr}[1]{\left\{#1\right\}}
\newcommand{\sqbr}[1]{\left[#1\right]}
\newcommand{\ceil}[1]{\left\lceil#1\right\rceil}
\newcommand{\floor}[1]{\left\lfloor#1\right\rfloor}

\newcommand{\ovl}[1]{\mkern 1.5mu\overline{\mkern-1.5mu#1\mkern-1.5mu}\mkern 1.5mu}

\newcommand{\diam}[2]{\mbox{diam}_{#1}(#2)}
\newcommand{\pdiam}[2]{\widetilde{\mbox{diam}}_{#1}(#2)}
\newcommand{\diamc}[1]{\mbox{diam}\br{#1}}
\newcommand{\spn}[1]{~\textit{sp}\paren{#1}}
\newcommand{\ind}[1]{\mathbbm{1}_{\flbr{#1}}}
\newcommand{\rel}[2]{\textit{Rel}_{#1}{\br{#2}}}

\newcommand{\gap}[1]{\text{gap}\br{#1}}
\newcommand{\evi}{\texttt{EVI}}
\newcommand{\epe}{\texttt{EPE}}

\newcommand{\inv}{^{-1}}

\newcommand{\al}[1]{\begin{align}#1\end{align}}
\newcommand{\nal}[1]{\begin{align*}#1\end{align*}}

\newtheorem{thm}{Theorem}[section]
\newtheorem{lemma}[thm]{Lemma}
\newtheorem{prop}[thm]{Proposition}
\newtheorem{cor}[thm]{Corollary}

\newtheorem{defn}[thm]{Definition}

\newtheorem{assum}[thm]{Assumption}

\newtheorem*{remark}{Remark}

\title{Provably Adaptive Average Reward Reinforcement Learning for Metric Spaces}

\author[1]{\href{mailto:<avikkar@iisc.ac.in>?Subject=Your UAI 2025 paper}{Avik Kar}}
\author[1]{\href{mailto:<rahulsingh@iisc.ac.in>?Subject=Your UAI 2025 paper}{Rahul Singh}}

\affil[1]{
    Department of Electrical Communication Engineering\\
    Indian Institute of Science\\
    Bengaluru
}
  
\begin{document}
\maketitle

\begin{abstract}
    We study infinite-horizon average-reward reinforcement learning (RL) for Lipschitz MDPs, a broad class that subsumes several important classes such as linear and RKHS MDPs, function approximation frameworks, and develop an adaptive algorithm $\text{ZoRL}$ with regret bounded as $\mathcal{O}\big(T^{1 - d_{\text{eff.}}^{-1}}\big)$, where $d_{\text{eff.}}= 2d_\mathcal{S} + d_z + 3$, $d_\mathcal{S}$ is the dimension of the state space and $d_z$ is the zooming dimension. In contrast, algorithms with fixed discretization yield $d_{\text{eff.}} = 2(d_\mathcal{S} + d_\mathcal{A}) + 2$, $d_\mathcal{A}$ being the dimension of action space. $\text{ZoRL}$ achieves this by discretizing the state-action space adaptively and zooming into ''promising regions'' of the state-action space. $d_z$, a problem-dependent quantity bounded by the state-action space's dimension, allows us to conclude that if an MDP is benign, then the regret of $\text{ZoRL}$ will be small. The zooming dimension and $\text{ZoRL}$ are truly adaptive, i.e., the current work shows how to capture adaptivity gains for infinite-horizon average-reward RL. $\text{ZoRL}$ outperforms other state-of-the-art algorithms in experiments, thereby demonstrating the gains arising due to adaptivity.
\end{abstract}

\section{Introduction}\label{sec:intro}
Reinforcement Learning (RL)~\citep{sutton2018reinforcement} is a popular model for systems involving real-time sequential decision-making and has applications in many fields such as robotics, natural language processing~\citep{ibarz2021train,sodhi2023effectiveness}.~An agent interacts sequentially with an environment by applying actions and gathers rewards.~The environment is modeled as a Markov decision process (MDP)~\citep{puterman2014markov}, its transition probabilities are not known to the agent.~Its goal is to choose actions sequentially so as to maximize the cumulative rewards.

The current work develops an RL algorithm for infinite-horizon average reward Lipschitz MDPs on metric spaces.~Popular frameworks such as tabular and linear MDPs that have been well-studied in detail in RL literature, are not suitable for real-world applications since these typically involve nonlinear systems that reside on continuous spaces~\citep{kumar2021rma}.~For continuous spaces,~the learning regret could grow linearly with time horizon $T$ unless the problem has some structure~\citep{kleinberg2008multi}.~Hence, we focus on Lipschitz MDPs, which is a very general class and subsumes several popular classes such as linear MDPs~\citep{jin2020provably}, RKHS MDPs~\citep{chowdhury2019online}, linear mixture models, RKHS approximation, and the nonlinear function approximation framework~\citep{osband2014model,kakade2020information}.~See~\cite{maran2024no,maran2024projection} for more details. 

Throughout, we use $d_{\cS},d_{\cA}$ to denote the dimensions of the state-space and the action-space respectively, and $d:=d_{\cS}+d_{\cA}$.~In episodic RL for Lipschitz MDPs, the regret is known to scale as $\ctO\big(K^{1 - \deff\inv}\big)$\footnote{$\ctO$ suppresses poly-logarithmic dependence in $K$ or $T$.}, where $K$ is the number of episodes, while $\deff$ is the effective dimension associated with the \textit{underlying MDP} and also importantly the \textit{algorithm}.~A naive algorithm that uses a fixed discretization has $\deff=d+2$~\citep{song2019efficient}. One can use problem structure to reduce $\deff$; prior works on episodic Lipschitz MDPs such as~\citet{sinclair2019adaptive, cao2020provably} reduce effective dimension to $d_z+2$, where the zooming dimension $d_z$ measures the size of the near-optimal state-action pairs.~These gains are achieved by performing an adaptive discretization of the state-action space and ``zooming in'' to only the promising regions of the state-action space by creating a finer grid around these as time progresses.~However,~\cite{kar2024adaptive} show that zooming technique and algorithms developed for episodic MDPs are inappropriate for average reward RL tasks, in that $d_z\to d$ as $T\to\infty$, which is what one would have obtained via a naive fixed discretization scheme.~\cite{kar2024adaptive} derives an $\cO(\eps^{2 d_\cS + d^\eps_z + 1} \log{T})$ upper-bound on the regret with respect to an $\eps$~suboptimal comparator policy class, where $d^\eps_z$ is the ``$\eps$-zooming dimension'' and satisfies $d^\eps_z \leq d$. However, $d_z^\eps \to d$ in the limit $\eps \downarrow 0$, which shows that no adaptivity gains are achieved if the policy class contains optimal policy, i.e., one wants to attain optimal performance.~In a later version of the same paper,~\cite{kar2024policy} rectifies this issue to some extent by competing against an optimal policy class. They work directly in the policy space, and show zooming behavior in this space rather than the state-action space, i.e., their algorithm ``activates'' more number of policies from the near-optimal regions in the policy space.~They obtain $\deff = d^\Phi_z + 2$, where $d^\Phi_z$ measures the size of near-optimal policies in the set of policies $\Phi$ that can be chosen.~$d^\Phi_z$ is the log-covering number of the set consisting of $(\beta, 2\beta]$-suboptimal policies in $\Phi$.~However, $d^\Phi_z$ can be prohibitively large if either the MDP or the policy-set $\Phi$ is not structured, since it involves coverings in function spaces~\citep{guntuboyina2012l1}. The current work remedies this and upper-bounds the regret in terms of an alternative notion of zooming dimension, one that can be bounded by $d$ in the worst case.~Though the analysis of our algorithm is performed in the policy space, it relates the suboptimality of a policy with that of the associated state-action pairs, thereby deriving an upper-bound of the number of plays of suboptimal policies in terms of coverings of the state-action space.

\subsection{Contributions}
\label{subsec:contribution}
We propose a computationally efficient algorithm~\algo~for Lipschitz MDPs in the infinite-horizon average reward RL setup.~\algo~combines adaptive discretization with the principle of optimism and yields zooming behavior.~We provide a regret upper-bound of~\algo~as a function of the zooming dimension $d_z$, where $d_z$ is defined in terms of the suboptimality gap of the state action pairs~\eqref{def:subgap}. We show that the regret of~\algo~is upper-bounded as $\ctO\big(T^{1 - \deff\inv}\big)$, where $\deff = 2 d_\cS + d_z + 3$, and $d_z \le d$. In order to attain a low $\deff$, we had to overcome several challenges.~These are discussed in detail below.
\begin{enumerate}
    \item \textit{Bypassing Policy Covers}:~As is discussed above, working with policy coverings could lead to a large $\deff$.~Let $\Phi\uc{\beta}$ denote the set of all $(\beta, 2\beta]$-suboptimal policies.~By establishing an upper-bound on the total number of plays of $\Phi\uc{\beta}$ in terms of the $\beta$-covering number of the set of all $\beta$-suboptimal state-action pairs, the current work attains a small $\deff$.~Our proof hinges on the existence of certain ``key cells.''~More specifically, we show that whenever~\algo~plays a suboptimal policy $\phi$, there exists a ball in the state-action space that satisfies the following two properties: 
    (i) it has not been visited sufficiently many times, and (ii) the stationary measure under $\phi$ assigns a large probability mass to it.~Such a ball is called a ``key cell'' for that particular episode, see Fig.~\ref{fig:keycell}.~Lemma~\ref{lem:gap_phi} unveils a relation between the suboptimality of a policy, and the suboptimality gap of the state-action pairs through which this policy passes. This result plays a crucial role in proving the existence of key cells. We derive an upper-bound on the number of plays of a cell during which it is a key cell and policies from $\Phi\uc{\beta}$ are played; here $\beta$ can be chosen from $(0,1]$. This upper-bound helps us to express the regret in terms of a covering of a state-action space, which yields a bound that depends upon the zooming dimension~\eqref{def:zoomingdim}.
    
    \item \textit{Adaptive Episode Durations}: In order to attain $\deff = 2 d_{\cS} + d_z + 3$, we have to ensure that with a high probability, the key cells are visited at least a certain number of times in each episode.~This is achieved by choosing the episode durations as a function of the ``proxy diameter'' of the policy that is played currently. We note that the popular approaches for choosing episode duration, such as ending the episode upon doubling the number of visits to any cell, would fail to yield $\deff = 2 d_{\cS} + d_z + 3$.
\end{enumerate}
We verify the gains of \algo~over both popular fixed discretization-based algorithms and existing adaptive discretization-based algorithms through simulation experiments.
\subsection{Past Works}
\textit{\underline{Lipschitz Bandits}}: The idea of zooming was first proposed in ~\citep{kleinberg2008multi} for Lipschitz multi-armed bandits.~\citet{bubeck2011x} proposed a similar idea that uses a hierarchical partition of the arm space to perform adaptive discretization.

\textit{\underline{Lipschitz MDPs}}:~\citet{domingues2021kernel} uses smoothing kernels in order to construct model estimates and obtain $\ctO\Big(H^3 K^{1 - (2d+1)\inv}\Big)$ regret.~Provable gains arising due to adaptive discretization and zooming is first demonstrated in \citep{cao2020provably}. They obtain $\ctO\Big(H^{2.5+(2 d_z + 4)\inv} K^{1 - (2d_z+1)\inv}\Big)$ regret, where $d_z$ is the zooming dimension defined specifically for episodic RL.~In another work, \citet{sinclair2023adaptive} proposes a model-based algorithm with adaptive discretization and shows the regret to be upper-bounded as $\ctO\Big(L_v H^\frac{3}{2} K^{1 - (d_z + d_\cS)\inv}\Big)$, where $L_v$ is the Lipschitz constant for the value function.~As compared the general function approximation-based works, regret bounds obtained in works on Lipschitz MDPs have a worse growth rate as a function of time horizon. However, this is expected since Lipschitz MDPs are a more general class of MDPs and have a regret lower-bound of $\Omega(K^{1 - (d_z+2)\inv})$~\citep{sinclair2023adaptive}.

\textit{\underline{Non-episodic RL}}:~The minimax regret of state-of-the-art algorithms for finite MDPs~\citep{jaksch2010near,tossou2019near} with $S$ states and $A$ actions is bounded as $\ctO(\sqrt{DSAT})$ where $D$ is the diameter of the MDP.~For finite MDPs in which the transition kernel is a mixture of $d$ component transition kernels, regret is upper-bounded as $\ctO(d \sqrt{D T})$~\citep{wu2022nearly}.~The current work develops algorithm for continuous MDPs.~\cite{wei2021learning} analyzes continuous MDPs under the assumption that the relative value function is a linear function of the features, and obtains a $\ctO(\sqrt{T})$ regret.~Another work,~\citet{he2023sample} approximates the MDP, as well as the value function by using general function classes. They derive a regret upper-bound of $\ctO(\textit{poly}(d_E, B) \sqrt{d_F T})$ regret, where $B$ is the span of the relative value function, $d_E,d_F$ are the eluder dimension and log-covering number of the function class, respectively. When the underlying continuous MDP has a $\alpha$-H\"older continuous and infinitely often smoothly differentiable transition kernel, then~\citet{ortner2012online} shows how to obtain a $\ctO\Big(T^\frac{2d + \alpha}{2d + 2 \alpha}\Big)$ regret.~To the best of our knowledge, only~\citep{kar2024adaptive,kar2024policy}\footnote{~\cite{kar2024policy} is a later version of the same paper~\cite{kar2024adaptive}.} have studied adaptive discretization for average reward Lipschitz MDPs; however, they analyze regret with respect to a given class of policies.~For~\citep{kar2024adaptive}, when this class is ``sufficiently rich'' so that it contains an optimal policy, then their algorithm does not exhibit adaptivity gains, i.e., their zooming dimension reduces to $d$, which is what one would attain via a fixed discretization scheme.~In~\cite{kar2024policy}, the zooming dimension could be even larger than $d$ if the policy class is complex.

\section{Problem Setup}\label{sec:prelim}
\textbf{Notation.} The set of natural numbers is denoted by $\bN$.~We denote the span of a $\bR$-valued function $f \in \bR^X$ by $\spn{f}$, i.e., $\spn{f} = \max_{x \in X}{f(x)} - \min_{x \in X}{f(x)}$.~We abbreviate ``with high probability'' as ``w.h.p.''~For a $\sigma$-algebra $\cF$ and a measure $\mu: \cF \to \bR$, we let $\norm{\mu}_{TV}$ denote its total variation norm~\citep{folland2013real}, i.e., $\norm{\mu}_{TV} := \sup \flbr{|\mu(B)|: B \in \cF}$.~$a \vee b$ denotes the maximum, and $a \wedge b$ denotes the minimum of two real numbers $a$ and $b$.~$\ceil{a}$ denotes the smallest integer that is greater than or equal to $a$.~At certain places, we use a single variable~($z$) to denote state-action pairs.

Let $\cM = (\cS, \cA, p, r)$ be an MDP, where the state-space $\cS$ and action-space $\cA$ are compact sets of dimension $d_\cS$ and $d_\cA$, respectively.~Let $\cS$ be endowed with Borel $\sigma$-algebra $\cB_\cS$. To simplify exposition, we assume that $\cS = [0,1]^{d_\cS}$ and $\cA = [0,1]^{d_\cA}$ without loss of generality.~We denote the system state and action taken at time $t$ by $s_t,a_t$ respectively.~The state $s_t$ evolves as follows,

\begin{align*}
     \bP\left(s_{t+1}\in B| s_t=s, a_t=a\right) = p(s,a,B), \mbox{ a.s.},&\notag\\
     \forall (s,a,B) \in \cS \times \cA \times \cB_\cS,~t \in \{0\}\cup\bN,&
\end{align*}

where $p: \cS \times \cA \times \cB_\cS \to [0,1]$ is the transition kernel that is not known by the agent.~The agent earns a reward $r(s_t,a_t)$ at time $t$, where the reward function $r: \cS \times \cA \to [0,1]$ is a measurable map.~The goal of the agent is to maximize the infinite horizon average reward.~The spaces $\cS, \cA$ are endowed with metrics $\rho_\cS$ and $\rho_\cA$, respectively. The space $\cS \times \cA$ is endowed with a metric $\rho$ that is sub-additive, i.e., we have,
\begin{align*}
    \rho\br{(s,a),(s\up,a\up)} \leq \rho_\cS (s, s\up) + \rho_\cA (a, a\up),
\end{align*}
for all $(s, a), (s\up, a\up) \in \cS \times \cA$.~For $\cZ \subseteq \cS \times \cA$, $\diamc{\cZ} := \sup_{z_1,z_2 \in \cZ}{\rho(z_1,z_2)}$.~A stationary deterministic policy is a measurable map $\phi: \cS \to \cA$ that implements the action $\phi(s)$ when the system state is $s$. Let $\Phi_{SD}$ be the set of all such policies.~The infinite horizon average reward of a policy $\phi$ when it acts on an MDP $\cM$ is denoted by $J_\cM(\phi)$, and is defined as,
\begin{align*}
    J_\cM(\phi) := \underset{T \to \infty}{\lim\inf}{\frac{1}{T} \bE_{\cM,\phi}\sqbr{\sum_{t-0}^{T-1}{r(s_t,a_t)}}},
\end{align*}
where $\bE_{\cM,\phi}$ denotes expectation taken under consideration that policy $\phi$ is used to take actions throughout on the MDP $\cM$.~The optimal average reward of the MDP $\cM$ is defined as $J\ust_\cM := \sup_{\phi \in \Phi_{SD}}{J_\cM(\phi)}$.~The regret~\citep{lattimore2020bandit} of a learning algorithm $\psi$ until $T$ is defined as,
\begin{align}
    \cR(T;\psi) &:= T J\ust_\cM -  \sum_{t=0}^{T-1} r(s_t,a_t).\label{def:regret}
\end{align}
The goal of this work is to design a learning algorithm with tight regret upper bound for Lipschitz MDPs.
~An MDP is Lipschitz if it satisfies the assumption below.
\begin{assum}[Lipschitz continuity]\label{assum:lip}
    \begin{enumerate}
        \item[(i)] \label{assum:lip_r} The reward function $r$ is $L_r$-Lipschitz, i.e., $\forall ~s, s\up \in \cS, a, a\up \in \cA$,
        \begin{align*}
            | r(s,a) - r(s\up,a\up) | \le L_r \rho\left((s,a),(s\up,a\up)\right).
        \end{align*}
        \item[(ii)] \label{assum:lip_p} 
        The transition kernel $p$ is $L_p$-Lipschitz, i.e., $\forall~ s, s\up \in \cS, a, a\up \in \cA$,
        \begin{align*}
            \norm{p(s,a,\cdot) - p(s\up, a\up, \cdot)}_{TV} \le L_p \rho\br{(s,a),(s\up,a\up)}.
        \end{align*}
    \end{enumerate}
\end{assum}
The following assumption ensures that the underlying MDP is ergodic and is typically required for average reward setup~\citep{ortner2020regret,wei2021learning,hao2021adaptive}.
\begin{assum}[Uniform ergodicity]\label{assum:unif_ergodic}
    We assume that $\{s_t\}$, the controlled Markov process~(CMP) induced by transition kernel $p$ under application of any  stationary deterministic policy is uniformly ergodic~\citep{douc2018markov}, that is, for every $\phi \in \Phi_{SD}$, there exists a unique distribution $\mu\uc{\infty}_{\phi,p}$, two constants, $C \in (0,\infty)$ and $\alpha \in (0,1)$ such that
    \begin{align*}
        \norm{\mu\uc{t}_{\phi,p,s} - \mu\uc{\infty}_{\phi,p}}_{TV} \le C \alpha^t, ~\forall s \in \cS, t \in \{0\} \cup \bN,
    \end{align*}
    where $\mu\uc{t}_{\phi,p,s}$ denotes the distribution of $s_t$ under the application of policy $\phi$ given $s_0=s$.
\end{assum}
We note that even when $\cM$ is known,~\eqref{assum:unif_ergodic} is the weakest known sufficient condition that ensures a computationally efficient way to obtain an optimal policy~\citep{arapostathis1993discrete}.~Consider the Average Reward Optimality Equation~(AROE) corresponding to the MDP $\cM$, $J + h(s) = \max_{a \in \cA}{\flbr{r(s,a) + \int_{S}{h(s\up)~ p(s,a,s\up)~ ds\up}}}$.~It can be shown that under Assumption~\ref{assum:unif_ergodic}, there exists a function $h_\cM: \cS \to \bR$ such that $(J\ust_\cM, h_\cM)$ satisty the AROE~\citep{hernandez2012adaptive} where $h_{\cM}$ is the relative value function.~Imposing an additional condition $h(s\lst) = 0$ results in unique solution to the AROE, where $s\lst$ is a designated state.~Also, there exists a stationary deterministic policy $\phi\ust$ that is optimal, i.e., $J\ust_\cM = J_\cM(\phi\ust)$.~Similarly, for a policy $\phi \in \Phi_{SD}$ there is a function $h^\phi_\cM:\cS \to \bR$ such that $(J_\cM(\phi), h^\phi_\cM)$ is the solution of $J + h(s) = r(s,\phi(s)) + \int_{S}{h(s\up)~ p(s,\phi(s),s\up)~ ds\up}$.~See Appendix~\ref{app:gen_res} for more details on properties of average reward MDPs.~The suboptimality gap~\citep{burnetas1997optimal} of a state-action pair is defined as follows:
\begin{align}
    \gap{s,a} := &J\ust_{\cM} + h_{\cM}(s) - r(s,a) \notag\\
    &- \int_{\cS}{h_{\cM}(s\up)~ p(s,a,s\up)~ ds\up}.\label{def:subgap}
\end{align}

\textbf{Zooming dimension.} Let us denote the set of state-action pairs $(s,a)$ such that $\gap{s,a} \leq \beta$ by $\cZ_\beta$.~We define the zooming dimension as
\begin{align}\label{def:zoomingdim}
    d_z := \inf{\flbr{d\up > 0 ~|~ \cN_{c_s \beta}\br{\cZ_\beta} \leq c_z \beta^{-d\up},~\forall  \beta > 0}},
\end{align}
where $\cN_{c_s \beta}\br{\cZ_\beta}$ denotes the $c_s \beta$-covering number~\citep{cao2020provably} of $\cZ_\beta$, $c_s$ \eqref{def:cs} and $c_z$ are problem-dependent constants.~Note that $d_z$ is logarithm of the covering number of a subset of $\cS \times \cA$, hence $d_z \leq d$.
\section{Algorithm}\label{sec:algo}
The proposed algorithm, \algo~discretizes the state-action space in a non-uniform grid adaptively, and the grid becomes finer as time progresses. In this section, first, we explain the adaptive discretization process.
\begin{defn}[Cells]\label{def:cell}
    A cell is a dyadic cube with vertices from the set $\{2^{-\ell}(v_1, v_2, \ldots, v_d): v_j \in \{0,1,\ldots,2^\ell\}, j=1,2,\ldots,d\}$ with sides of length $2^{-\ell}$, where $\ell\in\bN$. The quantity $\ell$ is called the level of the cell. We also denote the collection of cells of level $\ell$ by $\cP^{(\ell)}$.~For a cell $\zeta\subseteq \cS \times \cA$, its $\cS$-projection is called an $\cS$-cell,
    \begin{align}
        \pi_\cS(\zeta) :& = \left\{s \in \cS \mid (s,a) \in \zeta \textit{ for some } a \in \cA \right\},
    \end{align}
    and its level is the same as that of $\zeta$.~Denote the set of $\cS$-cells of level $\ell$ by $\cQ^{(\ell)}$.~For a cell\slash $\cS$-cell $\zeta$, we let $\ell(\zeta)$ denote its level, and let $q(\zeta)$ denote a point from $\zeta$ that is its unique representative point. $q\inv$ maps a representative point to the cell\slash $\cS$-cell that the point is representing, i.e., $q\inv(z) = \zeta$ such that $q(\zeta) = z$.\footnote{With a slight abuse of notation, we use the maps $\ell(\cdot)$, $q(\cdot)$ and $q\inv(\cdot)$ for both cells and $\cS$-cells. Note that for cells and $\cS$-cells, these maps have different domains and codomains.}
\end{defn}

\begin{defn}[Partition tree]\label{def:part_tree}
    A partition tree of depth $\ell$ is a tree in which
    (i) Each node at a depth $m \leq \ell$ of the tree is a cell of level $m$. (ii) If $\zeta$ is a cell of level $m$, where $m<\ell$ then, a) all the cells of level $m+1$ that collectively generate a partition of $\zeta$, are the child nodes of $\zeta$. The corresponding cells are called child cells, and we use $\textit{Child}(\zeta)$ to denote all the child cells of $\zeta$. b) $\zeta$ is called the parent cell of these child nodes.~The set of all ancestor nodes of cell $\zeta$ is called ancestors of $\zeta$.
\end{defn}
\algo~\eqref{algo:zorl} maintains a set of ``active cells.''~The following rule is used for activating and deactivating cells.
\begin{defn}[Activation rule]\label{def:activationrule}
    For a cell $\zeta$ define, 
    \begin{align}
        N_{\max}(\zeta) &:= \frac{c_a 2^{d_\cS+2} \log{\br{\frac{T}{\delta}}}}{\diamc{\zeta}^{d_\cS+2}}, \label{Nmax} \mbox{ and},\\
        N_{\min}(\zeta) &:= \begin{cases}
            ~~1 &\mbox{ if } \zeta = \cS \times \cA\\
            \frac{c_a \log{\br{\frac{T}{\delta}}}}{\diamc{\zeta}^{d_\cS+2}}, &\mbox{otherwise,}\label{Nmin}
        \end{cases}
    \end{align}
    where $c_a>1$ is a constant that satisfies \eqref{def:ca}, and $\delta \in (0,1)$ is the confidence parameter.~The number of visits to $\zeta$ is denoted $N_t(\zeta)$ and is defined as follows.
    \begin{enumerate}
        \item Any cell $\zeta$ is said to be active if $N_{\min}(\zeta) \leq N_t(\zeta) < N_{\max}(\zeta)$.
        \item $N_t(\zeta)$ is defined for all cells as the number of times $\zeta$ or any of its ancestors has been visited while being active until time $t$, i.e.,
    \begin{align}\label{def:visitcounter}
        N_t(\zeta) &:= \sum_{i=0}^{t-1}{\ind{(s_i, a_i) \in \zeta_i}},
    \end{align}
    where $\zeta_i$ is the unique cell that is active at time $i$ and satisfies $\zeta \subseteq \zeta_i$.
    \end{enumerate}
Denote the set of active cells at time $t$ by $\cP_t$.
\end{defn}
We note that since the diameter of a child cell is half that of its parent, a parent cell is deactivated, and its child cells are activated simultaneously.~Since a cell is partitioned by its child cells, the set of active cells at time $t$, $\cP_t$ forms a partition of the state action space. \algo~clusters all the state-action pairs into the active cells by utilizing the information gathered until $t$. Each point in an active cell (cluster) $\zeta$ looks similar for the purpose of generating optimal actions, and is hence represented via its unique representative point $q(\zeta)$.~Denote the collection of representative points of the active cells at time $t$ by $\cZ_t:= \flbr{q(\zeta): \zeta \in \cP_t}$.~Let $\ell_{\max,t}$ be the level of the smallest cells in $\cP_t$.~At time $t$,~\algo~partitions the state-space into $\cS$-cells of level $\ell_{\max,t}$. We denote this $\cS$-cell partition by $\cQ_t$, i.e., $\cQ_t := \cQ\uc{\ell_{\max,t}}$, and the corresponding representative points by $\cS_t$, i.e., $\cS_t := \flbr{q(\zeta) : \zeta \in \cQ_t}$. $\cS_t$ can be thought of as the discretized state space at time $t$. \algo~maintains estimates of the transition probability kernel that has support on $\cS_t$.

Now, we introduce a generic notation for discretized transition kernels, which will be used often in this paper.~Let $\tilde{\cS}$ be a set of representative points of a partition of $\cS$ consisting of only $\cS$-cells. Then, for a continuous transition kernel $\tilde{p}$, and $\tilde{\cZ} \subseteq \cS \times \cA$, we define $\wp_{\tilde{\cZ} \to \tilde{\cS},\tilde{p}}(z,\cdot): \tilde{\cZ} \mapsto [0,1]^{\tilde{\cS}}$ as follows, 
\begin{align}\label{def:disc_p}
    \wp_{\tilde{\cZ} \to \tilde{\cS},\tilde{p}}(z,s) := \tilde{p}(z,q\inv(s)),~\forall z \in \tilde{\cZ}, s \in \tilde{\cS}.
\end{align}
The kernel $\wp_{\tilde{\cZ} \to \tilde{\cS},\tilde{p}}$ can be viewed as a discretization of $\tilde{p}$.

\textbf{Estimating the Transition Kernel.}
~Let $N_t(\zeta, \xi)$ be the total number of transitions from a cell $\zeta$, or from its active ancestors to a $\cS$-cell $\xi$ until $t$, i.e., $N_t(\zeta, \xi) := \sum_{i=1}^{t-1}{\ind{(s_i, a_i, s_{i+1}) \in \zeta_i \times \xi}}$.~For any state-action pair $z$, we let $q\inv_t(z)$ denote the active cell that contains $z$.~Denote $\tilde{\cS}_t(z) := \{q(\xi) : \xi \in \cQ^{(\ell(q\inv_t(z)))}\}$, which is the set of representative states of the $\cS$-cells of level $\ell(q\inv_t(z))$.~We first construct an estimate $\hat{p}^{(d)}_t$~\eqref{eq:kernel_estimate} of the discretized version of the true stochastic kernel as follows,
\begin{align}\label{eq:kernel_estimate}
    \hat{p}^{(d)}_t(z,s) := \frac{N_t\br{q\inv(z), q\inv(s)}}{1 \vee N_t\br{q\inv(z)}},
\end{align}
$z \in \cZ_t, s \in \tilde{\cS}_t(z)$.~Note that the distribution $\hat{p}^{(d)}_t(z,\cdot)$ is supported on a finite set $\tilde{\cS}_t(z)$, and the sets $\{\tilde{\cS}_t(z)\}$ are adaptive.~$\hat{p}^{(d)}_t(z,\cdot)$ is then extended to obtain a continuous kernel $\hat{p}_t$.~$\hat{p}_t$ is defined as,
\begin{align}\label{def:p_hat}
    \hat{p}_t(z,B) := &\sum_{s \in \tilde{\cS}_t(z)}{\frac{\lambda(B \cap q\inv(s))}{\lambda(q\inv(s))} \hat{p}^{(d)}_t(z,s)},
\end{align}
where $z \in \cZ_t$, $B \in \cB_\cS$, and $\lambda(\cdot)$ is the Lebesgue measure on $(\cS,\cB_\cS)$.~To obtain a computationally feasible algorithm, we work with the discretization $\wp_{\cZ_t \to \cS_t,\hat{p}_t}$ of $\hat{p}_t$.

Note that the set $\tilde{S}_t(z)$ depends upon the diameter of the active cell containing $z$, so that the support of the discrete kernel $\hat{p}^{(d)}_t(z,\cdot)$ varies with $z$.~The construction of $\wp_{\cZ_t \to \cS_t,\hat{p}_t}$ from $\hat{p}^{(d)}_t$ ensures that the support of the discrete kernel at every point is the same ($\cS_t$). This allows us to use the \evi~algorithm, which will be introduced later in this section.

\textbf{Concentration Inequality.}~\algo~constructs a confidence ball centered at $\wp_{\cZ_t \to \cS_t,\hat{p}_t}$ that contains discretized version of the true transition kernel, $p$ w.h.p. For a cell $\zeta \in \cP_t$, the confidence radius associated with the estimate $\wp_{\cZ_t \to \cS_t,\hat{p}_t}(q(\zeta),\cdot)$ is defined as follows,
\begin{align}\label{def:eta_k}
    \eta_t(\zeta) &:= \min \Bigg\{2,  3 \br{\frac{c_a \log{\br{\frac{T}{\delta}}}}{N_t(\zeta)}}^\frac{1}{d_S + 2} \notag\\
    &\qquad\qquad + (3 L_p + C_p) \diamc{\zeta}\Bigg\},
\end{align}
where $C_p$ is an upper bound on the derivatives of the transition density functions, as described in Assumption~\ref{assum:bdd_der}, and the constant $c_a \geq 1$ satisfies~\eqref{def:ca}.~It turns out that the following value of $c_a$ satisfies~\eqref{def:ca}:
\begin{align}
    c_a = \frac{2 d^{\frac{d_\cS}{2}}}{9} \frac{\log{\br{6 d^\frac{d}{2}}}}{\log{\br{\frac{T}{\delta}}}} + \frac{d}{d_\cS+2} + 1. \label{value:ca}
\end{align}
Lemma~\ref{lem:conc_ineq} shows that w.h.p.,
\begin{align*}
    \norm{\wp_{\cS\times\cA \to \cS_t,p}(z,\cdot) - \wp_{\cZ_t \to \cS_t,\hat{p}_t}(q(\zeta),\cdot)}_{TV}& \leq \eta_t(\zeta),\\
    &\forall z \in \zeta,
\end{align*}
for every $t$ and every $\zeta \in \cP_t$.~This leads to the definition of the confidence ball that \algo~uses.

Now, we introduce the discrete state-action space that we will use in the definition of the confidence ball.~The set of all the relevant cells for $s \in \cS$ at time $t$ are defined as $\rel{t}{s} := \{ \zeta \in \cP_t \mid \exists a \in \cA \mbox{ such that } (s,a) \in \zeta \}$.~These are those active cells whose $\cS$-projection contain the state $s$.~Thus, $\rel{t}{s}$ can be seen as the set of those cells in the state-action space that are associated with state $s$ currently. Recall that $\cS_t$ is the discrete state space at time $t$.~Define
\begin{align*}
    &\cA_t(s) \\
    &:= \cup_{\zeta \in \rel{t}{s}}{\left\{a \in \cA \mid q(\zeta) = (s\up,a) \mbox{ for some } s\up \in \cS \right\}}. 
\end{align*}
$\cA_t(s)$ denotes the set of actions that are available to the agent that can be played by it currently in state $s$. The discrete action space at time $t$ is given by $\cA_t := \{\cA_t(s) : s \in \cS_t\}$.~Let $\cS_t \times \cA_t := \{(s,a) \mid s \in \cS_t, a \in \cA_t(s)\}$.~Define the confidence ball,
\begin{align}
    &\cC_t := \notag\\
    &\big\{\te: \cS_t \times \cA_t \mapsto [0,1]^{\cS_t} \mid \sum_{s \in \cS_t}{\te(z,s)} = 1,~\forall z \in \cS_t \times \cA_t, \notag\\
    & \norm{\te(z\up,\cdot) - \wp_{\cZ_t \to \cS_t, \hat{\wp}_t}(\bar{z},\cdot)}_1 \leq \eta_t(q\inv(\bar{z})) \mbox{ for every } \notag\\
    & \bar{z} \in \cZ_t, z\up \in q\inv(\bar{z}) \cap \cS_t \times \cA_t\big\}.\label{def:confball}
\end{align}

As a consequence of Lemma F.1, $\cC_t$ contains $\wp_{\cS_t \times \cA_t \to \cS_t, p}$ w.h.p. Denote the time when the $k$-th episode of \algo~begins by $\tau_k$.~At the beginning of each episode $k$, \algo~constructs a set of discrete MDPs $\cM^{+}_{\tau_k}$ with transition kernel can be chosen from $\cC_{\tau_k}$, and reward function is equal to the true rewards at the discrete points $\cS_t\times\cA_t$, plus a bonus term. Such a set of MDPs is called the ``extended MDP'' and it is commonly used to incorporate optimism in upper confidence bound-based RL algorithms~\citep{jaksch2010near}. The optimal average reward of the extended MDP exceeds the optimal average reward of the true MDP since $\cC_t$ contains the true discretized transition kernel $\wp_{\cS_t \times \cA_t \to \cS_t, p}$ w.h.p.; this yields an ``optimistic push'' which ensures ``sufficient exploration.''~The confidence ball shrinks with the number of visits to different state-action pairs; this causes a reduction in the amount of optimism bonus. The extended MDP thus closely approximates the true MDP in the ``important regions'' (those necessary for recovering an optimal policy) of the state-action space as time progresses.~Next, we discuss the extended MDP in detail, how to solve it, and its role in \algo.

\textbf{Extended MDP.}~Consider the following modified reward function defined on $\cS_t \times \cA_t$,
\begin{align*}
    \tilde{r}_t(s,a) &= r(q(q\inv_t(s,a))) + L_r \diamc{q\inv_t(s,a)}, 
\end{align*}
in which a bonus term proportional to the diameter of the active cell that contains $(s,a)$ has been included in order to compensate for the ``discretization error.''~Consider the following collection of MDPs $\cM^{+}_t := \{(\cS_t, \cA_t, \tilde{p}, \tilde{r}_t) : \tilde{p} \in \cC_t\}$.~One may view $\cM^{+}_t$ as an MDP with the finite state space $\cS_t$ and an extended action space, hence the name extended MDP. An element from the extended action space has two components: control input from $\cA_t$, and a transition kernel from $\cC_t$. Let $\Phi_t$ be the set of those policies $\phi$ that satisfy $\phi(s) \in \cA_t(s),~\forall s \in \cS_t$.~Denote the optimal average reward of $\cM^+_t$ by $J\ust_{\cM^+_t}$.~\algo~uses the \evi~algorithm in order to obtain an optimal policy for the extended MDP at the beginning of every episode.~This is discussed next.

\begin{algorithm}[ht]
    \caption{Extended Value Iteration~(\evi)}
    \label{algo:evi}
    \begin{algorithmic}
        \STATE {\bfseries Input} Extended MDP $\cM^+$, accuracy parameter $\gamma > 0$.
        \STATE {\bfseries Initialize} $v_0 = \{0\}^{\abs{S}}$, $n = 0$.
        \WHILE{True}
            \STATE $v_{n+1} = \cT v_n$~\eqref{def:T_v}
            \IF{$\spn{v_{n+1} - v_{n}} \leq \gamma$}
                \STATE {\bfseries break}
            \ENDIF
            \STATE $n \leftarrow n+1$
        \ENDWHILE
        \RETURN Greedy Policy w.r.t. $v_n$ 
    \end{algorithmic}
\end{algorithm}
\begin{algorithm}[ht]
    \caption{Extended Policy Evaluation~(\epe)}
    \label{algo:epe}
    \begin{algorithmic}
        \STATE {\bfseries Input} Extended MDP $\cM^+$, policy $\phi$, accuracy parameter $\gamma > 0$, reference state $s\lst$.
        \STATE {\bfseries Initialize} $v_0 = \{0\}^{\abs{S}}$, $n = 0$.
        \WHILE{True}
            \STATE $v_{n+1} = \underset{\te \in \cC}{\max}\Big\{\tilde{r}(s,\phi(s)) +\underset{s\up \in S}{\sum}{\te(s,\phi(s),s\up) v_n(s\up)}\Big\}$
            \IF{$\spn{v_{n+1} - v_{n}} \leq \br{v_{n+1}(s\lst) - v_{n}(s\lst)}\gamma$}
                \STATE {\bfseries break}
            \ENDIF
            \STATE $n \leftarrow n+1$
        \ENDWHILE
        \RETURN $v_{n+1}(s\lst) - v_{n}(s\lst)$
    \end{algorithmic}
\end{algorithm}

\textbf{\evi} (Algorithm~\ref{algo:evi}) takes as input an extended MDP, and an error tolerance parameter $\gamma > 0$, and returns a policy whose average reward is $\gamma$-close to the optimal value of the extended MDP.~A generic extended MDP $\cM^+ = \{(S, A, \tilde{p}, \tilde{r}) : \tilde{p} \in \cC\}$ has a discrete state space $S$, and discrete action space $A = \{ A(s) : s \in S\}$ where $A(s)$ is the set of actions that are permissible in state $s$.~$\cC$ is a set of transition kernels which yield a distribution over $S$ for each point in $S \times A$.~$\tilde{r}$ is the reward function.~Given the extended MDP $\cM^+$, define the following operator $\cT: \bR^{S}\mapsto\bR^{S}$,
\begin{align}
    \cT v(s) =  \max_{\substack{a \in A(s)\\ \te \in \cC}}\Big\{\tilde{r}(s,a) + \sum_{s\up \in S}{\te(s,a,s\up) v(s\up)}\Big\}.\label{def:T_v}
\end{align}
See that $\mathcal{T}$ is the Bellman operator~\citep{puterman2014markov} for the extended MDP, $\mathcal{M}^+$, where maximization of the value is done over the extended action space, $A(s) \times \mathcal{C}$. Recall that the Bellman operator for usual MDPs maximizes over the set of all actions. At time $\tau_k$, \algo~calls \evi$(M^+_{\tau_k}, 1/\sqrt{T})$. The \evi~subroutine then applies the Bellman operator~\eqref{def:T_v} for $M^+_{\tau_k}$ repetitively until stopping criterion is met and returns the policy $\tilde{\phi}_k \in \Phi_{\tau_k}$, which is $1/\sqrt{T}$-near optimal (Lemma~\ref{lem:conv_evi}). \algo~then extends $\tilde{\phi}_k$ on the entire continuous space $\cS$ to obtain $\phi_k$ as follows: for every state in the $\cS$-cell $\xi \in \cQ_{\tau_k}$, $\phi_k$ plays $\tilde{\phi}_k(q(\xi))$, i.e.,
\begin{align}
    \phi_k(s) = \tilde{\phi}_k(q(\xi)),\forall s \in \xi,~\xi \in \cQ_{\tau_k}.  \label{map:phi_varphi}
\end{align}
\textbf{Episode Duration.}~\algo~chooses the duration of the $k$-th episode as a function of the expected diameter of the states visited at stationarity of the chosen policy, $\phi_k$. Define the extended MDP $\cM^{d,+}_{t} = \{(\cS_t, \cA_t, \tilde{p}, d_t) : \tilde{p} \in \cC_t\}$, where
\begin{align*}
    d_t(s,a) := \diamc{q\inv_t(s,a)},~\forall (s,a) \in \cS_t \times \cA_t.
\end{align*}
Let $\tilde{\phi} \in \Phi_t$.~We define the proxy diameter of $\tilde{\phi}$ at time $t$ as the average reward of the policy $\tilde{\phi}$ evaluated on MDP $\cM^{d,+}_{t}$ and denote it by $\pdiam{t}{\tilde{\phi}}$.~To be precise, $\pdiam{t}{\tilde{\phi}}$ is the optimal value of $\cM^{d,+}_{t}$ when the control input component of the extended action is chosen according to the policy $\tilde{\phi}$, and the transition kernel is chosen so as to maximize the average reward.~Define the diameter of a policy $\phi \in \Phi_{SD}$ at time $t$ as follows:
\begin{align}
    \diam{t}{\phi} := \int_{\cS}{\diamc{q\inv_t(s,\phi(s))} \mu\uc{\infty}_{\phi,p}(s) ds}.\label{def:diam_pol}
\end{align}
In Appendix~\ref{app:prop_pdiam}, we show that $\pdiam{\tau_k}{\tilde{\phi}_k}$ is a tight upper-bound of $\diam{\tau_k}{\phi_k}$ for every $k$.~The duration of the $k$-th episode, $H_k$ is chosen as,
\begin{align}\label{def:epi_dur}
    H_k = \frac{C_H \log{\br{T/\delta}}}{\pdiam{\tau_k}{\tilde{\phi}_k}^{2(d_\cS + 1)}},
\end{align}
where $C_H$, a problem-dependent quantity of $\cO(\log(T))$, satisfies \eqref{def:CH}. This choice of episode duration ensures a reduction of the diameter of the chosen policy in every episode. \algo~uses \epe~(Algorithm~\ref{algo:epe}) in order to compute $\pdiam{\tau_k}{\tilde{\phi}_k}$. $\epe(\cM^{d,+}_{t}, \tilde{\phi}, \gamma, s\lst)$ returns a value from $[\br{1+\gamma}\inv\pdiam{t}{\tilde{\phi}}, \br{1-\gamma}\inv\pdiam{t}{\tilde{\phi}}]$~(Corollary~\ref{cor:conv_epe}) for any $\tilde{\phi} \in \Phi_t$ where $\gamma$ is a parameter chosen by the agent.

\begin{algorithm}[ht]
    \caption{Zooming Algorithm for RL~(\algo)}
    \label{algo:zorl}
    \begin{algorithmic}
        \STATE {\bfseries Input} Horizon $T$, upper-bounds on $L_r$, $L_p$, $C_p$, constants $c_a$, $C_H$ and accuracy parameter $\gamma > 0$
        \STATE {\bfseries Initialize} $h=0$, $k=0$, $H_0 = 0$, $\cP_0 = \{\cS \times \cA\}$
        \FOR{$t= 0$ to $T-1$}
            \IF{$h \geq H_k$}
                \STATE $k \leftarrow k+1$, $h \leftarrow 0$, $\tau_k = t$, $s\lst \in \cS_t$
                \STATE Construct $\cM^{+}_{\tau_k}$ and $\cM^{d,+}_{\tau_k}$
                \STATE $\tilde{\phi}_k = \evi(\cM^+_{\tau_k}, 1/\sqrt{T})$
                \STATE Obtain $\phi_k$ from $\tilde{\phi}_k$ according to \eqref{map:phi_varphi}
                \STATE $d_k = \epe(\cM^{d,+}_{\tau_k}, \tilde{\phi_k}, \gamma, s\lst)$
                \STATE $H_k = C_H \log{\br{T/\delta}}~ d_k^{-2(d_\cS + 1)}$
            \ENDIF
            \STATE $h \leftarrow h+1$
            \STATE Play $a_t = \phi_k(s_t)$, observe $s_{t+1}$ and receive $r(s_t, a_t)$
            \IF{$N_t(q\inv_t(s_t, a_t)) = N_{\max}(q\inv_t(s_t, a_t))$}
                \STATE $\cP_{t+1} = \cP_t \cup \textit{Child}(q\inv_t(s_t, a_t)) \setminus \{q\inv_t(s_t, a_t)\}$
            \ELSE
                \STATE $\cP_{t+1} = \cP_t$
            \ENDIF
        \ENDFOR
	\end{algorithmic}
\end{algorithm}
\section{Regret Analysis}\label{sec:regret}
We let $\Delta(\phi) := J\ust_{\cM} - J_{\cM}(\phi)$ denote the suboptimality of policy $\phi$.~The following result establishes a relation between the suboptimality of a policy, and the suboptimality gap of the state-action pairs through which this policy passes, where suboptimality gap of state-action pair is defined in \eqref{def:subgap}.~Its proof is deferred to Appendix~\ref{app:gen_res}.
\begin{lemma}\label{lem:gap_phi}
    Consider the MDP $\cM = (\cS, \cA, p,r)$. For any policy $\phi \in\Phi_{SD}$, we have
    \begin{align*}
        \Delta(\phi) = \int_{\cS}{\gap{s,\phi(s)}~ \mu\uc{\infty}_{\phi,p}(s)~ ds}.
    \end{align*}
\end{lemma}
We make the following assumption on the true kernel $p$ for deriving concentration bound for the estimate of the discretized transition kernel $\wp_{\cS_t \times \cA_t \to \cS_t,p}$~\eqref{def:disc_p}.
\begin{assum}[Bounded Radon-Nikodym derivative]\label{assum:bdd_der}
    The probability measures $\{p(s,a,\cdot)\}$ are absolutely-continuous w.r.t. the Lebesgue measure on $(\cS,\cB_\cS)$, with density functions given by $\{f_{(s,a)}\}$.~We assume that these densities satisfy 
    \begin{align*}
        \norm{\frac{\partial f_{(s,a)}(s^+)}{\partial s^+(i)}}_\infty \leq C_p, \forall (s,a) \in \cS \times \cA, i = 1, 2, \ldots,  d_\cS,
    \end{align*}
    where the variable $s^+ = (s^+(1),s^+(2),\cdots,s^+(d_\cS))$ represents the next state.
\end{assum}
Assumption~\ref{assum:bdd_der} ensures that the discretizations of $p(s,a,\cdot)$ with respect to the partitions $\cQ^{(\ell(q\inv_t(s,a)))}$ and $\cQ_t$ are at most $C_p~ \diamc{q\inv_t(s,a)}$ distance apart (Lemma~\ref{lem:disc_dist}).~Using this result, Lemma~\ref{lem:conc_ineq} shows that under Assumption~\ref{assum:lip} and Assumption~\ref{assum:bdd_der}, $\cap_{t=0}^{T-1}\{\wp_{\cS_t \times \cA_t \to \cS_t, p} \in \cC_t\}$ occurs w.h.p.~The following assumption allows us to derive an upper-bound on the span of the EVI iterates, which is essential to ensure that the algorithm is not overly optimistic.
\begin{assum}[Bound on Stationary Distributions]\label{assum:statn_dist}
    There is a constant $\kappa > 0$ such that for every policy $\phi \in \Phi_{SD}$, and for every $\zeta \in \cB_\cS$, we have, $\kappa \cdot \lambda(\zeta) \leq \mu^{(\infty)}_{\phi,p}(\zeta)$, where $\lambda(\cdot)$ denotes the Lebesgue measure on $(\cS,\cB_\cS)$.
\end{assum}
\begin{remark}[Regarding Assumptions]
    In the average reward setup for continuous space MDPs, assumptions similar to Assumption~\ref{assum:statn_dist} or more restrictive assumptions are needed.~For example, \citet{ormoneit2002kernel} assumes that the transition kernel of the underlying MDP has a strictly positive Radon-Nikodyn derivative in order to show that a proposed adaptive policy converges to an optimal policy.~\citet{wang2023optimal} and \citet{shah2018q} derive optimal sample complexity for average reward RL and for discounted reward RL, respectively, under an assumption that the $m$-step transition kernel is bounded below by a known measure.~\citet{kar2024policy} also make the same assumption as ours in order to derive the regret upper-bound of their adaptive discretization-based algorithm.~\citet{wei2021learning} bounds the regret for average reward RL algorithm when the relative value function is a linear function of a set of known feature maps.~Their ``uniformly excited features'' assumption ensures that upon playing any policy, the confidence ball shrinks in each direction, which has a similar effect as Assumption~\ref{assum:statn_dist}.
\end{remark}
We now present our main result that provides an upper-bound on regret of \algo. We only provide a proof sketch here and delegate its detailed proof to the appendix.

\begin{thm}\label{thm:regupperbound}
     Under Assumptions~\ref{assum:lip}, \ref{assum:unif_ergodic}, \ref{assum:bdd_der} and \ref{assum:statn_dist}, with probability at least $1 - \delta$, $\cR(T;\algo)$ is upper-bounded as $\cO\big(T^{1-\deff\inv}\big)$ where $\deff = 2 d_\cS + d_z + 3$. 
\end{thm}

\begin{proof}[Proof sketch]
    We decompose the regret~\eqref{def:regret} in the following manner.~Let $K(T)$ denote the total number of episodes during $T$ timesteps.~Then, 
    \begingroup
    \allowdisplaybreaks
    \begin{align*}
       & \cR(T;\algo) = T J\ust_{\cM} - \sum_{k=1}^{K(T)}{\sum_{t = \tau_k}^{\tau_{k+1}-1}{r(s_t,a_t)}} \notag\\
        &= \underbrace{\sum_{k=1}^{K(T)}{H_k \br{J\ust_{\cM} - J_{\cM}(\phi_k)}}}_{(a)} \\
        &+ \underbrace{\sum_{k=1}^{K(T)}{\br{H_k~ J_{\cM}(\phi_k) - \sum_{t=\tau_k}^{\tau_{k+1}-1}{r(s_t,\phi_k(s_t))}}}}_{(b)}.
    \end{align*}
    \endgroup
    (a) captures the regret arising due to playing a suboptimal policy $\phi_k$ during the $k$-th episode, while (b) captures the possible degradation in performance during the transient stage as compared with the average rewards of the chosen policies. (a) and (b) are bounded separately below.

    \textbf{Bounding} (a): Step 1: In Lemma~\ref{lem:optimism}, we show that the policy obtained by solving~$\cM^+_t$ is optimistic, i.e., w.h.p. $J^{\star}_{\cM^+_t} \geq J\ust_{\cM}$.~Also, in Lemma~\ref{lem:ub_opt}, we show that w.h.p., $J^{\star}_{\cM^+_t} \leq J\ust_{\cM} + C_{ub}~ \diam{t}{\phi_k}$, where $C_{ub}$~is as defined in \eqref{def:Cub}.~As a consequence of the above two results, on a high probability set, a suboptimal policy $\phi$ will never be played from episode $k$ onwards if $\diam{\tau_k}{\phi} \leq C_{ub}\inv \cdot \Delta(\phi)$.~Note that the cumulative regret arising due to policies with $\Delta(\cdot)$ less than $\eps$ is at most $\eps T$. We choose $\eps$ optimally and restrict the analysis to regret arising from playing other policies.
    
    Step 2: We combine Step 1 with Lemma~\ref{lem:gap_phi} in Lemma~\ref{lem:keycell} and show that on a high probability set, in each episode $k$, there is a state $s \in \cS$ such that
    \begin{subequations}
        \begin{align}
            &\diamc{\zeta} \geq   \notag\\
            &\frac{1}{3 C_{ub}}\max\{\gap{s,\phi_k(s)}, C_{ub} \diam{\tau_k}{\phi_k}\}, \label{keycell:cond1}\\
            &\mu\uc{\infty}_{\phi_k,p}(\pi_\cS(\zeta)) \geq \br{\frac{\diam{\tau_k}{\phi_k}}{3}}^{d_\cS + 1},\label{keycell:cond2}
        \end{align}
    \end{subequations}
    where $\zeta = q\inv_{\tau_k}(s,\phi_k(s))$.~This cell $\zeta$ is called a key cell in the $k$-th episode.
    \begin{figure}[t]
        \centering
        \includegraphics[width=0.7\linewidth]{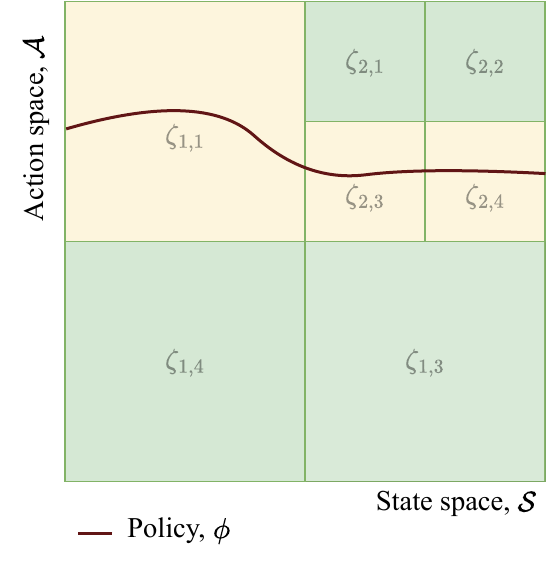}
        \caption{Key cell: The policy $\phi$ is played during the $k$-th episode. This diagram depicts the discretization grid at the beginning of the $k$-th episode.~Then, one of the cells $\zeta_{1,1}$, $\zeta_{2,3}$ and $\zeta_{2,4}$ must be a key cell with a high probability (Lemma~\ref{lem:keycell}). There must be a state $s$ such that $(s,\phi(s))$ belongs to this cell, and $s$ satisfies \eqref{keycell:cond1} and \eqref{keycell:cond2}.}
        \label{fig:keycell}
    \end{figure}
    
    Step 3: Then we show that with a high probability, the key cells of the $k$-th episode are visited at least $\cO\br{\log{\br{\frac{T}{\delta}}} \diamc{\zeta}^{-(d_\cS + 1)}}$ times during the $k$-th episode.~This is done in Lemma~\ref{lem:lb_num_visit}.

    Step 4: We obtain a bound on the cardinality of the key cells associated with playing policies from the set $\Phi_{2^{-i}} = \{\phi \in \Phi_{SD} \mid \Delta(\phi) \in (2^{-i}, 2^{-i+1}]\}$ by showing that these cells are contained within a set of cells that has a cardinality at most $\cO(2^{id_z})$. We then use this bound along with the lower-bound on the number of plays of the key cells, and conclude that the policies from $\Phi_{2^{-i}}$ are played for a maximum of $\cO\br{\log{\br{\frac{T}{\delta}}} 2^{i(2 d_\cS + d_z + 3)}}$ time-steps~(Lemma~\ref{lem:bdd_Phi_play}).
    
    Step 5: The term~(a) can be written as the sum of the regrets arising due to playing policies from the sets $\Phi_{2^{-i}}$, where $i=1,2,\ldots,\ceil{\log{\br{\frac{1}{\eps}}}}$, where $\eps = T^{-\frac{1}{2 d_\cS + d_z + 3}}$.~To bound the regret arising due to playing policies from $\Phi_{2^{-i}}$, we multiply $\cO\br{\log{\br{\frac{T}{\delta}}} 2^{i(2 d_\cS + d_z + 3}}$ by $2^{-i + 1}$.~We then add these regret terms from $i=1$ to $\ceil{\log{\br{\frac{1}{\eps}}}}$ and $\eps T$.
    
    Step 6: Lastly, we add $\sqrt{T}$ to the final bound to compensate for the inaccuracy caused by \evi~due to finite computational resources. This gives us the upper-bound on (a) w.h.p.

    \textbf{Bounding} (b): upper-bound on the term $(b)$ relies on the uniform ergodicity property~(Assumption~\ref{assum:unif_ergodic}) of $\cM$ and a trick that converts ``Markovian noise'' to ``martingale noise''~\citep{metivier1984applications}.~Proposition~\ref{prop:bddb} shows that on a high probability set, we must pay a constant penalty each time we change policy, which is $\cO(K(T) + \sqrt{T})$.~We show that the rule which decides when to start a new episode ensures that $K(T)$ is bounded above by $\cO(T^\frac{d_z + 1}{2 d_\cS + d_z + 3})$, and so is the term (b).

    Summing the upper-bounds on (a) and (b), we obtain the desired regret bound.
\end{proof}
\section{Simulations}\label{sec:sim}
We compare the performance of \algo~(Algorithm~\ref{algo:zorl})~with that of UCRL2~\citep{jaksch2010near}, TSDE~\citep{ouyang2017learning}, RVI-Q~\citep{borkar2000ode} which is a Q-learning algorithm for average-reward RL, ZoRL-$\eps$~\citep{kar2024adaptive}, and the heuristic algorithm PZRL-H~\citep{kar2024policy}.~For competitor policies that are designed for finite state-action spaces, we apply them on a uniform discretization of $\cS \times \cA$ performed at time $t=0$.~Simulation experiments are conducted on the following systems:~(i) \texttt{Continuous RiverSwim}, where the environment models an agent who is swimming in a river. (ii) Linear Quadratic~(LQ) control systems~\citep{abbasi2011regret} where the state evolves as $s_{t+1} = A s_t + B a_t + w_t$, and we truncate the state-action space in order to ensure that they are compact.~Denote the two systems of dimension $2 \times 2$ and $2 \times 4$ as \texttt{Truncated LQ-$1$} and \texttt{Truncated LQ-$2$}, respectively. (iii) \texttt{Non-linear System} where the state evolves as $s_{t+1} = A f(s_t) + B g(a_t) + w_t$, where $f$ and $g$ are non-linear functions.~Similar to the truncated LQ systems, we truncate the state-action space.~Details of the environments can be found in~\citet{kar2024adaptive}, and also in Appendix~\ref{app:sim}.~We plot the cumulative rewards averaged over $50$ runs in Figure~\ref{fig:perf}.~\algo~performs the best among all six algorithms on each of the environments.~Very recently, \citet{kar2024policy} has replaced PZRL-H with two algorithms, PZRL-MB and PZRL-MF.~In Appendix~\ref{app:sim} we compare their performance with~\algo.

\begin{figure}[ht]
    \centering
    \begin{subfigure}[b]{0.49\linewidth}
        \centering
        \includegraphics[width=\textwidth]{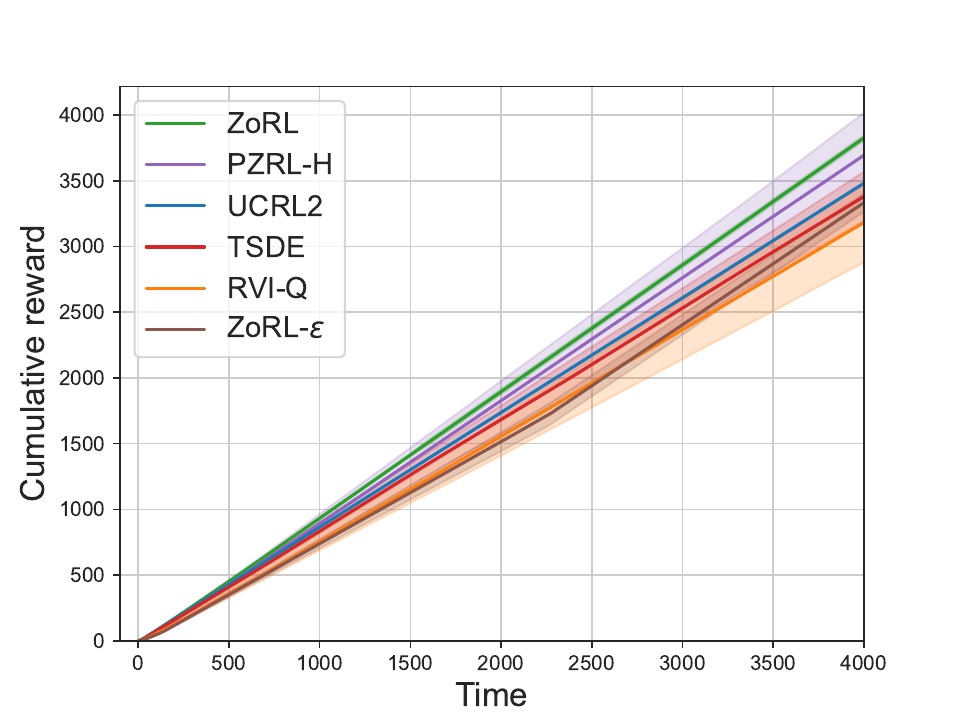} 
        \caption{Continuous RiverSwim}
        \label{fig:rm_rew}
    \end{subfigure}
    \begin{subfigure}[b]{0.49\linewidth}
        \centering
        \includegraphics[width=\textwidth]{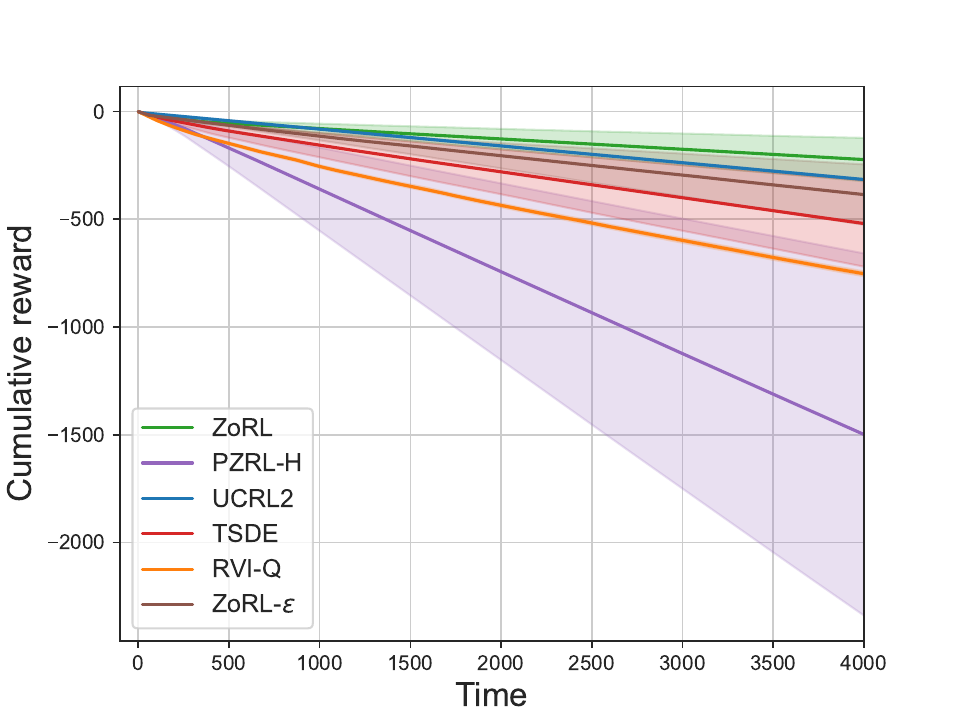} 
        \caption{Truncated LQ-$1$}
        \label{fig:ls1_rew}
    \end{subfigure}
    \begin{subfigure}[b]{0.49\linewidth}
        \centering
        \includegraphics[width=\textwidth]{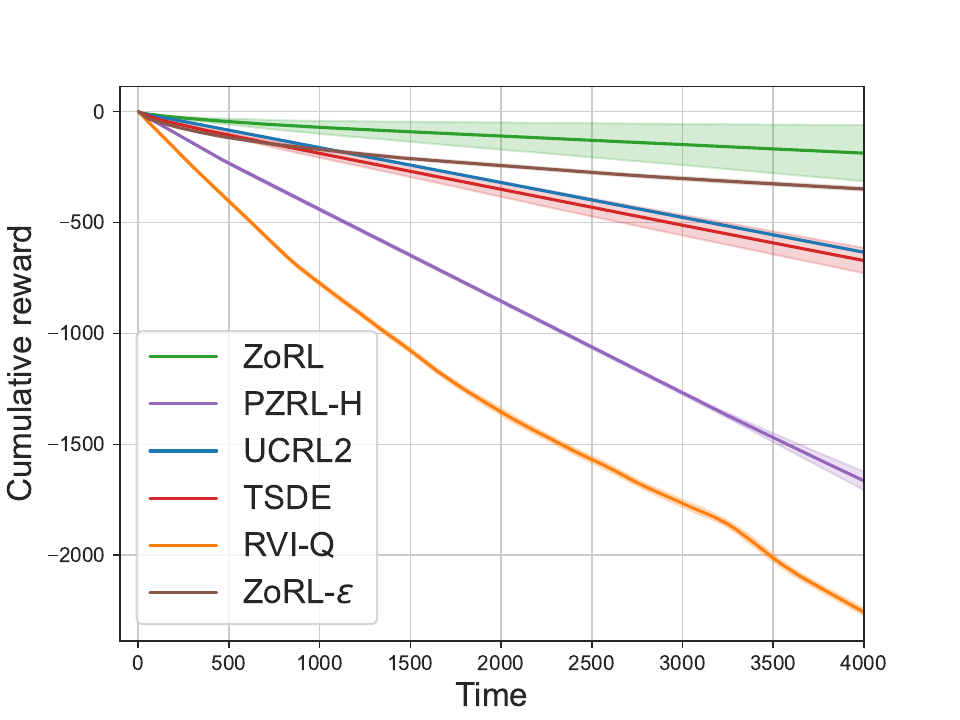}
         \caption{Truncated LQ-$2$}
        \label{fig:ls2_rew}
    \end{subfigure}
    \begin{subfigure}[b]{0.49\linewidth}
        \centering
        \includegraphics[width=\textwidth]{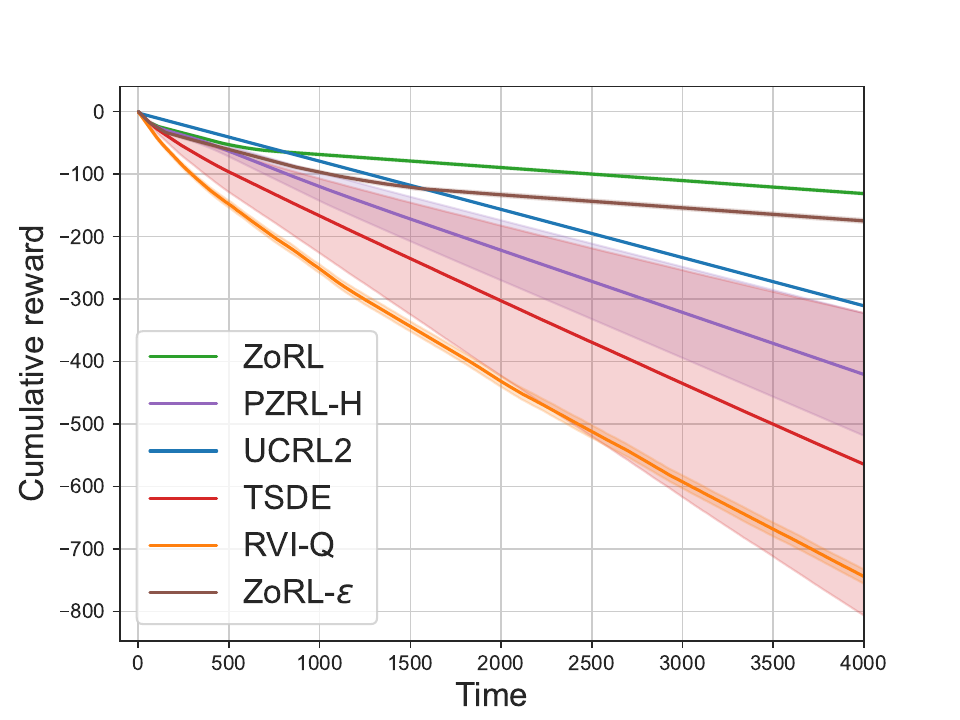}
         \caption{Non-linear System}
        \label{fig:nls_rew}
    \end{subfigure}
    \caption{Cumulative Reward Plots.}
    \label{fig:perf}
\end{figure}
\vspace{-5pt}
\section{Conclusion}\label{sec:conc}
We propose a computationally efficient algorithm for average reward RL for Lipschitz MDPs in continuous spaces, and show that it is truly adaptive, i.e. it achieves a regret of $\ctO\big(T^{1 - \deff\inv}\big)$, where $\deff = 2 d_\cS + d_z + 3$.~The zooming dimension $d_z$ is a problem-dependent quantity, measures the size of near-optimal state-action pairs and is bounded above by $d$, the dimension of the state-action space.~Simulation experiments support the theoretical findings.~\algo~overperforms the popular fixed discretization-based algorithms as well as adaptive discretization-based algorithms.

\subsubsection*{Acknowledgements}
This work is partially supported by the SERB Grant SRG/2021/002308. The authors acknowledge the Prime Minister's Research Fellowship to Avik Kar.

\bibliography{refs}

\newpage

\onecolumn

\title{Adaptive Discretization-based Non-Episodic Reinforcement Learning in Metric Spaces\\(Supplementary Material)}
\maketitle

\appendix
\textbf{Organization of the Appendix.}~Some properties of the MDPs that satisfy Assumption~\ref{assum:unif_ergodic} are discussed in Appendix~\ref{app:gen_res}. It also includes the proof of Lemma~\ref{lem:gap_phi}. Some important properties of extended MDPs can be found in Appendix~\ref{app:prop_emdp}. We use these properties while analyzing the regret of \algo. Next, in Appendix~\ref{app:prop_pdiam}, we show certain properties of the proxy diameters of policies. Results obtained in Appendix~\ref{app:prop_emdp} play a crucial role in deriving those properties. A high probability lower bound on the number of visits to the key cells in each episode is derived in Appendix~\ref{app:visits}.~In Appendix~\ref{app:regret}, we derive the desired regret bound.~Appendix~\ref{app:conc_ineq} covers the concentration results for estimates of the discretized model.~In Appendix~\ref{app:prop_evi_epe}, we derive bounds on inaccuracy that \evi~and \epe~injects into \algo~due to finite computation power.~Details of the experiments, the associated environments and additional simulation results are reported in Appendix~\ref{app:sim}.~Appendix~\ref{app:aux_res} derives some key results that are used in the proof of Lemma~\ref{lem:conc_ineq}.~Appendix~\ref{app:use_res} contains some known results that are used in this paper.

\section{General Results for MDPs}\label{app:gen_res}
Consider an MDP $\cM = (\cS, \cA, p, r)$ and a policy $\phi \in \Phi_{SD}$ that maps states in $\cS$ to actions in $\cA$.~We assume that the transition kernel $p$ satisfies Assumption~\ref{assum:unif_ergodic}.~Hence, there exists a unique invariant distribution $\mu\uc{\infty}_{\phi,p}$ for the controlled Markov process~(CMP) induced by the transition kernel $p$ under the application of policy $\phi$.~Under Assumption~\ref{assum:unif_ergodic}, there exists a solution to the following Poisson equation~\citep{hernandez2012further}:
\begin{align}
    J + h(s) = r(s,\phi(s)) + \int_{\cS}{h(s\up) p(s,\phi(s),ds\up)},~\forall s \in \cS.\label{eq:pois}
\end{align}
Specifically, $(J_\cM(\phi), h^\phi_\cM) \in \bR \times \bR^\cS$ satisfies \eqref{eq:pois}, where
\begin{align}
    J_\cM(\phi) &= \underset{T\to\infty}{\lim\inf}{\frac{1}{T} \bE\sqbr{\sum_{t=0}^{T-1}{r(s_t,\phi(s_t))} \mid s_0=s}} = \int_{\cS}{r(s,\phi(s)) \mu\uc{\infty}_{\phi,p}(ds)}, \label{def:Jphi}\\
   \mbox{ and } h^\phi_\cM(s\up) &= \sum_{t=0}^{\infty}{\int_{\cS}{r(s,\phi(s)) (\mu\uc{\infty}_{\phi,p} - \mu\uc{t}_{\phi,p,s\up})(ds)}},~\forall s\up \in \cS. \label{def:hphi}
\end{align}
Recall that $\mu\uc{t}_{\phi,p,s}$ denotes the distribution of $s_t$ when initial state is $s_0 = s$, where $\{s_t\}_t$ is the CMP induced by the transition kernel $p$ under the application of $\phi$.~$h^\phi_\cM$ is called the relative value function of $\phi$.

The following is popularly known as the average reward optimality equation~(AROE),
\begin{align*}
    &J + h(s) = \max_{a \in \cA}{\flbr{r(s,a) + \int_{\cS}{h(s\up) p(s, a, s\up) ds\up}}}, \mbox{ and}\\
    &h(s\lst) = 0,
\end{align*}
where $s\lst \in \cS$ is a designated state.~\citet{hernandez2012adaptive} shows that under Assumption~\ref{assum:unif_ergodic}, AROE has a solution.~A policy $\phi\ust$ is optimal if it satisfies the following,
\begin{align}
    \phi\ust(s) \in \underset{a \in \cA}{\arg\max}{\flbr{r(s,a) + \int_{\cS}{h^{\phi\ust}_\cM(s\up)p(s,a,s\up) d s\up}}},~\forall s \in \cS.
\end{align}
In that case, $J\ust_\cM = J_\cM(\phi\ust)$ and $h_\cM = h^{\phi\ust}_\cM$ solve AROE.

Denote the $t$-stage transition kernel under the application of policy $\phi$ by $p\uc{t}_\phi$, i.e.,
\begin{align}
    p\uc{t}_\phi(s,B) = \bP(s_{\tau+t} \in B \mid s_\tau = s, a_{t\up} = \phi(s_{t\up}), t\up = \tau, \tau + 1, \ldots, \tau+t-1),~t\in\bN,s \in \cS, B \in \cB_\cS,\tau \in \bN. \label{def:p_tstage}
\end{align}
Our next result shows that when $t$ is sufficiently large, then Assumption~\ref{assum:unif_ergodic} is equivalent to saying that $p\uc{t}_\phi$ has the ``contractive property,''~\eqref{def:contractive}. 

\begin{lemma}\label{lem:pn_contra}
    Consider an MDP $\cM = (\cS, \cA, p, r)$ such that $p$ satisfies Assumption~\ref{assum:unif_ergodic}. Then, for every policy $\phi \in \Phi_{SD}$ we have,
    \begin{align}
        \norm{p\uc{i}_\phi(s, \cdot) - p\uc{i}_\phi(s\up,\cdot)}_{TV} \leq 2 \alpha,~\forall s, s\up \in \cS, i \geq m\ust, \label{def:contractive}
    \end{align}
    where $p\uc{i}_\phi$ is the $i$-stage transition probability of the CMP induced by the transition kernel $p$ under the application of policy $\phi$ as defined in~\eqref{def:p_tstage}, and 
    \begin{align}
        m\ust := \ceil{\log_\frac{1}{\alpha}{(C)}} + 1. \label{def:mstar}
    \end{align}
    Conversely, if
    \begin{align*}
        \norm{p\uc{m}_\phi(s, \cdot) - p\uc{m}_\phi(s\up,\cdot)}_{TV} \leq 2 \alpha\up,~\forall s, s\up \in \cS,
    \end{align*}
    for some $m \in \bN$, then Assumption~\ref{assum:unif_ergodic} holds with $C = \frac{2}{\alpha\up}$ and $\alpha = {\alpha\up}^{\frac{1}{m}}$.
\end{lemma}
\begin{proof}
    We first note that $p\uc{i}_\phi(s, \cdot) = \mu\uc{i}_{\phi,p,s}$ for every $s \in \cS$.~Hence, for any $s, s\up \in \cS$,
    \begin{align*}
        \norm{p\uc{i}_\phi(s, \cdot) - p\uc{i}_\phi(s\up, \cdot)}_{TV} \leq \norm{\mu\uc{i}_{\phi,p,s} - \mu\uc{\infty}_{\phi,p}}_{TV} + \norm{\mu\uc{i}_{\phi,p,s\up} - \mu\uc{\infty}_{\phi,p}}_{TV}.
    \end{align*}
    Also, $C \alpha^i \leq \alpha$ for $i \geq \log_\frac{1}{\alpha}{(C)} + 1$.~Now, using Assumption~\ref{assum:unif_ergodic}, we have that when $i \geq m\ust$, then the following holds,
    \begin{align*}
        \norm{p\uc{i}_\phi(s, \cdot) - p\uc{i}_\phi(s\up, \cdot)}_{TV} &\leq \norm{\mu\uc{i}_{\phi,p,s} - \mu\uc{\infty}_{\phi,p}}_{TV} + \norm{\mu\uc{i}_{\phi,p,s\up} - \mu\uc{\infty}_{\phi,p}}_{TV}\\
        &\leq 2 \alpha.
    \end{align*}
    This concludes the proof of the first claim.
    
    Now, we prove the second claim. Consider the CMP that is described by the transition kernel $p$ and evolves under the application of the policy $\phi$. Consider two copies of this CMP, where these copies differ in the distribution of the initial state. Denote these distributions by $\mu\uc{0}_1$ and $\mu\uc{0}_2$. Denote the distributions of $s_i$ in the corresponding processes by $\mu\uc{i}_1$ and $\mu\uc{i}_2$, respectively. We show the following:
    \begin{align}\label{ineq:geom_close}
        \norm{\mu\uc{i}_{1} - \mu\uc{i}_{2}}_{TV} \leq \tilde{C} \cdot \tilde{\alpha}^i \norm{\mu\uc{0}_{1} - \mu\uc{0}_{2}}_{TV},~\forall i \in \bN,
    \end{align}
    where $\tilde{C} = \frac{1}{\alpha\up}$ and $\tilde{\alpha} = {\alpha\up}^{\frac{1}{m}}$.~The claim then follows by letting $\mu\uc{0}_{1} = \delta_s$ and $\mu\uc{0}_{2} = \mu\uc{\infty}_{\phi,p}$.~Note that,
    \begin{align}
        \norm{\mu\uc{m}_{1} - \mu\uc{m}_{2}}_{TV} &= 2~ \sup_{A \subseteq \cS}{\flbr{(\mu\uc{m}_{1} - \mu\uc{m}_{2})(A)}} \notag\\
        &= 2~\sup_{A \subseteq \cS}{\flbr{\int_{\cS}{p\uc{m}_\phi(s,A) ~d(\mu\uc{0}_{1} - \mu\uc{0}_{2})(s)}}} \notag\\
        &\leq \sup_{\substack{A \subseteq \cS \\ s,s\up \in \cS}}{\flbr{p\uc{m}_\phi(s,A) - p\uc{m}_\phi(s\up,A)}} \norm{\mu\uc{0}_{1} - \mu\uc{0}_{2}}_{TV} \notag\\
        &\leq \alpha\up \norm{\mu\uc{0}_{1} - \mu\uc{0}_{2}}_{TV}. \label{bdd:mu1m-mu2m}
    \end{align}
    Also, note that for any $i \in \bN$,
    \begin{align}
        \norm{\mu\uc{i}_{1} - \mu\uc{i}_{2}}_{TV} &= 2~ \sup_{A \subseteq \cS}{\flbr{(\mu\uc{i}_{1} - \mu\uc{i}_{2})(A)}} \notag\\
        &= 2~\sup_{A \subseteq \cS}{\flbr{\int_{\cS}{p(s,\phi(s),A) ~d(\mu\uc{i-1}_{1} - \mu\uc{i-1}_{2})(s)}}} \notag\\
        &\leq \sup_{\substack{A \subseteq \cS \\ s,s\up \in \cS}}{\flbr{p(s,\phi(s),A) - p(s\up,\phi(s\up),A)}} \norm{\mu\uc{i-1}_{1} - \mu\uc{i-1}_{2}}_{TV} \notag\\
        &\leq \norm{\mu\uc{i-1}_{1} - \mu\uc{i-1}_{2}}_{TV}, \label{bdd:mu1t-mu2t}
    \end{align}
    where the first step follows from the definition of the total variation norm, while the third step follows from Lemma~\ref{lem:bdd_dotdifLv}.~Combining \eqref{bdd:mu1m-mu2m} and \eqref{bdd:mu1t-mu2t}, we can write
    \begin{align*}
        \norm{\mu\uc{i}_{1} - \mu\uc{i}_{2}}_{TV} &\leq {\alpha\up}^{\floor{\frac{i}{m}}} \norm{\mu\uc{0}_{1} - \mu\uc{0}_{2}}_{TV}\\
        &\leq \frac{1}{\alpha\up}\br{{\alpha\up}^{\frac{1}{m}}}^i \norm{\mu\uc{0}_{1} - \mu\uc{0}_{2}}_{TV},~\forall i \in \bN.
    \end{align*}
    This concludes the proof of the lemma.
\end{proof}

Consider two CMPs $\{s_{1,i}\}$ and $\{s_{2,i}\}$, both of which are induced by $\phi$ operating on the MDP $\cM$ that has transition kernel $p$. Their initial state distributions are $\mu\uc{0}_1$ and $\mu\uc{0}_2$ respectively. Next, we derive an upper-bound on the cumulative sum of distances of the distributions of $s_{1,i}$ and $s_{2,i}$.
\begin{lemma}\label{lem:sum_tv_dist}
    Consider an MDP $\cM = (\cS, \cA, p, r)$ that satisfies Assumption~\ref{assum:unif_ergodic}, and a policy $\phi \in \Phi_{SD}$.~Let $\{s_{1,i}\}$ and $\{s_{2,i}\}$ be two CMPs induced by $\phi$ when it is applied to $\cM$. Let $\mu\uc{i}_1$ and $\mu\uc{i}_2$ denote the distributions of $s_{1,i}$ and $s_{2,i}$, respectively.~Then,
    \begin{align*}
        \sum_{i=0}^{\infty}{\norm{\mu\uc{i}_1 - \mu\uc{i}_2}_{TV}} \leq \frac{m\ust}{1 - \alpha} \norm{\mu\uc{0}_{1} - \mu\uc{0}_{2}}_{TV},
    \end{align*}
    where $m\ust$ is as defined in~\eqref{def:mstar}.
\end{lemma}
\begin{proof}
    From Lemma~\ref{lem:pn_contra}, we have that,
    \begin{align}\label{ineq:pn_contra}
        \norm{\mu\uc{i}_{1} - \mu\uc{i}_{2}}_{TV} \leq \alpha \norm{\mu\uc{0}_{1} - \mu\uc{0}_{2}}_{TV}, \mbox{ for } i \geq m\ust.
    \end{align}
    Also, for any $i \in \bN$ we have,
    \begin{align*}
        \norm{\mu\uc{i}_{1} - \mu\uc{i}_{2}}_{TV} &= 2~ \sup_{A \subseteq \cS}{\flbr{(\mu\uc{i}_{1} - \mu\uc{i}_{2})(A)}} \\
        &= 2~\sup_{A \subseteq \cS}{\flbr{\int_{\cS}{p(s,\phi(s),A) ~d(\mu\uc{i-1}_{1} - \mu\uc{i-1}_{2})(s)}}}\\
        &\leq \sup_{\substack{A \subseteq \cS \\ s,s\up \in \cS}}{\flbr{p(s,\phi(s),A) - p(s\up,\phi(s\up),A)}} \norm{\mu\uc{i-1}_{1} - \mu\uc{i-1}_{2}}_{TV} \\
        &\leq \norm{\mu\uc{i-1}_{1} - \mu\uc{i-1}_{2}}_{TV},
    \end{align*}
    where the first step follows from the definition of the total variation norm, and the third step follows from Lemma~\ref{lem:bdd_dotdifLv}.~Hence, 
    \begin{align}\label{ineq:non_expan}
        \norm{\mu\uc{i}_{1} - \mu\uc{i}_{2}}_{TV} \leq \norm{\mu\uc{0}_{1} - \mu\uc{0}_{2}}_{TV},~\forall i \in \bN
    \end{align}
    Using \eqref{ineq:pn_contra} iteratively, and~\eqref{ineq:non_expan}, we can write,
    \begin{align*}
        \sum_{t=0}^{\infty}{\norm{\mu\uc{i}_1 - \mu\uc{i}_2}_{TV}} &= \sum_{m=0}^{m\ust-1}{\sum_{i=0}^{\infty}{\norm{\mu\uc{m+i\cdot m\ust}_{1} - \mu\uc{m+i\cdot m\ust}_{2}}_{TV}}}\\
        &\leq \frac{m\ust}{1 - \alpha} \norm{\mu\uc{0}_{1} - \mu\uc{0}_{2}}_{TV}.
    \end{align*}
    where $m\ust = \ceil{\log_\frac{1}{\alpha}{(C)}} + 1$.~This concludes the proof.
\end{proof}
We now derive an upper-bound on the span of the relative value function $h^\phi_\cM$
~\eqref{def:hphi} associated with a policy $\phi\in\Phi_{SD}$.
\begin{lemma}[Bound on the span of relative value function]\label{lem:bdd_rvf_spn}
    Consider an MDP $\cM = (\cS, \cA, p, r)$ such that $p$ satisfies Assumption~\ref{assum:unif_ergodic}. For any policy $\phi \in \Phi_{SD}$, the span of the corresponding relative value function $h^\phi_\cM$~\eqref{def:hphi} can be bounded as,
    \al{
    \spn{h^\phi_\cM} \le \frac{m\ust\spn{r}}{1 - \alpha},
    }
    where $m\ust$ is as defined in~\eqref{def:mstar}.
\end{lemma}
\begin{proof}
    From the definition of $h^\phi_\cM$~\eqref{def:hphi} we obtain,
    \begin{align}
        \spn{h^\phi_\cM} &= \spn{\sum_{t=0}^{\infty}{\int_{\cS}{r(s,\phi(s))\br{\mu\uc{\infty}_{\phi,p} - \mu\uc{t}_{\phi,p,\cdot}}(ds)}}} \notag\\
        & \leq \sum_{t=0}^{\infty}{\spn{\int_{\cS}{r(s,\phi(s))\br{\mu\uc{\infty}_{\phi,p} - \mu\uc{t}_{\phi,p,\cdot}}(ds)}}} \notag\\
        & \leq \frac{1}{2} \sum_{t=0}^{\infty}{\max_{s}\norm{\mu\uc{\infty}_{\phi,p} - \mu\uc{t}_{\phi,p,s}}_{TV}} \spn{r}, \label{lem:bd_spn:ineq:1}
    \end{align}
    where the first inequality follows since span is a seminorm~\citep{puterman2014markov}, while the second inequality follows from Lemma~\ref{lem:bdd_dotdifLv}.~In Lemma~\ref{lem:sum_tv_dist} we let $\mu\uc{0}_1 = \mu\uc{\infty}_{\phi,p}$ and $\mu\uc{0}_2 = \delta_s$, where $\delta_s$ is the Dirac measure on $(\cS,\cB_\cS)$ centered at $s$, and get the following, 
    \begin{align*}
        \frac{1}{2} \sum_{t=0}^{\infty}{\max_{s}\norm{\mu\uc{\infty}_{\phi,p} - \mu\uc{t}_{\phi,p,s}}_{TV}} \spn{r} &\leq \frac{m\ust\spn{r}}{1 - \alpha}.
    \end{align*}
    This concludes the proof.
\end{proof}

\begin{lemma}[Bound on the span of policy evaluation iterates]\label{lem:bdd_pval_spn}
    Consider an MDP $\cM = (\cS, \cA, p, r)$ such that $p$ satisfies Assumption~\ref{assum:unif_ergodic}, and consider the policy evaluation algorithm applied to obtain the average reward of a policy $\phi \in \Phi_{SD}$ on $\cM$ i.e.,
    \begin{align}
        V^\phi_0(s) &= 0,\notag\\
        V^\phi_{i+1}(s) &= r(s,\phi(s)) + \int_{\cS}{p(s,\phi(s),s\up) V^\phi_i(s\up) ds\up},~i=1,2,\ldots.\label{def:epe_true}
    \end{align}
    We have,
    \al{
    \spn{V^\phi_i} \leq \frac{m\ust + 1}{1 - \alpha},
    }
where $m\ust = \ceil{\log_{\frac{1}{\alpha}}(C)} + 1$.
\end{lemma}
\begin{proof}
    Since Assumption~\ref{assum:unif_ergodic} holds, Lemma~\ref{lem:pn_contra} gives us the following,
    \begin{align*}
        \norm{p\uc{m\ust}_\phi(s, \cdot) - p\uc{m\ust}_\phi(s\up,\cdot)}_{TV} \leq 2 \alpha,~\forall s, s\up \in \cS,
    \end{align*}
    where $p\uc{m}_\phi$~\eqref{def:p_tstage} is the $m$-step transition kernel of the CMP induced by the transtion kernel $p$ under the application of policy $\phi$.~Also, note that
    \begin{align*}
        V^\phi_{i+m\ust}(s) = \sum_{j=0}^{m\ust}{\bE\sqbr{r(s_{i+j}, \phi(s_{i+j})) \mid s_i = s}} + \int_{\cS}{p\uc{m\ust}_\phi(s,s\up) V^\phi_i(s\up) ds\up}.
    \end{align*}
    Hence,
    \begin{align*}
        \spn{V^\phi_{i+m\ust}} &\leq \spn{\sum_{j=0}^{m\ust}{\bE\sqbr{r(s_{i+j}, \phi(s_{i+j})) \mid s_i = s}}} + \spn{\int_{\cS}{p\uc{m\ust}_\phi(s,s\up) V^\phi_i(s\up) ds\up}}\\
        &\leq m\ust + 1 + \frac{1}{2} \spn{V^\phi_i}  \norm{p\uc{m\ust}_\phi(s, \cdot) - p\uc{m\ust}_\phi(s\up,\cdot)}_{TV} \\
        &\leq m\ust + 1 + \alpha \spn{V^\phi_i},
    \end{align*}
    where the second inequality follows from Lemma~\ref{lem:bdd_dotdifLv}.~Using the above inequality, we have that for every $k \leq m\ust$,
    \begin{align*}
        \spn{V^\phi_{i\cdot m\ust + k}} &\leq (m\ust + 1) \sum_{j=0}^{i-1}{\alpha^j} + \alpha^i \spn{V^\phi_{k}} \\
        &\leq (m\ust + 1) \sum_{j=0}^{i-1}{\alpha^j} + m\ust \alpha^i \\
        &\leq \frac{m\ust + 1}{1 - \alpha}.
    \end{align*}
    This concludes the proof.
\end{proof}

\subsection{Proof of Lemma~\ref{lem:gap_phi}}
\begin{proof}
    Using the definition of $\gap{s,\phi(s)}$~\eqref{def:subgap}, we obtain that,
    \al{
        \int_{\cS}{\gap{s,\phi(s)}~ \mu\uc{\infty}_{\phi,p}(s)~ ds} &= \int_{\cS}{\br{J\ust_{\cM} + h_{\cM}(s) - r(s,\phi(s)) - \int_{\cS}{h_{\cM}(s\up)~p(s,\phi(s),s\up)~ ds\up}} \mu\uc{\infty}_{\phi,p}(s)~ ds} \notag\\
        &= J\ust_{\cM}\int_{\cS}{\mu\uc{\infty}_{\phi,p}(s)~ ds} + \int_{\cS}{h_{\cM}(s) \mu\uc{\infty}_{\phi,p}(s)~ ds} - \int_{\cS}{r(s,\phi(s)) \mu\uc{\infty}_{\phi,p}(s)~ ds} \notag\\
        &- \int_{\cS}{\br{\int_{\cS}{h_{\cM}(s\up)~p(s,\phi(s),s\up)~ ds\up}} \mu\uc{\infty}_{\phi,p}(s)~ ds} \notag\\
        &= J\ust_{\cM} + \int_{\cS}{h_{\cM}(s) \mu\uc{\infty}_{\phi,p}(s)~ ds} - J_\cM(\phi) \notag\\
        &- \int_{\cS}{h_{\cM}(s\up) \br{\int_{\cS}{p(s,\phi(s),s\up) \mu\uc{\infty}_{\phi,p}(s)~ds}}ds\up} \notag\\
        &= J\ust_{\cM} - J_\cM(\phi) + \int_{\cS}{h_{\cM}(s) \mu\uc{\infty}_{\phi,p}(s)~ ds} - \int_{\cS}{h_{\cM}(s) \mu\uc{\infty}_{\phi,p}(s)~ ds}\notag\\
        &= \Delta(\phi),
    }
    where the third equality follows from \eqref{def:Jphi} and the fourth equality follows from the property of the stationary distribution. This concludes the proof.
\end{proof}
\section{Properties of Extended MDP}\label{app:prop_emdp}
We present three results in this section.~We begin by showing that extended MDPs constructed by \algo~are optimistic, i.e., on the set $\cG_1$~\eqref{def:G_1}, the optimal average reward of the extended MDP $\cM^+_t$ is greater than or equal to the optimal average reward of the true MDP for all $t \in \{0,1,\ldots,T-1\}$.~Next, we show that the span of the \epe~iterates~\eqref{iter:v_epe} for the extended MDP $\cM^+_t$ and any $\phi \in \Phi_t$ are bounded for all $t \in \{0,1,\ldots,T-1\}$.~Lastly, we derive an upper-bound on the average reward of policy $\phi \in \Phi_t$ evaluated on MDP $\cM^+_t$ for every $t \in \{0,1,\ldots,T-1\}$.
\begin{lemma}[Optimism]\label{lem:optimism}
    On the set $\cG_1$, we have,
    \al{
        J\ust_{\cM^+_t} \geq J\ust_\cM, \mbox{ for every } t \in \{0,1,\ldots, T-1\},\label{ineq:optimism}
    }
    where $J\ust_{\cM^+_t}$ is the optimal average reward of the extended MDP~$\cM^+_t$, and $J\ust_\cM$ is the optimal average reward of the MDP $\cM$.
\end{lemma}
\begin{proof}
    Consider the value iteration algorithm applied to the MDP $\cM$. For every $s \in \cS$,
    \begin{align}
        V_0(s) &= 0, \notag\\
        V_{n+1}(s) &=  \max_{a \in \cA}\Big\{r(s,a) + \int_{\cS}{p(s, a, s\up) V_n(s\up) d s\up}\Big\},~\forall n \in \bN. \label{iter:vi}
    \end{align}
    We assumed that $\cM$ is uniformly ergodic in Assumption~\ref{assum:unif_ergodic}, and hence the following value iteration algorithm converges, i.e., $\lim_{n \to \infty}{\spn{V_{n+1} - V_n}} = J\ust_\cM$.~Also, it follows from~\citep{hernandez2012adaptive} that $\lim_{n \to \infty} |V_n(s) - (n J\ust_\cM + h_\cM(s)) |=0$ for every $s \in \cS$.~Since we have shown in Lemma~\ref{lem:bdd_rvf_spn} that $h_\cM$ is bounded, it then follows that 
    \al{
    \lim_{n \to \infty}{\frac{1}{n}V_n(s)} = J\ust_\cM,~\forall s\in \cS. \label{eq:V/n=j}
    }
    
    We will prove that $V_n(s\up) \leq v_n(s)$ for every $n \in \bN$, $s \in \cS_t$ and $s\up \in q\inv(s)$. We prove this via induction. The base case, i.e. $n=0$ is seen to hold trivially. Next, assume that the following hold for all $i \in [n]$, where $n \in \bN$,
    \begin{align}\label{ineq:opt_ind_hyp}
        v_i(s) &\geq V_i(s\up),~\forall s \in \cS_t,~\forall s\up \in q\inv(s).
    \end{align}
    Consider a state-action pair $(s,a) \in \cS \times \cA$ and let $\ts \in \cS_t$ such that $s \in q\inv(\ts)$.~Then,
    \begingroup
    \allowdisplaybreaks
    \begin{align}
        r(s,a) + \int_{\cS}{p(s,a,s\up) V_n(s\up) ds\up} &\leq r(s,a) + \sum_{s\up \in \cS_t}{\wp_{\cS \times \cA \to \cS_t,p}(s,a,s\up) v_n(s\up)} \notag\\
        &\leq r(q(\zeta)) + L_r \diamc{\zeta} + \sum_{s\up \in \cS_t}{\wp_{\cS \times \cA \to \cS_t,p}(s,a,s\up) v_n(s\up)} \notag\\
        & \leq \max_{\substack{\ta \in A_t(\bar{s})\\ \te \in \cC_t}}{\flbr{\tilde{r}_t(\ts,\ta) + \sum_{s\up \in \cS_t}{\te(\ts, \ta, s\up) v_n(s\up)}}} \notag\\
        &= v_{n+1}(\ts),\label{ineq:opt1}
    \end{align}
    \endgroup
    where the first inequality follows from~\eqref{ineq:opt_ind_hyp}, the second inequality follows from Assumption~\ref{assum:lip}~(i), while the third inequality follows from the definition of the set $\cG_1$.~Since we have shown the above inequality for an arbitrary action $a$, we get,
    \begin{align}
        V_{n+1}(s) &= \max_{a \in \cA}{\flbr{r(s,a) + \int_{\cS}{p(s,a,s\up) V_n(s\up) ds\up}}} \notag\\
        &\leq v_{n+1}(\ts).
    \end{align}
    This completes the induction argument. The proof is then completed by dividing both sides of this inequality by $n$ and then taking limit $n \to \infty$.
\end{proof}

\begin{lemma}\label{lem:bd_span_epe}
    Let $t \in \{0,1,\ldots,T-1\}$.~Consider the extended MDP $\cM^{+}_t$, a policy $\phi \in \Phi_t$ and the corresponding \epe~\eqref{algo:epe}~iterates:
    \begin{align}
        v^{\phi,t}_0(s) &= 0, \notag\\
        v^{\phi,t}_{n+1}(s) &= \max_{\te \in \cC_t} \flbr{\tilde{r}_t(s,\phi(s)) + \sum_{s\up \in \cS_t}{\te(s,\phi(s),s\up) v^{\phi,t}_n(s\up)}},~\forall s \in \cS_t, n \in \bN. \label{iter:v_epe}
    \end{align}
    On the set $\cG_1$, we have 
    \nal{
    \spn{v^{\phi,t}_n} \leq C_v,~\forall n \in \bN, t \in \bN,
    }
    where,
    \begin{align}
        C_v &:= \max{\flbr{\frac{\ovl{m} (\ovl{m} + 5)}{2} + \frac{3}{C \alpha^{\ovl{m}+1}} + \frac{4 \tilde{m}}{1 - \alpha}, \frac{\ceil{\log_{\br{\frac{1}{\alpha}}^{\tilde{m}\inv}}{\br{\frac{2}{\alpha}}} + 1}}{1 - \alpha^{\tilde{m}\inv}}}},\label{def:Cv}\\
        \ovl{m} &:= \ceil{\log_{\frac{1}{\alpha}}{\br{\frac{2C}{\kappa} \br{\frac{C_\eta \tilde{m} \sqrt{d}}{1-\alpha}}^{d_\cS}}}},\label{def:m_bar} \mbox{ and } \\
        \tilde{m} &:= \ceil{\log_{\frac{1}{\alpha}}\br{\frac{2C}{3\alpha - 1}}}. \label{def:tm}
    \end{align}
    $C$ and $\alpha$ are as in Assumption~\ref{assum:unif_ergodic}.
\end{lemma}
\begin{proof}
    We first note that $v^{\phi,t}_n(s)$ is the optimal value of the expected reward for the extended MDP $\cM_t^+$ that is accumulated during the first $n$ steps when the process starts in state $s$. The first component of the extended action of the extended MDP is taken to be policy $\phi$ and doesn't need to be optimized, while the second component is the transition kernel that maximizes the r.h.s. of \eqref{iter:v_epe} in every step $i \in \{0,1,\ldots,n-1\}$.~We consider the following two cases separately. 

    \textbf{Case 1:} When,
    \al{
    \max_{s \in \cS_t}\diamc{q_t\inv(s,\phi(s))}\geq \frac{1 - \alpha}{2 (3(1 + L_p) + C_p) \br{\tilde{m}+1}}. \label{cond:1}
    }
    Let $\zeta$ be the cell with the largest diameter from the set $\{q_t\inv(s,\phi(s)):~s\in \cS_t\}$.~We first show that $\{s_i\}_{i=0}^{\infty}$, the CMP induced by the transition kernel $p$ under the application of policy $\phi$, hits $\pi_\cS(\zeta)$ within 
    \begin{align*}
        \frac{\ovl{m} (\ovl{m} + 5)}{2} + \frac{3}{C \alpha^{\ovl{m}+1}}
    \end{align*}
    steps in expectation, where $\ovl{m}$ is as defined in~\eqref{def:m_bar}.~From Assumption~\ref{assum:unif_ergodic}, Assumption~\ref{assum:statn_dist} and~\eqref{cond:1}, we have that for any $s\up \in \cS$,
    \begin{align*}
        \mu\uc{i}_{\phi,p,s\up}(\pi_\cS(\zeta)) \geq \frac{1}{2}\mu\uc{\infty}_{\phi,p}(\pi_\cS(\zeta)), \mbox{ and } \mu\uc{i}_{\phi,p,s\up}(\pi_\cS(\zeta)) \leq \frac{3}{2}\mu\uc{\infty}_{\phi,p}(\pi_\cS(\zeta))~\forall i \geq \ovl{m}.
    \end{align*}
    Now, consider another process $\{x_i\}_{i=0}^{\infty}$ that is independent across time; $x_i$ assumes the value $1$ with a probability $\mu\uc{i}_{\phi,p,s\up}(\pi_\cS(\zeta))$, and $0$ with a probability $1 - \mu\uc{i}_{\phi,p,s\up}(\pi_\cS(\zeta))$.~Define the following random variables $T\uc{x}_{\{1\}}$ and $T\uc{s}_{\pi_\cS(\zeta),s\up}$,
    \begin{align*}
        T\uc{x}_{\{1\}} &:= \inf{\{i\geq 0 \mid x_i = 1\}}, \mbox{ and}\\
        T\uc{s}_{\pi_\cS(\zeta),s\up} &:= \inf{\{i\geq 0 \mid s_i \in \pi_\cS(\zeta), s_0 = s\up\}}.
    \end{align*}
    We note that the distributions of $T\uc{x}_{\{1\}}$ and $T\uc{s}_{\pi_\cS(\zeta),s\up}$ are identical, so that $\bE\sqbr{T\uc{x}_{\{1\}}} = \bE\sqbr{T\uc{s}_{\pi_\cS(\zeta),s\up}}$.~We derive an upper-bound on $\bE\sqbr{T\uc{x}_{\{1\}}}$, and this would also serve as the upper-bound on $\bE\sqbr{T\uc{s}_{\pi_\cS(\zeta),s\up}}$.~We have,
    \begin{align*}
        \bE\sqbr{T\uc{x}_{\{1\}}} &= \sum_{i=0}^{\infty}{i \cdot \mu\uc{i}_{\phi,p}(\pi_\cS(\zeta)) \prod_{j=0}^{i-1}{\br{1 - \mu\uc{j}_{\phi,p,s}(\pi_\cS(\zeta))}}} \\
        &\leq \frac{\ovl{m}(\ovl{m} -1)}{2} + \sum_{i=\ovl{m}}^{\infty}{\frac{3i}{2} \mu\uc{\infty}_{\phi,p}(\pi_\cS(\zeta)) \prod_{j=\ovl{m}}^{i-1}{\br{1 - \frac{1}{2}\mu\uc{\infty}_{\phi,p}(\pi_\cS(\zeta))}}} \\
        &\leq \frac{\ovl{m}(\ovl{m} -1)}{2} + \frac{3}{2}\mu\uc{\infty}_{\phi,p}(\pi_\cS(\zeta))\sum_{i=0}^{\infty}{i \br{1 - \frac{1}{2}\mu\uc{\infty}_{\phi,p}(\pi_\cS(\zeta))}}^i + \frac{3 \ovl{m}}{2}\mu\uc{\infty}_{\phi,p}(\pi_\cS(\zeta))\sum_{i=0}^{\infty}{\br{1 - \frac{1}{2}\mu\uc{\infty}_{\phi,p}(\pi_\cS(\zeta))}^i} \\
        &\leq \frac{\ovl{m} (\ovl{m} + 5)}{2} + \frac{6}{\mu\uc{\infty}_{\phi,p}(\pi_\cS(\zeta))}.
    \end{align*}
    Furthermore, from Assumption~\ref{assum:statn_dist}, and since $\bE\sqbr{T\uc{x}_{\{1\}}}=\bE\sqbr{T\uc{s}_{\pi_\cS(\zeta),s\up}}$, we get, 
    \begin{align*}
        \bE\sqbr{T\uc{s}_{\pi_\cS(\zeta),s\up}} \leq \frac{\ovl{m} (\ovl{m} + 5)}{2} + \frac{6}{\kappa} \br{\frac{\sqrt{d}}{\diamc{\zeta}}}^{d_\cS}.
    \end{align*}
    From \eqref{cond:1} we can write,
    \begin{align*}
        \bE\sqbr{T\uc{s}_{\pi_\cS(\zeta),s\up}} &\leq \frac{\ovl{m} (\ovl{m} + 5)}{2} + \frac{6}{\kappa} \br{\frac{(3(1 + L_p) + C_p) \sqrt{d} (\tilde{m} + 1)}{1 - \alpha}}^{d_\cS}  \\
        &\leq \frac{\ovl{m} (\ovl{m} + 5)}{2} + \frac{3}{C \alpha^{\ovl{m}+1}}.
    \end{align*}
    Next, consider two states $\ovl{s} \in \cS_t$, and $\tilde{s} \in q\inv(\ovl{s})$.~We note that on the set $\cG_1$, for the extended MDP $\cM_t^+$ whenever the state is $\ovl{s}$, there is an extended action such that the next state transition distribution is $p(\tilde{s},\phi(\tilde{s}),\cdot)$.~Hence, on the set $\cG_1$, there is a sequence of extended actions such that starting from any state, in expectation, within $\frac{\ovl{m} (\ovl{m} + 5)}{2} + \frac{3}{C \alpha^{\ovl{m}+1}}$ steps the process hits $q(\pi_\cS(\zeta))$ where $\pi_\cS(\zeta)$ is the $\cS$-projection of $\zeta$, the largest cell in $\{q_t\inv(s,\phi(s)):~s\in \cS_t\}$.
    
    Now, consider the process $\{s_t\}$ associated with the extended MDP, in which the initial state is $s \in \cS_t$.~We claim that for any state $s\up$, there exists a sequence of extended actions where the first components of the extended actions are chosen by $\phi$ such that $s\up$ can be reached in $\frac{2}{(3(1 + L_p) + C_p) ~\diamc{q_t\inv(s,\phi(s))}}$ steps in expectation.~This is true because there is a transition kernel in $\cC_t$ that assigns at least $\frac{3(1 + L_p) + C_p}{2} \diamc{q_t\inv(s,\phi(s))}$ transition probability to $s\up$ when the current state is from $s$.~To summarize, starting from any state using a sequence of actions the state process can reach $q(\zeta)$ in $\frac{\ovl{m} (\ovl{m} + 5)}{2} + \frac{3}{C \alpha^{\ovl{m}+1}}$ steps in expectation, and from $q(\zeta)$, again it can reach any other state using a sequence of actions in $\frac{2}{(3(1 + L_p) + C_p) \diamc{q_t\inv(s,\phi(s))}}$.~Therefore, there cannot be state $s\up$ such that 
    \nal{
    \max_{s \in \cS_t}{v^{\phi,t}_n(s)} > v^{\phi,t}_n(s\up) + \frac{\ovl{m} (\ovl{m} + 5)}{2} + \frac{3}{C \alpha^{\ovl{m}+1}} + \frac{2}{(3(1 + L_p) + C_p) \diamc{\zeta}}.
    }
    Now, from the lower-bound on $\diamc{\zeta}$~\eqref{cond:1}, we obtain that
    \begin{align}
        \spn{v^{\phi,t}_n} \leq \frac{\ovl{m} (\ovl{m} + 5)}{2} + \frac{3}{C \alpha^{\ovl{m}+1}} + \frac{4 \tilde{m}}{1 - \alpha}. \label{ub:case_1}
    \end{align}

    \textbf{Case 2:}~In this case, we have that
    \begin{align}
        \max{\{\diamc{q_t\inv(s,\phi(s))} : s \in \cS_t\}} < \frac{1 - \alpha}{2 (3(1 + L_p) + C_p) \br{\tilde{m}+1}}. \label{cond:2}
    \end{align}
    Let $\bar{\phi} \in \Phi_{SD}$ be the extension of policy $\phi \in \Phi_t$ such that
    \begin{align*}
        \bar{\phi}(s) = \phi(q(\pi_\cS(\zeta))), \mbox{ for ever } s \in \pi_\cS(\zeta), \mbox{ for every } \pi_\cS(\zeta) \in \cQ_t.
    \end{align*}
    Claim:~We claim that there is a sequence of extended actions for the extended MDP $\cM^+_t$ such that the first components of the extended actions are governed by $\phi$ and on the set $\cG_1$, the $m$-step state transition kernel prescribed by the sequence of extended actions is the same as the discretization of the $m$-step composition of true transition kernel induced under application of policy $\bar{\phi}$. Let the state process of the extended MDP be denoted by $\{\tilde{s}_i\}$ and let the state process of the extended MDP be denoted by $\{s_i\}$. Then, mathematically, our claim says that there exists a sequence of probability kernels $\{\tilde{p}_i \in \cC_t: i \in \{1,2,\ldots\}\}$ such that
    \begin{align*}
        \bP(\tilde{s}_i = s\up \mid \tilde{s}_0 = s, \tilde{p}, \phi) = \bP(s_i \in q_t\inv(s\up)\mid s_0 = s, \bar{\phi}),~\forall s, s\up \in \cS_t,
    \end{align*}
    where $\bP$ denotes the joint probability distribution of the processes $\{\tilde{s}_i\}$ and $\{s_i\}$, condition on $\tilde{p}$ and $\phi$ implies that the extended actions are governed by $\tilde{p}$ and $\phi$.~Similarly, ~condition on $\bar{\phi}$ implies that the actions are governed by $\bar{\phi}$.~We show this using mathematical induction. The base cases follow from Lemma~\ref{lem:conc_ineq}.~Let us assume that for every $s, s\up \in \cS_t$ and for every $j \in \{1, 2, \ldots i\}$,
    \begin{align*}
        \bP(\tilde{s}_j = s\up \mid \tilde{s}_0 = s, \tilde{p}, \phi) = \bP(s_j \in q_t\inv(s\up)\mid s_0 = s, \bar{\phi}).
    \end{align*}
    See that
    \begin{align*}
        \bP(\tilde{s}_{i+1} = s\up \mid \tilde{s}_0 = s, \tilde{p}, \phi) &=  \sum_{\ts \in \cS_t}{\bP(\tilde{s}_{i+1} = s\up \mid \tilde{s}_i = \ts, \tilde{p}, \phi) \bP(\tilde{s}_i = \ts \mid \tilde{s}_0 = s, \tilde{p}, \phi)} \\
        &= \sum_{\ts \in \cS_t}{\tilde{p}_{i+1}(\ts,\phi(\ts),s\up) \bP(s_i = q\inv_t(\ts) \mid s_0 = s, \bar{\phi})}.
    \end{align*}
    Here, we note that for every $s \in \cS \times \cA$, there is a kernel $\te_s \in \cC_t$ such that $\te_s(q\inv_t(s, \phi(s)), s\up) = p(s,\bar{\phi}(s), q\inv_t(s\up))$ for every $s\up \in \cS_t$. As the set $\cC_t$ is convex, for any probability measure $\nu$ on $(\cS,\cB_\cS)$,
    \begin{align*}
        \int_{\cS}{\te_s(\ts,\phi(\ts), s\up) d\nu(s)} \in \cC_t.
    \end{align*}
    Taking $\nu$ to be a measure that satisfies $\nu(B) = \bP(s_i \in B \mid s_i \in q\inv_t(\ts))$ for every $B \in \cB_\cS$, we get that
    \begin{align*}
        \int_{\cS}{\te_s(\ts,\phi(\ts), s\up) d\nu(s)} = \bP(s_{i+1} \in q\inv_t(s\up) \mid s_i \in q\inv_t(\ts)).
    \end{align*}
    Taking $\tilde{p}_{i+1}(\ts, \phi(\ts), \cdot) = \int_{\cS}{\te_s(\ts,\phi(\ts), \cdot) d\nu(s)}$, we get that
    \begin{align*}
        \bP(\tilde{s}_{i+1} = s\up \mid \tilde{s}_0 = s, \tilde{p}, \phi) &= \sum_{\ts \in \cS_t}{\tilde{p}_{i+1}(\ts,\phi(\ts),s\up) \bP(s_i = q\inv_t(\ts) \mid s_0 = s, \bar{\phi})}\\
        &= \sum_{\ts \in \cS_t}{\bP(s_{i+1} \in q\inv_t(s\up) \mid s_i \in q\inv_t(\ts)) \bP(s_i = q\inv_t(\ts) \mid s_0 = s, \bar{\phi})} \\
        &= \bP(s_{i+1} \in q\inv_t(s\up) \mid s_0 = s, \bar{\phi}).
    \end{align*}
    This completes the proof of our claim.

    From \eqref{cond:2}, we have that for any $\te \in \cC_t$,
    \begin{align*}
        \max_{s\in \cS_t}{\norm{\te(s,\phi(s),\cdot) - \tilde{p}_i(s,\phi(s),\cdot)}_1} \leq \frac{1 - \alpha}{2\tilde{m}},~\forall s \in \cS_t, s\up \in q_t\inv(s).
    \end{align*}
    Define the discretization of the $m$-step transition kernel under the application of policy $\bar{\phi}$ as follows:
    \begin{align*}
        \wp_{t,\phi}\uc{m}(s,s\up) := p_\phi\uc{m}(s,q\inv_t(s\up)),~\forall s \in \cS, s\up \in \cS_t.
    \end{align*}
    Let $\te_\phi\uc{m}$ denote the $m$-step transition kernel of the CMP induced by $\te$ under application of policy $\phi$.~From the previous claim and Lemma~\ref{lem:diff_kern_comp}, we have that
    \begin{align}
        \norm{\wp_{t,\phi}\uc{\tilde{m}}(s,\cdot) - \te_\phi\uc{\tilde{m}}(s,\cdot)}_1 \leq \frac{1 - \alpha}{2}, \label{diff:te_p}
    \end{align}
    where $p_\phi\uc{m}$ is defined in~\eqref{def:p_tstage}.~Also, observe that
    \begin{align}
        \max_{s,s\up\in \cS_t}{\norm{\wp\uc{\tilde{m}}_{t,\phi}(s,\cdot) - \wp\uc{\tilde{m}}_{t,\phi}(s\up,\cdot)}_1} \leq \frac{3\alpha - 1}{2}. \label{diff:pmpm}
    \end{align}
    Hence, combining \eqref{diff:te_p} and \eqref{diff:pmpm}, we have that for any $\te \in \cC_t$,
    \begin{align*}
        \max_{s,s\up\in \cS_t}{\norm{\te\uc{\tilde{m}}_\phi(s,\cdot) - \te\uc{\tilde{m}}_\phi(s\up,\cdot)}_1} &\leq \max_{s,s\up\in \cS_t}\bigg\{\norm{\te\uc{\tilde{m}}_\phi(s,\cdot) - \wp\uc{\tilde{m}}_{t,\phi}(s,\cdot)}_1 + \norm{\wp\uc{\tilde{m}}_{t,\phi}(s,\cdot) - \wp\uc{\tilde{m}}_{t,\phi}(s\up,\cdot)}_1 \\
        &\quad + \norm{\wp\uc{\tilde{m}}_{t,\phi}(s\up,\cdot) - \te\uc{\tilde{m}}_\phi(s\up,\cdot)}_1 \bigg\}\\
        &\leq \frac{1 - \alpha}{2}+ 3\alpha - 1 + \frac{1 - \alpha}{2} \\
        &= 2\alpha.
    \end{align*}
    Now, from Lemma~\ref{lem:pn_contra}, we have that the Markov chain induced by the transition kernel $\te$ under the application of policy $\phi$ is uniformly ergodic with constants $\frac{2}{\alpha}$ and $\alpha^{\tilde{m}\inv}$, i.e.,
    \begin{align*}
        \norm{\mu\uc{i}_{\phi,\te,s} - \mu\uc{\infty}_{\phi,\te}}_1 \leq \frac{2}{\alpha} \cdot \br{\alpha^{\tilde{m}\inv}}^i,~\forall i \in \bN.
    \end{align*}
    Hence, from Lemma~\ref{lem:bdd_pval_spn}, we conclude that
    \begin{align}
        \spn{v^{\phi,t}_n} \leq \frac{\ceil{\log_{\br{\frac{1}{\alpha}}^{\tilde{m}\inv}}{\br{\frac{2}{\alpha}}}} + 1}{1 - \alpha^{\tilde{m}\inv}}. \label{ub:case_2}
    \end{align}
    Combining the upper-bounds from \eqref{ub:case_1} and \eqref{ub:case_2}, we obtain the desired upper-bound.
\end{proof}

In the next lemma, we establish that the optimism injected by \algo~is not huge.

\begin{lemma}\label{lem:ub_opt}
    Consider time $t \in \bN$ and a policy $\phi \in \Phi_t$.~Let $\bar{\phi} \in \Phi_{SD}$ be the extension of $\phi$ as follows:
    \begin{align*}
        \bar{\phi}(s) = \phi(q(\xi)), \mbox{ for every } s \in \xi, \mbox{ for every } \xi \in \cQ_t.
    \end{align*}
    Then, we have that on the set $\cG_1$,
    \begin{align}\label{eq:lb_index}
        J_{\cM^+_t}(\phi) \leq J_\cM(\bar{\phi}) + C_{ub}~ \diam{t}{\bar{\phi}},~\forall t \in \bN, \phi \in \Phi_t,
    \end{align}
    where 
    $J_{\cM^+_t}(\phi)$ is the optimal value of $\cM^{+}_{t}$ when the control input component of the extended action is chosen according to the policy $\phi$, and the transition kernel is chosen so as to maximize the average reward, $\diam{t}{\bar{\phi}}$ is as defined in \eqref{def:diam_pol}, and
    \al{
    C_{ub} := 2 L_r + (3(1 + L_p) + C_p) C_v. \label{def:Cub}
    }
    $L_r, L_p$ are as stated in Assumption~\ref{assum:lip}, $C_p$ is as stated in Assumption~\ref{assum:statn_dist}, and $C_v$ is as defined in~\eqref{def:Cv}.
\end{lemma}
\begin{proof}
    Consider the iteration \eqref{iter:v_epe}. From Corollary~\ref{cor:conv_epe} it follows that
    \begin{align*}
        \lim_{n \to \infty}{\br{v^\phi_{n+1}(s) - v^\phi_n(s)}} = J_{\cM^+_t}(\phi),~\mbox{ for every } s \in \cS_t.
    \end{align*}
    As the sequence of Cesaro means converges to the same limit, we can write
    \begin{align*}
        \lim_{n \to \infty}{\frac{1}{n}v^\phi_n(s)} = J_{\cM^+_t}(\phi).
    \end{align*}
    Similarly, from the policy evaluation iteration for the true MDP~\eqref{def:epe_true}, we have that
    \begin{align*}
        \lim_{n \to \infty}{\frac{1}{n}V^{\bar{\phi}}_n(s)} = J_{\cM}(\bar{\phi}).
    \end{align*}
    In order to prove the lemma, we will show that on the set $\cG_1$, for every $n \in \bN$, for every $s \in \cS_t$ and for every $s\up \in q\inv(s)$, the following holds,
    \begin{align}\label{eq:lb_Vindex}
        v^\phi_n(s) \leq V^{\bar{\phi}}_n(s\up) + C_{ub}~ \bE_{p,\bar{\phi}}\sqbr{\sum_{i=0}^{n-1}{\diamc{q_t\inv(s_i, \bar{\phi}(s_i))}} \middle| s_0 = s\up},
    \end{align}
    where $\bE_{p,\phi}$ denotes that the expectation is taken with respect to the measure induced by $\phi$ when it is applied to MDP with transition kernel $p$.~We prove this using induction. The base case~$(n=0)$ is seen to hold trivially. Next, we assume that the following holds for $i \in \{0,1,\ldots,n\}$, where $n \in \bN$,
    \begin{align}\label{indhyp}
        v^\phi_i(s) \leq V^{\bar{\phi}}_i(s\up) + C_{ub}~ \bE_{p,\bar{\phi}}\sqbr{\sum_{j=0}^{i-1}{\diamc{q_t\inv(s_j, \bar{\phi}(s_j))}} \middle| s_0 = s\up},
    \end{align}
    for every $s \in \cS_t$ and for every $s\up \in q\inv(s)$.~Let us fix $s \in \cS_t$ and $s\up \in q\inv(s)$ arbitrarily, then from~\eqref{iter:v_epe} we obtain the following,
    \begingroup
    \allowdisplaybreaks
    \begin{align*}
        v^\phi_{n+1}(s) &= r(q(q\inv_t(s,\phi(s)))) + \max_{\te \in \cC_t}{\sum_{s\upp \in \cS_t}{\te(q(q\inv_t(s,\phi(s))), s\upp) v^\phi_n(s\upp)}} + L_r~ \diamc{q\inv_t(s,\phi(s))}\\
        &= r(q(q\inv_t(s,\phi(s)))) + \sum_{s\upp \in \cS_t}{\te_n(q(q\inv_t(s,\phi(s))), s\upp) \bar{V}^\phi_n(s\upp)} + L_r~ \diamc{q\inv_t(s,\phi(s))}\\
        &\leq r(s\up, \phi(s\up)) + \sum_{s\upp \in \cS_t}{\wp(s\up, \phi(s\up), s\upp; \cS_t \times A_t, \cQ_t) ~v^\phi_n(s\upp)} + \eta_t(q\inv_t(s,\phi(s))) \spn{v^\phi_n} + 2L_r~ \diamc{q\inv_t(s,\phi(s))} \\
        &\leq r(s\up, \phi(s\up)) + \int_\cS{p(s\up, \phi(s\up), s\upp) V^\phi_n(s\upp) ds\upp} + C_{ub}~ \bE_{p,\phi}\sqbr{\sum_{i=1}^{n}{\diamc{q_t\inv(s_i, \phi(s_i))}} \middle| s_0 = s\up} \\
        &\quad + \br{2L_r + (3(1 + L_p) + C_p) C_v} \diamc{q\inv_t(s,\phi(s))} \\
        &\leq r(s\up, \phi(s\up)) + \int_\cS{p(s\up, \phi(s\up), s\upp) V^\phi_n(s\upp) ds\upp} + C_{ub}~ \bE_{p,\phi}\sqbr{\sum_{i=1}^{n}{\diamc{q_t\inv(s_i, \phi(s_i))}} \middle| s_0 = s\up}\\
        &\quad +  \br{2L_r + (3(1 + L_p) + C_p) C_v} \diamc{q\inv_t(s,\phi(s))}\\
        &= V^\phi_{n+1}(s) + C_{ub}~ \bE_{p,\phi}\sqbr{\sum_{i=0}^{n}{\diamc{q_t\inv(s_i, \phi(s_i)))}} \middle| s_0 = s},
    \end{align*}
    \endgroup
    where $\te_n$ is a transition kernel belonging to the set $\cC_t$ that maximizes the expression in the r.h.s. of the first equality.~The first inequality follows from Lipschitz continuity of the reward function, the definition of event $\cG_1$ and from Lemma~\ref{lem:bdd_dotdifLv}.~The second inequality is obtained by invoking the induction hypothesis~\eqref{indhyp}, and by using the upper-bound on $\spn{v^\phi_n}$ from Lemma~\ref{lem:bd_span_epe}.~This concludes the induction argument, and proves~\eqref{eq:lb_Vindex}.~The proof of the claim follows by dividing both side of~\eqref{eq:lb_Vindex} by $n$ and taking limit $n \to \infty$.
\end{proof}


\section{Properties of Proxy Diameter}\label{app:prop_pdiam}
In this section, we present three results as the corollaries of the results obtained in the previous section.
\begin{cor}\label{cor:opt_pdiam}
    Fix a time $t$. Let $\phi \in \Phi_t$ and $\bar{\phi} \in \Phi_{SD}$ be the unique extension of $\phi$ such that 
    \al{
    \bar{\phi}(s\up) = \phi(s), \mbox{ for every } s \in \cS_t \mbox{ and } s\up \in q\inv(s).\label{def:pol_ext}
    }
    On the set $\cG_1$, we have,
    \al{
        \pdiam{t}{\phi} \geq \diam{t}{\phi},~\forall t \in \{0,1,\ldots, T-1\}, \phi \in \Phi_t.\label{ineq:opt_pdiam}
    }
    where $\pdiam{t}{\phi}$ is the average reward of policy $\phi$ evaluated on the extended MDP~$\cM^{d,+}_t$ and $\diam{t}{\bar{\phi}} = \int_{\cS}{q_t\inv(s,\phi(s)) \mu\uc{\infty}_{\phi,p}(s) ds}$.
\end{cor}
\begin{proof}
    Define the MDP, $\cM^d_t := (\cS, \cA, p, \tilde{d})$ where
    \begin{align*}
        \tilde{d}(s,a) = \diamc{q\inv_t(s,a)},\mbox{ for every } (s,a) \in \cS \times \cA.
    \end{align*}
    As $p$ satisfy Assumption~\ref{assum:unif_ergodic},
    \begin{align*}
        J_{\cM^d_t}(\bar{\phi}) = \diam{t}{\bar{\phi}},\mbox{ for every } \bar{\phi} \in \Phi_{SD}.
    \end{align*}
    Note that the extended policy evaluation~\eqref{iter:v_epe} and policy evaluation~\eqref{def:epe_true} algorithms are equivalent to extended value iteration~\eqref{iter:evi} and value iteration~\eqref{iter:vi} algorithms, respectively, except that the control inputs have to be chosen from singleton sets.~Then the proof follows from Lemma~\ref{lem:optimism}.
\end{proof}

\begin{cor}\label{cor:bd_span_epe}
    Let $t \in \{0,1,\ldots,T-1\}$.~Consider the extended MDP $\cM^{d,+}_t$, a policy $\phi \in \Phi_t$ and the corresponding \epe~\eqref{algo:epe}~iterates:
    \begin{align*}
        g^{\phi,t}_0(s) &= 0, \notag\\
        g^{\phi,t}_{n+1}(s) &= \max_{\te \in \cC_t} \flbr{d_t(s,\phi(s)) + \sum_{s\up \in \cS_t}{\te(s,\phi(s),s\up) g^{\phi,t}_n(s\up)}},~\forall s \in \cS_t, n \in \bN.
    \end{align*}
    On the set $\cG_1$, we have 
    \nal{
        \spn{g^{\phi,t}_n} \leq C_v,~\forall n \in \bN, t \in \bN,
    }
    where, $C_v$, $\ovl{m}$ and $\tilde{m}$ are defined in \eqref{def:Cv}, \eqref{def:m_bar} and \eqref{def:tm}, respectively.
\end{cor}
\begin{proof}
    Follows from Lemma~\ref{lem:bd_span_epe}.
\end{proof}

\begin{cor}\label{cor:ub_pdiam}
    Consider time $t \in \bN$ and a policy $\phi \in \Phi_t$.~Let $\bar{\phi} \in \Phi_{SD}$ be the extension of $\phi$ as defined in \eqref{def:pol_ext}.~Then, we have that on the set $\cG_1$,
    \begin{align*}
        \pdiam{t}{\phi} \leq  (C_{ub} + 1)~ \diam{t}{\bar{\phi}},~\forall t \in \bN, \phi \in \Phi_t,
    \end{align*}
    where $C_{ub}$ is as defined in \eqref{def:Cub}.
\end{cor}
\begin{proof}
    Noting that $J_{\cM^d_t}(\bar{\phi}) = \diam{t}{\bar{\phi}}$ and $J_{\cM^{d,+}_t}(\phi) = \pdiam{t}{\phi}$, the claim follows from Lemma~\ref{lem:ub_opt} and Corollary~\ref{cor:bd_span_epe}.
\end{proof}

\section{Guarantee on Number of Visits to Cells}\label{app:visits}
Recall that $\mu^{(t)}_{\phi,p,s}$ denotes the distribution of $s_t$ when policy $\phi$ is applied to the MDP that has the transition kernel $p$ and the initial state is $s$, and $\mu^{(\infty)}_{\phi,p}$ denotes the unique invariant distribution of the Markov chain induced by the policy $\phi$ on the MDP with transition kernel $p$.~Consider an $\cS$-cell $\xi$ for which the diameter is greater than $\eps$, and $\mu^{(\infty)}_{\phi,p}(\xi) \geq (\eps/3)^{d_\cS + 1}$ for all stationary deterministic policies $\phi$, where $\eps>0$.~Later we will choose an appropriate value for $\eps$.~From Assumption~\ref{assum:unif_ergodic} we get that for all $\phi \in \Phi_{SD}$ and for every initial state $s \in \cS$ we have,
\begin{align*}
    \mu^{(t)}_{\phi,p,s}(\xi) \geq \mu^{(\infty)}_{\phi,p}(\xi) - \frac{C}{2} \alpha^t.
\end{align*}
Since $\mu^{(\infty)}_{\phi,p}(\xi) \geq (\eps/3)^{d_\cS + 1}$, we have
\al{
    \mu^{(t)}_{\phi,p,s}(\xi) \geq \frac{1}{2} \mu^{(\infty)}_{\phi,p}(\xi),~\forall t \geq t\ust(\eps),\label{def:t_star_1}
}
where,
\al{
t\ust(\eps) := \ceil{\log_{\frac{1}{\alpha}}{\br{C\br{\frac{3}{\eps}}^{d_\cS + 1}}}}. \label{def:t_star_2}
}

\begin{lemma}\label{lem:lb_n_epi}
    Fix $k \in \bN$ and consider a $\cS$-cell $\xi \in \cQ_{\tau_k}$ such that $\mu^{(\infty)}_{\phi,p}(\xi) \geq (\eps/3)^{d_\cS + 1}$. Let $\zeta \in \cP_{\tau_k}$ denote the active cell that contains $\left\{(s,\phi_k(s))\right\}_{s \in \xi}$.~Let $n_k(\zeta)$ be the number of visits to $\zeta$ in the $k$-th episode, and $H_k$ be the duration of the $k$-th episode. Then, with a probability at least $1 - \frac{\delta}{3}$, we have,
    \begin{align*}
        n_k(\zeta) \geq \frac{H_k~ \mu^{(\infty)}_{\phi,p}(\xi)}{2 t\ust(\eps)} - \sqrt{\frac{H_k}{t\ust(\eps)} \log{\br{\frac{6 T}{t\ust(\eps) \delta}}}} - 1.
    \end{align*}
\end{lemma}
\begin{proof}
    Denote $m: = \floor{H_k / t\ust(\eps)}$ and $t_i := \tau_k + i ~t\ust(\eps)$. Let $i\ust \in \{0\}\cup\bN$ be such that $t_{i\ust} \leq T < t_{i\ust+1}$. Define the following martingale difference sequence $\{b_i\}_i$ w.r.t. the filtration $\{\cF_{t_i}\}_i$,
    \begin{align*}
        b_i := \ind{s_{t_i} \in \xi} - \bE\sqbr{\ind{s_{t_i} \in \xi} \mid \cF_{t_{i-1}}},~i=1,2,\ldots,i\ust.
    \end{align*}
    Also, define
    \begin{align*}
        g_i := \ind{(i-1) t\ust(\eps) \leq H_k},~i=1,2,\ldots,i\ust,
    \end{align*}
    and note that it is $\{\cF_{t_i}\}_i$-predictable sequence. It can be shown that $b_i$'s are conditionally $\frac{1}{2}$ sub-Gaussian, i.e., $\bE[\exp(\beta~ b_i)\mid \cF_{t_{i-1}}] \leq \exp(\beta^2/8)$~\citep{raginsky2013concentration}. Also, note that $\{g_i\}_i$ is a $\{0,1\}$-valued, $\{\cF_{t_i}\}$-predictable stochastic process. Hence, we can use Corollary \ref{cor:self_norm_vec} and obtain,
    \begin{align}\label{ineq:7}
        \bP\br{\sum_{i=1}^{m+1}{\ind{s_{t_i} \in \xi}} \leq \sum_{i=1}^{m+1}{\bE\sqbr{\ind{s_{t_i} \in \xi} \mid \cF_{t_{i-1}}}} - \sqrt{\frac{m+2}{2} \log{\br{\frac{3(m+2)}{\delta}}}}} \leq \frac{\delta}{3}.
    \end{align}
    From~\eqref{def:t_star_1},~\eqref{def:t_star_2} we have that
    \al{
        \bE\sqbr{\ind{s_{t_{i-1}} \in \xi} \mid \cF_{t_{i-1}}} \geq \frac{1}{2} \mu^{(\infty)}_{\phi,p}(\xi).\label{ineq:8}
        }
    Also, observe that $m + 1 > \frac{H_k}{t\ust(\eps)}$ and $m \leq \frac{H_k}{t\ust(\eps)}$. Since under \algo~algorithm we have $H_k \geq 2 t\ust(\eps)$, we get $m+2 \leq 2m$. Upon using~\eqref{ineq:8} and $m+2 \leq 2m$ in~\eqref{ineq:7}, we obtain,
    \begin{align*}
        \bP\br{\sum_{i=1}^{m}{\ind{s_{t_i} \in \xi}} \leq \frac{H_k~ \mu^{(\infty)}_{\phi,p}(\xi)}{2 t\ust(\eps)} - \sqrt{\frac{H_k}{t\ust(\eps)} \log{\br{\frac{6 H_k}{t\ust(\eps) \delta}}}}- 1} \leq \frac{\delta}{3}.
    \end{align*}
    The claim then follows since $H_k \leq T$, and $\sum_{i=1}^{m}{\ind{s_{t_i} \in \xi}} \leq n_k(\zeta)$.
\end{proof}

\begin{cor}\label{cor:G_2}
    Fix an $\eps > 0$. Consider the triplet $(k, \xi, \zeta)$ such that $k \in \{0\}\cup\bN$, $\xi \in \cQ_{\tau_k}$, $\diamc{\xi} \geq \eps$, $\mu\uc{\infty}_{\phi,p}(\xi) \geq (\eps/3)^{d_\cS + 1}$, $\zeta \in \cP_{\tau_k}$, and for every $s \in \xi$, $(s,\phi_k(s)) \in \zeta$.~Define the event,
    \begin{align}
        \cG_{2,\eps}:= \flbr{n_k(\zeta) \geq \frac{H_k~ \mu^{(\infty)}_{\phi,p}(\xi)}{2 t\ust(\eps)} - \sqrt{\frac{H_k}{t\ust(\eps)} \log{\br{\frac{12 T^2 d^\frac{d}{2}}{t\ust(\eps) \eps^d \delta}}}} - 1,~ \forall (k,\xi,\zeta) \mbox{ that satisfies the above conditions.}}, \label{def:G2}
    \end{align}
    where $t\ust(\eps) = \ceil{\log_{\frac{1}{\alpha}}{\br{C\br{\frac{3}{\eps}}^{d_\cS + 1}}}}$. We have, $\bP(\cG_{2,\eps}) \geq 1 - \frac{\delta}{3}$.
\end{cor}
\begin{proof}
    Since $k$ denotes the episode number, it can not exceed $T$. By definition of $\cP_{\tau_k}$ and $\cQ_{\tau_k}$, $\diamc{\zeta} \geq \diamc{\xi}$. Also, the number of cells that have a diameter greater than $\eps$ is less than $(\sqrt{d}/\eps)^d$. So, the total number of possible combinations of $(k, \xi, \zeta)$ that satisfies the given condition is at most $T (\sqrt{d}/\eps)^d$.~The proof then follows from Lemma \ref{lem:lb_n_epi} by taking a union bound over all $(k, \xi, \zeta)$ and by the fact that $H_k \leq T$.
\end{proof}
\section{Regret Analysis}\label{app:regret}
\textbf{Regret decomposition:} 
Recall the regret~\eqref{def:regret} decomposition of $\algo$,
\begin{align}
    \cR(T;\algo) &= T J\ust_{\cM} - \sum_{k=1}^{K(T)}{\sum_{t = \tau_k}^{\tau_{k+1}-1}{r(s_t,a_t)}} \notag\\
    &= \underbrace{\sum_{k=1}^{K(T)}{H_k \br{J\ust_{\cM} - J_{\cM}(\phi_k)}}}_{(a)} + \underbrace{\sum_{k=1}^{K(T)}{\br{H_k~ J_{\cM}(\phi_k) - \sum_{t=\tau_k}^{\tau_{k+1}-1}{r(s_t,\phi_k(s_t))}}}}_{(b)}.\label{eq:decompregret}
\end{align}

The term (a) captures the regret arising due to the gap between the optimal value of the average reward and the average reward of the policies $\{\phi_k\}$ that are actually played in different episodes, while (b) captures the sub-optimality arising since the distribution of the induced Markov chain does not reach the stationary distribution in finite time.~(a) and (b) are bounded separately.

\textbf{Bounding} (a): 
This term can be further decomposed into the sum of the regrets arising due to playing policies from the sets $\Phi\uc{2^{-i}}$, for $i = 1, 2, \ldots, \ceil{\log{\br{1/\eps}}}$, and the regret arising from playing all $\eps$-optimal policies.~To bound the regret arising due to policies from $\Phi\uc{2^{-i}}$, we count the number of timesteps in which policies from $\Phi\uc{2^{-i}}$ are played, and then multiply it by $2^{-i+1}$.~We then add these regret terms from $i=1$ to $\ceil{\log{\br{1/\eps}}}$. Note that the cumulative regret arising from playing the set of $\eps$-optimal policies is upper-bounded by $\eps T$.~Recall that at the beginning of the $k$-th episode, \algo~solves $\cM^+_{\tau_k}$ with the accuracy parameter set equal to $\frac{1}{\sqrt{T}}$.~This ``loss of accuracy'' as compared to the case where \algo~could have solved $\cM^+_{\tau_k}$ accurately at the beginning of every episode, leads to an additional term in the upper-bound of (a).~From Lemma~\ref{lem:conv_evi}, the difference between the two solutions is at most $\frac{1}{\sqrt{T}}$ for each episode, hence this term can be upper-bounded as $\sqrt{T}$.~Hence, we bound (a) by firstly considering that~\algo~solves $\cM^+_{\tau_k}$ for the optimal policy (with complete accuracy), and then add $\sqrt{T}$ to obtain the upper-bound of term (a).

The regret arising due to playing policies from the set $\Phi\uc{2^{-i}}$ is bounded as follows.~Lemma~\ref{lem:keycell} proves the existence of a key cell in every episode on the set $\cG_1$.~Its proof relies crucially on Lemma~\ref{lem:gap_phi} and on the properties of the index of policies that are derived in Section~\ref{app:prop_emdp}.~Lemma~\ref{lem:lb_num_visit} gives a lower-bound of the number of plays of a key cell in any episode by \algo~using Lemma~\ref{lem:keycell}, Corollary~\ref{cor:G_2}, and Lemma~\ref{lem:bdd_epi_tool}.~Next, Lemma~\ref{lem:bdd_Phi_play} establishes an upper-bound on the number of timesteps when policies from $\Phi\uc{2^{-i}}$ are played. This upper-bound multiplied by $2^{-i+1}$, is the regret arising from playing policies from $\Phi\uc{2^{-i}}$.~Next, we derive an important property of the policy $\phi \in \Phi_{SD}$ that is played in the $k$-th episode. This is used to upper-bound the number of plays of sub optimal policies.
\begin{lemma}\label{lem:keycell}
    Consider a sample path from the set $\cG_1$~\eqref{def:G_1}.~For each $k=1,2,\ldots$, there exists at least one $s \in \cS$ (where $s$ could vary with $k$, and here we are suppressing dependence upon $k$) such that 
    \begin{align*}
        &\diamc{q\inv_{\tau_k}(s,\phi_k(s))} \geq \frac{1}{3 C_{ub}} \max \left\{\gap{s,\phi_k(s)}, C_{ub}~ \diam{\tau_k}{\phi_k}\right\}, \\
    \mbox{ and }&\mu\uc{\infty}_{\phi_k,p}(\pi_\cS(q\inv_{\tau_k}(s,\phi_k(s)))) \geq (\diam{\tau_k}{\phi_k} / 3)^{d_\cS + 1}.
    \end{align*}
    Such a $q\inv_{\tau_k}(s,\phi_k(s))$ is called a key cell for the $k$-th episode.
\end{lemma}
\begin{proof}
     Let us fix $k \in \bN$ and a policy $\phi \in \Phi_{\tau_k}$. Let $\bar{\phi}$ be the unique continuous extension of $\phi$ as defined in \eqref{def:pol_ext}.~We will first show that if
    \begin{align}\label{cond:noplayphi_1}
        \diam{\tau_k}{\bar{\phi}} \leq \Delta(\bar{\phi})/C_{ub},
    \end{align}
    then $\bar{\phi}$ will not be played from episode $k$ onwards.~From Lemma~\ref{lem:optimism} we have that on the set $\cG_1$, $J\ust_{\cM^+_{\tau_k}} = J_{\cM^+_{\tau_k}}(\tilde{\phi}_k) \geq J\ust_{\cM}$.~Hence, if $J_{\cM^+_{\tau_k}}(\phi) < J\ust_{\cM}$, then the algorithm will not play $\bar{\phi}$.~From Lemma~\ref{lem:ub_opt} we have that on the set $\cG_1$, $J_{\cM^+_{\tau_k}}(\phi) \leq J_{\cM}(\bar{\phi}) + C_{ub}~ \diam{\tau_k}{\bar{\phi}}$.~Thus, on $\cG_1$, $\bar{\phi}$ will never be played from the $k$-th episode onwards if
    \begin{align*}
        J_{\cM}(\bar{\phi}) + C_{ub}~ \diam{\tau_k}{\bar{\phi}} \leq J\ust_{\cM},
    \end{align*}
    or, if $\diam{\tau_k}{\bar{\phi}} \leq \Delta(\bar{\phi})/ C_{ub}$.~In other words, on the set $\cG_1$,
    \begin{align}
        \diam{\tau_k}{\phi_k} > \Delta(\phi_k)/ C_{ub}.\label{cond:phi_k}
    \end{align}
    We will prove the result by contradiction.~Let us assume that for all $s \in \cS$ that satisfy $\mu\uc{\infty}_{\phi_k,p}(\pi_\cS(q\inv_{\tau_k}(s,\phi_k(s)))) \geq (\diam{\tau_k}{\phi_k} / 3)^{d_\cS + 1}$, the following is true:
    \begin{align}
        \diamc{q\inv_{\tau_k}(s,\phi_k(s))} \leq \frac{1}{3 C_{ub}} \max{\{\gap{s,\phi_k(s)}, C_{ub} \diam{\tau_k}{\phi_k}\}}. \label{assum:contra}
    \end{align}
    Define the following sets of $\cS$-cells:
    \begin{align*}
        \cQ\uc{1} &:= \{\xi \in \cQ_{\tau_k} \mid \mu\uc{\infty}_{\phi_k,p}(\xi) < (\diam{\tau_k}{\phi_k} /3)^{d_\cS + 1},~\diamc{q\inv_{\tau_k}(q(\xi),\phi_k(q(\xi)))} \geq \diam{\tau_k}{\phi_k}/ 3 \},\\
        \cQ\uc{2} &:= \{\xi \in \cQ_{\tau_k} \mid \diamc{q\inv_{\tau_k}(q(\xi),\phi_k(q(\xi)))} < \diam{\tau_k}{\phi_k}/ 3 \},\\
        \cQ\uc{3} &:= \{\xi \in \cQ_{\tau_k} \mid \mu\uc{\infty}_{\phi_k,p}(\xi) \geq (\diam{\tau_k}{\phi_k} /3)^{d_\cS + 1},~\diamc{q\inv_{\tau_k}(q(\xi),\phi_k(q(\xi)))} \geq \diam{\tau_k}{\phi_k}/ 3 \}.
    \end{align*}
    We observe that $\cQ_{\tau_k}$ is partitioned by $\cQ\uc{1}$, $\cQ\uc{3}$ and $\cQ\uc{3}$.~Note that $\abs{\cQ\uc{1}} \leq (\diam{\tau_k}{\phi_k}/3)^{-d_\cS}$.~Also, note that by the necessary condition for $\phi_k$ to be played and by our assumption, for every $\xi \in \cQ\uc{3}$, $\frac{1}{3} \diam{\tau_k}{\phi_k} \leq \diamc{q\inv_{\tau_k}(q(\xi),\phi_k(q(\xi)))} \leq \frac{1}{3 C_{ub}}\min_{s\in \zeta}\{\gap{s,\phi_k(s)}\}$. Then,
    \begin{align*}
        \diam{\tau_k}{\phi_k} &= \int_{\cS}{\diamc{q\inv_{\tau_k}(s,\phi_k(s))} \mu\uc{\infty}_{\phi_k,p}(s) ~ds}\\
        &= \sum_{\xi \in \cQ_{\tau_k}}{\diamc{q\inv_{\tau_k}(q(\xi),\phi_k(q(\xi)))} \mu\uc{\infty}_{\phi_k,p}(\xi)} \\
        &= \sum_{\xi \in \cQ\uc{1}}{\diamc{q\inv_{\tau_k}(q(\xi),\phi_k(q(\xi)))} \mu\uc{\infty}_{\phi_k,p}(\xi)} + \sum_{\xi \in \cQ\uc{2}}{\diamc{q\inv_{\tau_k}(q(\xi),\phi_k(q(\xi)))} \mu\uc{\infty}_{\phi_k,p}(\xi)}\\
        &\quad + \sum_{\xi \in \cQ\uc{3}}{\diamc{q\inv_{\tau_k}(q(\xi),\phi_k(q(\xi)))} \mu\uc{\infty}_{\phi_k,p}(\xi)}\\
        &\leq \frac{\diam{\tau_k}{\phi_k}}{3} + \frac{\diam{\tau_k}{\phi_k}}{3} + \frac{1}{3 C_{ub}}\int_{\cS}{\gap{s,\phi_k(s)} \mu\uc{\infty}_{\phi_k,p}(s)~ ds} \\
        &= \frac{\diam{\tau_k}{\phi_k}}{3} + \frac{\diam{\tau_k}{\phi_k}}{3} + \frac{\Delta(\phi_k)}{3~ C_{ub}} \\
        &< \diam{\tau_k}{\phi_k},
    \end{align*}
which yields us a contradiction.~Hence, we conclude that our assumption~\eqref{assum:contra} was wrong.~This concludes the proof.
\end{proof}

Define, 
\al{
\eps(T) := T^{-\frac{1}{2d_\cS + d_z + 3}},~~\teps(T) := T^{-\frac{1}{2d_\cS + d + 3}}
}
Note that $\eps(T) \geq \teps(T)$ since $d_z \leq d$.~Also, note that $t\ust(\eps(T)) \leq t\ust(\teps(T))$, where $t\ust(\cdot)$ is defined ins \eqref{def:t_star_2}.

\textbf{Choosing $C_H$:} We choose the constant associated with the episode duration~\eqref{def:epi_dur} of \algo~as,
\begin{align}
    C_H \geq 16~ t\ust(\teps(T))~ \br{\frac{3(1 + C_{ub})}{1-\gamma}}^{2(d_\cS + 1)} \frac{\log{\br{\frac{12 T^2 d^\frac{d}{2}}{t\ust(\eps(T)) \teps(T)^d \delta}}} + 1}{\log(T/\delta)}. \label{def:CH}
\end{align}
\begin{lemma}\label{lem:lb_num_visit}
    Pick a sample path from the set $\cG_1 \cap \cG_{2,\eps}$, where $\cG_1$ and $\cG_{2,\eps}$ are as in \eqref{def:G_1} and \eqref{def:G2}, respectively.~Let $\zeta$ be a key cell in episode $k$ (such key cells have been shown to exist in Lemma~\ref{lem:keycell}), i.e., for some $\xi \subseteq \pi_\cS(\zeta)$ such that $\xi \in \cQ_{\tau_k}$, and for some $s \in \xi$, the following holds,
    \begin{align*}
        &\diamc{\zeta} > \frac{1}{3 C_{ub}} \max{\{\gap{s,\phi_k(s)}, C_{ub}~ \diam{\tau_k}{\phi_k}\}}, \mbox{ and},\\
        &\mu\uc{\infty}_{\phi_k,p}(\xi) \geq (\diam{\tau_k}{\phi_k} / 3)^{d_\cS + 1}.
    \end{align*}
    Then, if~$\Delta(\phi_k) \geq \eps(T) C_{ub}$, then the number of visits to $\zeta$ during the $k$-th episode can be lower-bounded as follows, 
    \al{
    n_k(\zeta) \geq \frac{4 t\ust(\teps(T))}{t\ust(\eps(T))} \br{\log{\br{\frac{12 T^2 d^\frac{d}{2}}{t\ust(\eps(T)) \teps(T)^d \delta}}} + 1} \diamc{\zeta}^{-(d_\cS + 1)}.
    }
\end{lemma}

\begin{proof}
    Recall that on $\cG_1$ we have $\diam{\tau_k}{\phi_k} \geq \frac{\Delta(\phi_k)}{C_{ub}}$~\eqref{cond:phi_k}. Hence, $\diam{\tau_k}{\phi_k} > \eps(T)$ and $\mu\uc{\infty}_{\phi_k,p}(\xi) \geq (\eps(T) / 3)^{d_\cS + 1}$.~So, upon using Corollary~\ref{cor:G_2} we obtain,
    \begin{align*}
        n_k(\zeta) \geq \frac{H_k~\mu^{(\infty)}_{\phi_k,p}(\xi)}{2 t\ust(\eps(T))} - \sqrt{\frac{H_k}{t\ust(\eps(T))} \log{\br{\frac{8 T^2 d^\frac{d}{2}}{t\ust(\eps(T)) \eps(T)^d \delta}}}} - 1.
    \end{align*}
    Next, we note that the duration of the $k$-th episode $H_k$ can be lower-bounded as follows,
    \begin{align}
        H_k &\geq \frac{C_H (1 - \gamma)^{2(d_\cS + 1)} \log{\br{T/\delta}}}{\pdiam{\tau_k}{\phi_k}^{2(d_\cS + 1)}} \notag\\
        &\geq \frac{C_H (1 - \gamma)^{2(d_\cS + 1)} \log{\br{T/\delta}}}{(3(1 + C_{ub}))^{2(d_\cS + 1)}} \br{\frac{3}{\diam{\tau_k}{\phi_k}}}^{2(d_\cS + 1)} \notag\\
        &\geq \frac{16 t\ust(\eps(T))}{\mu\uc{\infty}_{\phi_k,p}(\xi)^2} \br{\log{\br{\frac{8 T^2 d^\frac{d}{2}}{t\ust(\eps(T)) \eps(T)^d \delta}}} + 1}, \label{lb:Hk}
    \end{align}
    where the first inequality follows from the lower-bound of $H_k$~\eqref{bdd:hk}, the second inequality follows since from Corollary~\ref{cor:ub_pdiam} we have $\pdiam{\tau_k}{\phi_k} \leq (1 + C_{ub}) \diam{\tau_k}{\phi_k}$.~The third inequality follows from the fact that $\mu\uc{\infty}_{\phi_k,p}(\xi) \geq (\diam{\tau_k}{\phi_k} / 3)^{d_\cS + 1}$.~Lemma~\ref{lem:bdd_epi_tool} when combined with \eqref{lb:Hk} yields
    \begin{align*}
        n_k(\zeta) &\geq \frac{H_k~ \mu^{(\infty)}_{\phi_k,p}(\xi)}{2 t\ust(\eps(T))} - \sqrt{\frac{H_k}{t\ust(\eps(T))} \log{\br{\frac{8 T^2 d^\frac{d}{2}}{t\ust(\eps(T)) \eps(T)^d \delta}}}} - 1 \\
        &\geq \frac{H_k~ \mu^{(\infty)}_{\phi_k,p}(\xi)}{4 t\ust(\eps(T))},
    \end{align*}
    or,
    \begin{align*}
        n_k(\zeta) &\geq \frac{C_H (1 - \gamma)^{2(d_\cS + 1)} \log{\br{\frac{T}{\delta}}}}{4~t\ust(\eps(T))}~\pdiam{\tau_k}{\phi_k}^{-2(d_\cS + 1)} \times (\diam{\tau_k}{\phi_k} / 3)^{d_\cS + 1}\\
        &\geq \frac{C_H \log{\br{\frac{T}{\delta}}}}{4~t\ust(\eps(T))~(3(1 + C_{ub})^2)^{d_\cS + 1}}~\diam{\tau_k}{\phi_k}^{-(d_\cS + 1)}\\
        &\geq \frac{C_H \log{\br{\frac{T}{\delta}}}}{4~t\ust(\eps(T))~(3(1 + C_{ub}))^{2(d_\cS + 1)}}~\diamc{\zeta}^{-(d_\cS + 1)}\\
        &\geq \frac{4 t\ust(\teps(T))}{t\ust(\eps(T))} \br{\log{\br{\frac{12 T^2 d^\frac{d}{2}}{t\ust(\eps(T)) \teps(T)^d \delta}}} + 1} \diamc{\zeta}^{-(d_\cS + 1)},
    \end{align*}
    where the first inequality follows from the lower-bound of $H_k$~\eqref{bdd:hk} and from the fact that $\mu^{(\infty)}_{\phi_k,p}(\xi) \geq (\diam{\tau_k}{\phi_k} / 3)^{d_\cS + 1}$. The second and the third inequality follow from the fact that $\pdiam{\tau_k}{\phi_k} \leq (1+C_{ub})\diam{\tau_k}{\phi_k}$, and $\diam{\tau_k}{\phi_k} < 3~\diamc{\zeta}$, respectively. The fourth inequality follows from \eqref{def:CH}. This concludes the proof.
\end{proof}

\begin{lemma}\label{lem:bdd_Phi_play}
    Consider the set of policies $\Phi\uc{2^{-i}} = \{\phi \in \Phi_{SD} \mid \Delta(\phi) \in (2^{-i}, 2^{-i+1}]\}$, where $i \in \bN$.~On the set $\cG_1$,~\algo~can play policies from the set $\Phi\uc{2^{-i}}$ for a maximum of $\cO(\log{\br{\frac{T}{\delta}} 2^{i (2d_\cS + d_z + 3)}})$ time steps.
\end{lemma}
\begin{proof}
    We prove this lemma in the following three steps: First, we derive the number of episodes in which a cell can serve as a key cell while policies from $\Phi\uc{2^{-i}},~i \in \bN$ are being played. Secondly, we derive an upper-bound on the episode duration when policies from $\Phi\uc{2^{-i}}$ are played. Thirdly, we multiply upper-bounds on the number of episodes with the upper-bound on the duration of the episodes and then sum it over all possible key cells corresponding to policies in $\Phi\uc{2^{-i}}$, and this yields the desired upperbound on cumulative plays from $\Phi\uc{2^{-i}}$. 
    
    Before proceeding with proving these three properties, we begin with some preliminary results.~Recall that for $\beta>0$, the set $\cZ_\beta \subseteq \cS \times \cA$ consists of those state-action pairs $(s, a)$ for which $\gap{s,a} \leq \beta$.~Let us denote the smallest subset of $\cP_t$ that covers $\cZ_\beta$, as the active covering of $\cZ_\beta$ at time $t$.~From Lemma~\ref{lem:keycell}, we obtain that if for all $j = 0, 1, \ldots, i$, the active covering of $\cZ_{2^{-j}}$ at time $\tau_k$ does not contain a cell $\zeta$ that satisfies the following conditions, 
    \begin{enumerate}
        \item $\diamc{\zeta} \geq \frac{\sqrt{d}}{3 C_{ub}}~2^{-j}$,  and 
        \item $\mu\uc{\infty}_{\phi,p}(\xi) \geq \br{\Delta(\phi)/3 C_{ub}}^{d_\cS + 1}$ for all $\xi$ which satisfy $\xi \in \cQ_{\tau_k}$ and $\xi \subseteq \pi_\cS(\zeta)$,
    \end{enumerate}
    then there is no cell that qualifies to be a key cell for a policy from the set $\Phi\uc{2^{-i}}$.~Thus, under the above condition,~\algo~will not play a policy from $\Phi\uc{2^{-i}}$ $k$-th episode onwards.~Let $\cY_j$ be the covering of $\cZ_{2^{-j}}$ by cells of diameter $\frac{\sqrt{d}}{3 C_{ub}}~2^{-j}$.~We make the following observation: If every cell in $\cY_j$ for $j = 1, 2, \ldots, i$ is split, then no cell in the active covers of $\cZ_{2^{-j}}$ for $j = 1, 2, \ldots, i$ can serve as the key cell while playing policies from $\Phi\uc{2^{-i}}$. This is a sufficient condition for any policy from $\Phi\uc{2^{-i}}$ to be not played by \algo.
    
    \texttt{Step 1:} First, we bound the number of episodes when a cell $\zeta \in \cY_i$ or any of its ancestors has served as a key cell.~From the cell activation rule~\eqref{def:activationrule}, we have that $\zeta$ would be split when the number of visits to $\zeta$ exceeds $c_a 2^{d_\cS+2} \log{\br{\frac{T}{\delta}}} \diamc{\zeta}^{-(d_\cS+2)}$. In Lemma~\ref{lem:bdd_Phi_play}, we derived the lower-bound on the number of visits to a key cell. Invoking that lower-bound, we obtain that $\zeta$ can be played in at most
    \begin{align*}
        \frac{c_a t\ust(\eps(T)) 2^{d_\cS+2} \log{\br{\frac{T}{\delta}}}}{4 t\ust(\teps(T)) \br{\log{\br{\frac{12 T^2 d^\frac{d}{2}}{t\ust(\eps(T)) \teps(T)^d \delta}}} + 1}}  \diamc{\zeta}^{-1}
    \end{align*}
    episode as a key cell when the corresponding episode plays a policy from $\Phi\uc{2^{-i}}$. Replacing $\diamc{\zeta}$ with $\frac{\sqrt{d}}{3 C_{ub}}~2^{-j}$, we obtain that $\zeta$ can be played in at most
    \begin{align*}
        \frac{3 c_a t\ust(\eps(T)) C_{ub} 2^{d_\cS+2} \log{\br{\frac{T}{\delta}}}}{4 t\ust(\teps(T)) \sqrt{d} \br{\log{\br{\frac{12 T^2 d^\frac{d}{2}}{t\ust(\eps(T)) \teps(T)^d \delta}}} + 1}} ~2^{j}
    \end{align*}
    episode as a key cell when the corresponding episode plays a policy from $\Phi\uc{2^{-i}}$.
    
    \texttt{Step 2:} Now, we produce an upper-bound on the length of the episodes while playing policies from $\Phi\uc{2^{-i}}$. See that
    \begin{align*}
        H_k &\leq \frac{C_H (1+\gamma)^{2(d_\cS + 1)} \log{\br{\frac{T}{\delta}}}}{\pdiam{\tau_k}{\phi_k}^{2(d_\cS + 1)}} \\
        &\leq \frac{C_H (1+\gamma)^{2(d_\cS + 1)} \log{\br{\frac{T}{\delta}}}}{\diam{\tau_k}{\phi_k}^{2(d_\cS + 1)}} \\
        &\leq \frac{C_H ((1+\gamma) C_{ub})^{2(d_\cS + 1)} \log{\br{\frac{T}{\delta}}}}{2^{-i 2(d_\cS + 1)}},
    \end{align*}
    where the first inequality follows from the upper-bound on $H_k$~\eqref{bdd:hk}, the second inequality follows from Corollary~\ref{cor:opt_pdiam}, and the third inequality follows from the definition of $\Phi\uc{2^{-i}}$.
    
    \texttt{Step 3:} First, we note that the cardinality of $\cY_j$ is at most $c_z 2^{j d_z}$ for every $j \in \bN$, where the scaling constant of the zooming dimension,
    \al{
        c_s := \frac{\sqrt{d}}{3C_{ub}}. \label{def:cs}
    }
    This follows from the definition of the zooming dimension~\eqref{def:zoomingdim}.~Multiplying the bounds from step $1$ and step $2$, we obtain an upper-bound on the number of plays of a cell $\zeta \in \cY_j$ as a key cell while playing policies from $\Phi\uc{2^{-i}}$. Summing this upper-bound for all cells in $\cY_j$ and then summing those terms over $j = 1, 2, \ldots, i$, we obtain that the total number of time steps in which policies from $\Phi\uc{2^{-i}}$ is played, can be bounded above by 
    \begin{align*}
        &\sum_{j = 1}^{i}{\sum_{\zeta \in \cY_j}{\br{\frac{3 c_a t\ust(\eps(T)) C_{ub} 2^{d_\cS+2} \log{\br{\frac{T}{\delta}}}}{4 t\ust(\teps(T)) \sqrt{d} \br{\log{\br{\frac{12 T^2 d^\frac{d}{2}}{t\ust(\eps(T)) \teps(T)^d \delta}}} + 1}} ~2^{j}} \times \br{\frac{C_H ((1+\gamma)C_{ub})^{2(d_\cS + 1)} \log{\br{\frac{T}{\delta}}}}{2^{-i 2(d_\cS + 1)}}}}}\\
        &=\frac{3 c_a c_z t\ust(\eps(T)) C_H C_{ub}^{2d_\cS + 3} (1+\gamma)^{2(d_\cS + 1)} 2^{d_\cS+2} \br{\log\br{\frac{T}{\delta}}}^2}{4 t\ust(\teps(T)) \sqrt{d} \br{\log{\br{\frac{12 T^2 d^\frac{d}{2}}{t\ust(\eps(T)) \teps(T)^d \delta}}} + 1}} 2^{i 2(d_\cS + 1)} \sum_{j = 0}^{i}{2^{j (d_z + 1)}} \\
        &\leq \frac{3 c_a c_z t\ust(\eps(T)) C_H C_{ub}^{2d_\cS + 3} (1+\gamma)^{2(d_\cS + 1)} 2^{d_\cS+1} \br{\log\br{\frac{T}{\delta}}}^2}{ t\ust(\teps(T)) \sqrt{d} \br{\log{\br{\frac{12 T^2 d^\frac{d}{2}}{t\ust(\eps(T)) \teps(T)^d \delta}}} + 1}} 2^{i (2d_\cS + d_z + 3)}.
    \end{align*}
    This concludes the proof.
\end{proof}
Let us denote
\begin{align}
    C\up := \frac{3 c_a c_z t\ust(\eps(T)) C_H C_{ub}^{2d_\cS + 3} (1+\gamma)^{2(d_\cS + 1)} 2^{d_\cS+1} \br{\log\br{\frac{T}{\delta}}}^2}{ t\ust(\teps(T)) \sqrt{d} \br{\log{\br{\frac{12 T^2 d^\frac{d}{2}}{t\ust(\eps(T)) \teps(T)^d \delta}}} + 1}}. \label{def:Cup}
\end{align}
As has been discussed earlier at the beginning of this section, we derive an upper-bound on (a) of \eqref{eq:decompregret} by summing the three terms: the regret due to playing policies from the set $\Phi\uc{2^{-i}},~i = 1, 2, \ldots, \ceil{\log{1/\eps(T)}}$, the regret due to playing other policies, and the suboptimality that arises due to the inaccuracy in the solution of the extended MDPs at the beginning of every episode, which can be bounded by $\sqrt{T}$. The first term is bounded using the bound obtained on the number of plays of policies from $\Phi\uc{i}$ in Lemma~\ref{lem:bdd_Phi_play}.~The regret arising from playing policies that are not in $\cup_{i=1}^{\ceil{\log{1/\eps}}}{\Phi\uc{2^{-i}}}$ is at most $\eps(T) T$. Hence,
\begin{align}
    \sum_{k=1}^{K(T)}{H_k (J\ust_\cM - J_\cM(\phi_k)} &\leq C\up \sum_{i=1}^{i\ust}{2^{i(2d_\cS + d_z + 3)} \times 2^{-i+1}} + \eps(T) T + \sqrt{T}\notag\\
    & \leq 2 C\up~ 2^{i\ust(2d_\cS + d_z + 2)} + T^\frac{2d_\cS + d_z + 2}{2d_\cS + d_z + 3} + \sqrt{T} \notag\\
    & \leq (2 C\up + 1)~ T^\frac{2d_\cS + d_z + 2}{2d_\cS + d_z + 3} + \sqrt{T},\label{bdda}
\end{align}
where the second step follows from Lemma~\ref{lem:bdd_Phi_play}.

\textbf{Bounding} (b): We now provide an upper-bound on the term $(b)$ of~\eqref{eq:decompregret}. This proof relies on the uniform ergodicity property~(Assumption~\ref{assum:unif_ergodic}) of the underlying MDP $\cM$ and a trick that converts Markovian noise to martingale noise using the Poisson equation~\eqref{eq:pois} \citep{metivier1984applications}.
\begin{prop}\label{prop:bddb}
Define
\al{
\cG_3 := \left\{ \omega:~\eqref{bdd:fluctuation} \mbox{ holds } \right\},
}
    \begin{align}
        \sum_{k=1}^{K(T)}{\sum_{t=\tau_k}^{\tau_{k+1}-1}{J_\cM(\phi_k) - r(s_t,\phi_k(s_t))}} \leq \frac{m\ust}{1-\alpha} \sqrt{\frac{T}{2} \log{\br{\frac{3}{\delta}}}} + \frac{m\ust}{1 - \alpha} (1 + K(T)), \label{bdd:fluctuation}
    \end{align}
    where $K(T)$ denotes the total number of episodes until time $T$, and $m\ust = \ceil{\log_{\frac{1}{\alpha}}\br{C}} + 1$. 
Then, we have,
\al{
\bP\br{\cG_3} \geq 1 - \frac{\delta}{3},~\delta \in (0,1). \label{bdd:fluc}
}
\end{prop}
\begin{proof}
    Let us denote the episode index at time $t$ by $k(t)$.~We begin by converting the Markovian noise to a martingale difference sequence, i.e.,
    \begingroup
        \allowdisplaybreaks
        \begin{align}
            &\sum_{t=0}^{T-1}{J_\cM(\phi_{k(t)}) - r(s_t,\phi_{k(t)}(s_t))} \notag\\
            &= \sum_{t=0}^{T-1}{\int_{\cS}{h^{\phi_{k(t)}}_\cM(s) p(s_t,\phi_{k(t)}(s_t), ds)} - h^{\phi_{k(t)}}_\cM(s_t)} \notag\\
            &= \sum_{t=1}^{T-1}{\int_{\cS}{h^{\phi_{k(t)}}_\cM(s) p(s_{t-1},\phi_{k(t-1)}(s_{t-1}), ds)} - h^{\phi_{k(t)}}_\cM(s_t)} \notag\\
            &\quad + \sum_{t=1}^{T-1}{\int_{\cS}{ h^{\phi_{k(t)}}_\cM(s) p(s_t,\phi_{k(t)}(s_t), ds)} - \int_{\cS}{ h^{\phi_{k(t)}}_\cM(s) p(s_{t-1},\phi_{k(t-1)}(s_{t-1}), ds)}} \notag\\
            &\quad + \int_{\cS}{h^{\phi_1}_\cM(s) p(s_0,\phi_1(s_0), ds)} - h^{\phi_1}_\cM(s_0) \notag\\
            &= \sum_{t=1}^{T-1}{\int_{\cS}{h^{\phi_{k(t)}}_\cM(s) p(s_{t-1},\phi_{k(t-1)}(s_{t-1}), ds)} - h^{\phi_{k(t)}}_\cM(s_t)} \notag\\
            &\quad + \sum_{t=1}^{T-1}{\int_{\cS}{\br{h^{\phi_{k(t)}}_\cM(s) -  h^{\phi_{k(t-1)}}_\cM(s)} p(s_{t-1},\phi_{k(t-1)}(s_{t-1}), ds)}} \notag\\
            &\quad + \int_{\cS}{h^{\phi_{k(T-1})}_\cM(s) p(s_{T-1},\phi_{k(T-1)}(s_{T-1}), ds)} - h^{\phi_1}_\cM(s_0). \label{ineq:abs_dif_2}
        \end{align}
    \endgroup    
    Now consider the first summation term in the r.h.s. of \eqref{ineq:abs_dif_2}.~Denote $m_t = \int_{\cS}{h^{\phi_{k(t)}}_\cM(s) p(s_{t-1},\phi_{k(t-1)}(s_{t-1}), ds)} - h^{\phi_{k(t)}}_\cM(s_t)$.~Noting that $\phi_k$ is $\cF_{\tau_k-1}$-measurable, we obtain the following:
    \begin{align*}
        \bE\sqbr{m_t \mid \cF_{t-1}} &= \bE\sqbr{\int_{\cS}{h^{\phi_{k(t)}}_\cM(s) p(s_{t-1},\phi_{k(t-1)}(s_{t-1}), ds)} - h^{\phi_{k(t)}}_\cM(s_t) \mid \cF_{t-1}} \\
        &= \int_{\cS}{h^{\phi_{k(t)}}_\cM(s) p(s_{t-1},\phi_{k(t-1)}(s_{t-1}), ds)} - \int_{\cS}{h^{\phi_{k(t)}}_\cM(s) p(s_{t-1},\phi_{k(t-1)}(s_{t-1}), ds)}\\
        &=0.
    \end{align*}
    Hence, $\flbr{m_t}$ is a martingale difference sequence.~Also, from the bound on the span of $h^\phi_\cM$ that was derived in Lemma~\ref{lem:bdd_rvf_spn}, we have that $m_t \in \sqbr{-\frac{m\ust}{1-\alpha}, \frac{m\ust}{1-\alpha}}$.~An application of Azuma-Hoeffding inequality~(Lemma~\ref{lem:ah_ineq}), yields the following: for each $\delta \in (0,1)$, with probability at least $1 - \frac{\delta}{3}$ we have,
    \begin{align}
        \sum_{t=1}^{T-1}{\int_{\cS}{h^{\phi_{k(t)}}_\cM(s) p(s_{t-1},\phi_{k(t-1)}(s_{t-1}), ds)} - h^{\phi_{k(t)}}_\cM(s_t)} \leq \frac{m\ust}{1-\alpha} \sqrt{\frac{T}{2} \log{\br{\frac{3}{\delta}}}}. \label{bd:1}
    \end{align}
    Now, consider the second summation term in the r.h.s. of \eqref{ineq:abs_dif_2}. The $t$-th element in this summation can assume a non-zero value only when a new episode starts at time $t$.~Hence, upon using Lemma~\ref{lem:bdd_rvf_spn}, we conclude that this summation can be upper-bounded as
    \begin{align}
        \sum_{t=1}^{T-1}{\int_{\cS}{\br{h^{\phi_{k(t)}}_\cM(s) -  h^{\phi_{k(t-1)}}_\cM(s)} p(s_{t-1},\phi_{k(t-1)}(s_{t-1}), ds)}} \leq \frac{m\ust}{1 - \alpha} K(T), \label{bd:2}
    \end{align}
    where $K(T)$ denotes the number of episodes that have been started until time $T$ by the learning algorithm.~Again by using Lemma~\ref{lem:bdd_rvf_spn}, the third term can be bounded as,
    \begin{align}
        \int_{\cS}{h^{\phi_{k(T-1})}_\cM(s) p(s_{T-1},\phi_{k(T-1)}(s_{T-1}), ds)} - h^{\phi_1}_\cM(s_0)\leq \frac{m\ust}{1 - \alpha}. \label{bd:3}
    \end{align}
    Putting all the individual bounds from \eqref{bd:1}, \eqref{bd:2} and \eqref{bd:3} together, we have that for any $\delta \in (0,1)$ with probability at least $1 - \delta$,
    \begin{align}
        \sum_{t=1}^{T-1}{J_\cM(\phi_{k(t)}) - r(s_t,\phi_{k(t)}(s_t))} \leq \frac{m\ust}{1-\alpha} \sqrt{\frac{T}{2} \log{\br{\frac{3}{\delta}}}} + \frac{m\ust}{1 - \alpha} (1 + K(T)).\label{ub:b}
    \end{align}
    This concludes the proof.
\end{proof}
Upon combining the upper-bounds on all the terms of the regret decomposition, we obtain the upper-bound on the regret. This is done in the next section.

\subsection{Proof of Theorem~\ref{thm:regupperbound}}
\begin{proof}
    We first derive an upper-bound on $K(T)$, which is the total number of episodes.~The number of episodes of length greater than $T^{\frac{2 d_\cS + 2}{2 d_\cS + d_z + 3}}$ is trivially bounded above by $T^{\frac{d_z + 1}{2 d_\cS + d_z + 3}}$. Now let us bound the number of episodes of length less than $T^{\frac{2 d_\cS + 2}{2 d_\cS + d_z + 3}}$. If the length of the $k$-th episode is less than $T^{\frac{2 d_\cS + 2}{2 d_\cS + d_z + 3}}$, then from the rule of setting episode duration~\eqref{def:epi_dur}, we have
    \begin{align*}
        \frac{C_H \log{\br{\frac{T}{\delta}}}}{\pdiam{\tau_k}{\phi_k}^{2(d_\cS + 1)}} \leq T^{\frac{2 d_\cS + 2}{2 d_\cS + d_z + 3}},
    \end{align*}
    or
    \begin{align*}
        \pdiam{\tau_k}{\phi_k} \geq \br{C_H \log{\br{\frac{T}{\delta}}}}^{\frac{1}{2(d_\cS+1)}} T^{-\frac{1}{2 d_\cS + d_z + 3}}.
    \end{align*}
    From Corollary~\ref{cor:opt_pdiam} and Corollary~\ref{cor:ub_pdiam}, we obtain that
    \begin{align*}
        \frac{1}{3(C_{ub}+1)}\pdiam{\tau_k}{\phi_k} \leq \frac{1}{3} \diam{\tau_k}{\phi_k}.
    \end{align*}
    Also, from the condition of a cell $\zeta$ to be a key cell in the $k$-th episode, we have that
    \begin{align*}
        \diamc{\zeta} \geq \frac{1}{3} \diam{\tau_k}{\phi_k}.
    \end{align*}
    Combining the above three relations, we obtain that if the length of the $k$-th episode is less than $T^{\frac{2 d_\cS + 2}{2 d_\cS + d_z + 3}}$, then the diameter of the corresponding key cell is greater than 
    \begin{align*}
        \frac{\br{C_H \log{\br{\frac{T}{\delta}}}}^{\frac{1}{2(d_\cS+1)}}}{3(C_{ub}+1)} T^{-\frac{1}{2 d_\cS + d_z + 3}}.
    \end{align*}
    From the definition of the zooming dimension~\eqref{def:zoomingdim}, it follows that there can at most be $\cO\br{T^{\frac{d_z}{2 d_\cS + d_z + 3}}}$ such key cells activated by \algo, and each key cell of level $\ell$ becomes deactivated when it has been played in $\cO(2^\ell)$ episodes.~Hence there can be at most $\cO\br{T^{\frac{d_z + 1}{2 d_\cS + d_z + 3}}}$ episodes of length less than $T^{\frac{2 d_\cS + 2}{2 d_\cS + d_z + 3}}$. Hence,
    \begin{align*}
        K(T) \leq C_K T^{\frac{d_z + 1}{2 d_\cS + d_z + 3}},
    \end{align*}
    where $C_K$ is a constant.
    
    We now add all the upper-bounds of various regret components from \eqref{bdda} and \eqref{ub:b}, and use the upper-bound on $K(T)$ derived above. This yields,
    \begin{align*}
        \cR(T;\algo) & \leq (2 C\up + 1)~ T^\frac{2d_\cS + d_z + 2}{2d_\cS + d_z + 3} + \sqrt{T} + \frac{m\ust}{1-\alpha} \sqrt{\frac{T}{2} \log{\br{\frac{3}{\delta}}}} + \frac{m\ust}{1 - \alpha} (1 + K(T)) \\
        & \leq (2 C\up + 1)~ T^\frac{2d_\cS + d_z + 2}{2d_\cS + d_z + 3} +  \br{1 + \frac{m\ust}{1-\alpha} \sqrt{\frac{1}{2} \log{\br{\frac{3}{\delta}}}}} \sqrt{T} + \frac{m\ust}{1 - \alpha} \br{1 + C_K T^{\frac{d_z + 1}{2 d_\cS + d_z + 3}}} \\
        &= \ctO\br{T^{\frac{2 d_\cS + d_z + 2}{2 d_\cS + d_z + 3}}}.
    \end{align*}
    Note that $\bP(\cG_1 \cap \cG_{2,\eps} \cap \cG_3) \geq 1 - \delta$.~Thus, we have the desired regret upper-bound with probability at least $1 - \delta$.
\end{proof}
\section{Concentration Inequality}\label{app:conc_ineq}
In this section, we will show that the discretized MDP kernel belongs to a confidence ball around its estimate.~First, let us introduce some notations.~Let $\tilde{\cZ} \subseteq \cS \times \cA$, and $\tilde{\cQ}$ be a partition of $\cS$ that is made of $\cS$-cells.~Let $\tilde{\cS}$ be the set of representative points of the $\cS$-cells in $\tilde{\cQ}$.~Recall the discretization of $p$ given $\tilde{\cZ}$ and $\tilde{\cS}$, $\wp_{\tilde{\cZ} \to \tilde{\cS},p}$~\eqref{def:disc_p}.~Denote the continuous extension of $\wp_{\tilde{\cZ} \to \tilde{\cS},p}$ by $\bar{\wp}_{\tilde{\cZ} \to \tilde{\cS},p}$, i.e.,
\begin{align*}
    \bar{\wp}_{\tilde{\cZ} \to \tilde{\cS},p}(z,B) := \sum_{\xi \in \cQ}{\frac{\lambda(B \cap \xi)}{\lambda(\xi)} \wp_{\tilde{\cZ} \to \tilde{\cS},p}(z,q(\xi))},
\end{align*}
for every $z \in \cZ, B \in \cB_\cS$.~Define the set,
\begin{align}\label{def:G_1}
    \cG_1 := \cap_{t=0}^{T-1}{\flbr{\norm{\wp_{\cS \times \cA \to \cS_t, p}(z\up, \cdot) - \wp_{\cZ_t \to \cS_t, \hat{p}_t}(z,\cdot)}_1} \leq \eta_t(\zeta) \mbox{ for every } z \in \cZ_t, z\up \in q\inv(z)}.
\end{align}
We show that $\cG_1$ holds with a high probability.
\begin{lemma}\label{lem:conc_ineq}
    $\bP(\cG_1) \geq 1 - \frac{\delta}{3}$, where $\cG_1$ is as in~\eqref{def:G_1}.
\end{lemma}

\begin{proof}
    Fix $t$, and consider a point $z \in \cZ_t$.~Within this proof, we denote $q\inv_t(z)$ by $\zeta$.~Let $\zeta$ be of level $\ell$, and note that $\zeta$ is active at time $t$.~Let $z\up$ be an arbitrary point in $\zeta$.~We want to get a high probability bound on $\norm{\wp_{\cZ_t \to \cS_t,\hat{p}_t}(z,\cdot) - \wp_{\cS \times \cA \to \cS_t,p}(z,\cdot)}_1$.~We have,
    \begin{align}
        &\norm{\wp_{\cZ_t \to \cQ_t, \hat{p}_t}(z,\cdot) - \wp_{\cS \times \cA \to \cS_t,p}(z,\cdot)}_1 \notag\\
        = & \norm{\hat{p}_t(z,\cdot) - \bar{\wp}_{\cS \times \cA \to \cS_t,p}(z\up,\cdot)}_{TV} \notag\\
        \leq & \norm{\hat{p}_t(z,\cdot) - \bar{\wp}_{\cS \times \cA \to \cS\uc{\ell},p}(z\up,\cdot)}_{TV} + \norm{\bar{\wp}_{\cS \times \cA \to \cS\uc{\ell},p}(z\up,\cdot) - \bar{\wp}_{\cS \times \cA \to \cS_t,p}(z\up,\cdot)}_{TV} \notag\\
        \leq & \norm{\hat{p}\uc{d}_t(z,\cdot) - \wp_{\cS \times \cA \to \cS\uc{\ell},p}(z\up,\cdot)}_{1} + \norm{\bar{\wp}_{\cS \times \cA \to \cS\uc{\ell},p}(z\up,\cdot) - \bar{\wp}_{\cS \times \cA \to \cS_t,p}(z\up,\cdot)}_{TV}. \label{eq:con_ineq_decomp}
    \end{align}
    By definition, $\cQ_t$ is a finer partition of $\cS$ than $\cQ\uc{\ell}$.~Hence, from Lemma~\ref{lem:disc_dist}, we have that \nal{\norm{\bar{\wp}_{\cS \times \cA \to \cS\uc{\ell},p}(z\up,\cdot) - \bar{\wp}_{\cS \times \cA \to \cS_t,p}(z\up,\cdot)}_{TV} \leq C_p~\diamc{\zeta}.}
    
    Next, we will provide a high probability upperbound on the first term of r.h.s. of \eqref{eq:con_ineq_decomp}.~We will denote $\wp_{\cS \times \cA \to \cS\uc{\ell},p}(z\up,\cdot)$ by $p\uc{d}_t(z\up,\cdot)$ in order to simplify the notation.~Note that both $\hat{p}\uc{d}_t(z,\cdot)$ and $p\uc{d}_t(z\up,\cdot)$ have the support $\tilde{\cS}_t(z)$, where $|\tilde{\cS}_t(z)| \leq d^{\frac{d_\cS}{2}} \diamc{\zeta}^{-d_\cS}$.~Let $\tilde{\cS}^+_t(z)$ denote the collection of those points in $\cS_t$ such that for any $s \in \tilde{\cS}^+_t(z)$, we have $\hat{p}\uc{d}_t(z,s) - p\uc{d}_t(z\up,s) > 0$.~So, we can write the following:
    \begin{align}
        \bP\br{\norm{\hat{p}\uc{d}_t(z,\cdot) - p\uc{d}_t(z\up,\cdot)}_1 \geq \iota} &= \bP\br{\max_{\cS\up \subset \tilde{\cS}^+_t(z)}{\sum_{s \in \cS\up}{\hat{p}\uc{d}_t(z,s) - p\uc{d}_t(z\up,s)}} \geq \frac{\iota}{2}} \notag\\
        &= \bP\br{\cup_{\cS\up \subset \tilde{\cS}^+_t(z)}{\flbr{\sum_{s \in \cS\up}{\hat{p}\uc{d}_t(z,s) - p\uc{d}_t(z\up,s) \geq \frac{\iota}{2}}}}}.\label{cineq:union}
    \end{align}
    Note that if $\cS\up \subset \tilde{\cS}^+_t(z)$, then $\tilde{\cS}_t(z) \setminus \cS\up \not\subset \tilde{\cS}^+_t(z)$. Hence the number of subsets of $\tilde{\cS}^+_t(z)$ is at most $2^{|\tilde{\cS}_t(z)|-1}$.~If $\bP\br{\sum_{s \in \cS\up}{\hat{p}\uc{d}_t(z,s) - p\uc{d}_t(z\up,s) \geq \frac{\iota}{2}}} \leq b_\iota,~\forall \cS\up \subset \tilde{\cS}^+_t(z)$, then by an application of union bound in~\eqref{cineq:union}, we obtain that the following must hold,
    \begin{align}
        \bP\br{\norm{\hat{p}\uc{d}_t(z,\cdot) - p\uc{d}_t(z\up,\cdot)}_1 \geq \iota} \leq 2^{|\tilde{\cS}_t(z)|-1} b_\iota.\label{ineq:l1byunion}
    \end{align}
    Consider a fixed $\xi \subseteq \cS$.~Define the following random processes,
    \begin{align}
        v_i(z) &:= \ind{(s_i, a_i) \in \zeta_i},\\
        v_i(z,\xi) &:= \ind{(s_i, a_i, s_{i+1}) \in \zeta_i \times \xi},\\
        w_i(z,\xi) &:= v_i(z,\xi) - p(s_i,a_i,\xi) v_i(z),
    \end{align}
    where $i = 0, 1, \ldots, T-1$.~Let $\cS\up \subset S^+_t$ and $\xi = \cup_{s \in \cS\up}{q\inv(s)}$. Then we have,
    \begin{align}
        \sum_{s \in \cS\up}{\hat{p}\uc{d}_t(z,s) - p\uc{d}_t(z\up,s)} &= \frac{N_t\br{\zeta, \xi}}{N_t\br{\zeta}} - p(z\up,\xi) \notag\\
        &= \frac{N_t\br{\zeta, \xi} - p(z\up,\xi) N_t\br{\zeta}}{N_t\br{\zeta}} \notag\\
        &\leq \frac{1}{N_t\br{\zeta}}\br{\sum_{i = 0}^{t - 1}{w_i(z,\xi)}} + \frac{L_p}{2 N_t\br{\zeta}} \sum_{i=0}^{N_t(\zeta)}{\diamc{\zeta_{t_i}}}\notag\\
        &\leq \frac{1}{N_t\br{\zeta}}\br{\sum_{i = 0}^{t - 1}{w_i(z,\xi)}} + 1.5 L_p~ \diamc{\zeta}, \label{ineq:determ}
    \end{align}
    where the last step follows from Lemma~\ref{lem:avg_diam}.~Note that $\flbr{w_i(z,\zeta)}_{i \in [T-1]}$ is martingale difference sequence w.r.t. $\flbr{\cF_i}_{i \in [T-1]}$.~Moreover, $\abs{w_i(z,\zeta)} \leq 1$. Hence from Lemma~\ref{lem:ah_ineq} we have,
    \begin{align*}
        \bP\br{\flbr{\frac{\sum_{i=0}^{t-1}{w_i(z,\xi)}}{N_t\br{\zeta}} \geq \sqrt{\frac{2}{N_t(\zeta)} \log{\br{\frac{3}{\delta}}}}, N_t(\zeta) = N}} \leq \frac{\delta}{3}.
    \end{align*}
    Upon combining this with~\eqref{ineq:determ} we get,
    \begin{align*}
        \bP\br{\flbr{\sum_{s \in \cS\up}{\hat{p}\uc{d}_t(z,s) - p\uc{d}_t(z\up,s)} \geq \sqrt{\frac{2}{N_t(\zeta)} \log{\br{\frac{3}{\delta}}}} + 1.5 L_p~ \diamc{\zeta}, N_t(\zeta) = N}}  \leq \frac{\delta}{3}.
    \end{align*}
    Upon using~\eqref{ineq:l1byunion} in the above, and taking a union bound over all possible values of $N$, we obtain,
    \begin{align*}
        \bP\br{\flbr{\norm{\hat{p}\uc{d}_t(z,\cdot) - p\uc{d}_t(z\up,\cdot)}_1 \geq \sqrt{\frac{2 |\tilde{\cS}_t(z)|}{N_t(\zeta)} \log{\br{\frac{3T}{\delta}}}} + 3 L_p~ \diamc{\zeta}, N_t(\zeta) = N}} \leq \frac{\delta}{3}.
    \end{align*}
    Note that we do not have to take a union over all possible values of $\tilde{\cS}_t(z)$ because of the one-to-one correspondence between $N_t(\zeta)$ and $\tilde{\cS}_t(z)$.~Replacing $|\tilde{\cS}_t(z)|$ by its upper-bound $d^{\frac{d_\cS}{2}} \diamc{\zeta}^{-d_\cS}$, we have,
    \begin{align}
        \bP\br{\norm{\hat{p}\uc{d}_t(z,\cdot) - p\uc{d}_t(z\up,\cdot)}_1  \geq \diamc{\zeta}^{-\frac{d_\cS}{2}} \sqrt{\frac{2~d^{\frac{d_\cS}{2}} \log{\br{\frac{3 T}{\delta}}}}{N_t(\zeta)}} + 3 L_p~ \diamc{\zeta}} \leq \frac{\delta}{3}.
    \end{align}
    Let $\cN_1 := 2 d^\frac{d}{2} \br{\frac{T}{c_a \log{\br{T/\delta}}}}^\frac{d}{d_\cS + 2}$, which is the number of cells the \algo~can activate under all sample paths.~Upon taking union bound over all the cells that could possibly be activated in all possible sample paths at some $t$ and using the fact that $N_t(\zeta) \geq N_{\min}(\zeta)$, the above inequality yields that with a probability at least $1 - \frac{\delta}{3}$, the following holds,
    \begin{align}
        \norm{\hat{p}\uc{d}_t(z,\cdot) - p\uc{d}_t(z\up,\cdot)}_1 &\leq 3~\br{\frac{c_a \log{\br{\frac{T}{\delta}}}}{N_t(\zeta)}}^\frac{1}{d_\cS + 2} + 3 L_p~ \diamc{\zeta},\label{ineq:mu_z}
    \end{align}
    for every $z \in \zeta$, $\zeta \in \cP_t$, and $t \in \{0,1, \ldots, T-1\}$, where $c_a$ is a constant that satisfies
    \al{
        d^{\frac{d_\cS}{2}} \log{\br{\frac{3 T \cN_1}{\delta}}} \leq 4.5 c_a \log{\br{\frac{T}{\delta}}}. \label{def:ca}
    }
    After some algebraic manipulation, we obtain that it suffices to have,
    \nal{
        c_a = \frac{2 d^{\frac{d_\cS}{2}}}{9} \frac{\log{\br{6 d^\frac{d}{2}}}}{\log{\br{\frac{T}{\delta}}}} + \frac{d}{d_\cS+2} + 1.
    }
    The proof follows upon combining the upper-bounds of the first and the second terms of \eqref{eq:con_ineq_decomp}.
\end{proof}

\begin{remark}\label{rem:conc_ineq}
    See that $\cap_{t=0}^{T-1}\{\wp_{\cS_t \times \cA_t \to \cS_t, p}(\cdot, \cdot) \in \cC_t\} \subseteq \cG_1$, where $\cC_t$ is as defined in~\eqref{def:confball}. Hence, 
    \nal{
        \bP\br{\cap_{t=0}^{T-1}\{\wp_{\cS_t \times \cA_t \to \cS_t, p} \in \cC_t\}} \geq 1 - \frac{\delta}{3}.
    }
\end{remark}
\section{Properties of Extended Value Iteration~(EVI) and Extended Policy Evaluation~(EPE)}\label{app:prop_evi_epe}
We recall the definition of Extended MDP at time $t$ that was discussed in Section~\ref{sec:algo},
\begin{align*}
    \cM^+_t = \{(\cS_t, \cA_t,\tilde{p}, \tilde{r}_t) : \tilde{p} \in \cC_t\}, 
\end{align*}
where $\cS_t$ and $\cA_t$ are the discretized state and action space respectively, at time $t$, while $\tilde{r}$ is the discretized reward function with an additional bonus term. $\cC_t$ is a set of plausible discrete transition kernels.~Note that \algo~calls the \evi~subroutine~(Algorithm~\ref{algo:evi}) with a parameter $\gamma$ which specifies the desired accuracy; upon calling \evi~with accuracy parameter $\gamma$, it returns a policy that is $\gamma$-optimal for the extended MDP.~We begin with introducing some notation. For $\phi \in \Phi_t$, $J_{\cM^+_t}(\phi)$ denotes the value of the policy $\phi$ evaluated on the extended MDP $\cM^+_t$. To be precise, this is the optimal average reward when the control action for the extended MDP is chosen according to the policy $\phi$, and the kernel is chosen so as to maximize the average reward.~The next result is similar in spirit to~\citet[Theorem 7]{jaksch2010near}.

\begin{lemma}\label{lem:conv_evi}
    Fix a time $t \in \bN$.~Consider the extended MDP $\cM^+_t$ and the corresponding \evi~iterates:
    \begin{align}
        v_0(s) &= 0, \notag\\
        v_{n+1}(s) &= \max_{\substack{a \in \cA_t(s)\\ \te \in \cC_t}} \flbr{\tilde{r}_t(s,a) + \sum_{s\up \in \cS_t}{\te(s,a,s\up) v_n(s\up)}},~\forall s \in \cS_t, n \in \bN. \label{iter:evi}
    \end{align}
    Then,
    \begin{align*}
        \lim_{n \to \infty}{\br{v_{n+1}(s) - v_n(s)}} = J\ust_{\cM^{+}_t}.
    \end{align*}
    Moreover, whenever $\spn{v_{n+1} - v_{n}} \leq \gamma$, the policy that chooses greedy actions which are optimal w.r.t. $v_n$, is $\gamma$-optimal.
\end{lemma}
\begin{proof}
    Consider the $n$-th step of the \evi~iteration, and let the action $a_n(s)$ and the kernel $\te_n$ maximize the r.h.s. of \eqref{iter:evi}, i.e.,
    \nal{
        (a_n(s), \te_n) \in \underset{\substack{a \in \cA_t(s)\\ \te \in \cC_t}}{\arg\max} \flbr{\tilde{r}_t(s,a) + \sum_{s\up \in \cS_t}{\te(s,a,s\up) v_n(s\up)}}, \mbox{ for every } s \in \cS_t.
    }
    Let $s\ust \in \arg\max_{s \in \cS_t}{v_n}(s)$. Then, $\te_i(s,\cdot)$ has to be chosen from the set $\cC_t$ in such a manner that one assigns the maximum possible probability to a state in $s\ust$. Thus, we must have $\te_i(s, s\ust) \geq \min{\flbr{1, \frac{1}{2}\eta_t(q_t\inv(s,a_n(s)))}}$, where $q_t\inv(s,a_n(s))$ is the active cell at time $t$ that contains $(s,a_n(s))$. Since $\eta_t(q_t\inv(s,a_n(s))) > 0$ for all $s \in \cS_t$, it follows that $\te_i(s\ust, s\ust) > 0$. It is evident that the associated Markov chain is aperiodic. The proof then follows from \citet[Theorem $9.4.4$]{puterman2014markov}. The second claim follows from \citet[Theorem 8.5.6]{puterman2014markov}.
\end{proof}

The next result follows from the previous result. It proves the convergence of the \epe~algorithm~\eqref{algo:epe}, also derives the gap between the true value of a policy and that returned by the~\epe.
\begin{cor}\label{cor:conv_epe}
    Fix a time $t \in \bN$.~Recall the extended MDP $\cM^{d,+}_{t} = \{(\cS_t, \cA_t, \tilde{p}, d_t) : \tilde{p} \in \cC_t\}$, where
    \begin{align*}
        d_t(s,a) = \diamc{q\inv_t(s,a)},~\forall (s,a) \in \cS_t \times \cA_t,
    \end{align*}
    policy $\phi \in \Phi_t$ and the corresponding \epe~iterates:
    \begin{align}
        g^\phi_0(s) &= 0, \notag\\
        g^\phi_{n+1}(s) &= \max_{\te \in \cC_t} \flbr{d_t(s,\phi(s)) + \sum_{s\up \in \cS_t}{\te(s,\phi(s),s\up) g^\phi_n(s\up)}},~\forall s \in \cS_t, n \in \bN. \label{iter:epe}
    \end{align}
    Then
    \begin{align*}
        \lim_{n \to \infty}{\br{g^\phi_{n+1}(s) - g^\phi_n(s)}} = \pdiam{t}{\phi}.
    \end{align*}
    Moreover, when $\spn{g^\phi_{n+1} - g^\phi_{n}} \leq \gamma (g^\phi_{n+1}(s\lst) - g^\phi_{n}(s\lst))$, i.e. the stopping criteria is met, then $(g^\phi_{n+1}(s\lst) - g^\phi_{n}(s\lst))$ satisfies the following:
    \begin{align}
        \frac{\pdiam{t}{\phi}}{1 + \gamma} \leq (g^\phi_{n+1}(s\lst) - g^\phi_{n}(s\lst)) \leq \frac{\pdiam{t}{\phi}}{1 - \gamma}. \label{bd:err_epe}
    \end{align}
\end{cor}

\begin{proof}
    Similar to the proof of Lemma~\ref{lem:conv_evi}, one can show that the transition kernels which maximize the r.h.s. in every iteration of \epe~\eqref{iter:epe} are aperiodic. The convergence of \epe~ then follows from~\citet[Theorem $9.4.4$]{puterman2014markov}.~From \citet[Theorem 8.5.6]{puterman2014markov}, it follows that 
    \begin{align*}
        \abs{(g^\phi_{n+1}(s\lst) - g^\phi_{n}(s\lst)) - \pdiam{t}{\phi}} \leq \gamma~ (g^\phi_{n+1}(s\lst) - g^\phi_{n}(s\lst)),
    \end{align*}
    or
    \begin{align*}
        g^\phi_{n+1}(s\lst) - g^\phi_{n}(s\lst) \leq \frac{\pdiam{t}{\phi}}{1 - \gamma}, \mbox{ and},
        g^\phi_{n+1}(s\lst) - g^\phi_{n}(s\lst) \geq \frac{\pdiam{t}{\phi}}{1 + \gamma}.
    \end{align*}
    This concludes the proof.
\end{proof}

\begin{remark}[Upper and lower-bounds of episode duration]\label{rem:bdd_hk}
Let $d_k = \epe(\cM^{d,+}_{\tau_k}, \tilde{\phi_k}, \gamma, s\lst)$ be the value of the policy $\tilde{\phi}_k$ evaluated on $\cM^{d,+}_{\tau_k}$.~From Corollary~\ref{cor:conv_epe} we have,
    \begin{align*}
        \frac{\pdiam{\tau_k}{\tilde{\phi}_k}}{1 + \gamma} \leq d_k \leq \frac{\pdiam{\tau_k}{\tilde{\phi}_k}}{1 - \gamma}.
    \end{align*}
    As $H_k = \frac{C_H \log{\br{\frac{T}{\delta}}}}{d_k^{2(d_\cS + 1)}}$, we conclude that
    \begin{align}
        \frac{C_H (1 - \gamma)^{2(d_\cS + 1)} \log{\br{\frac{T}{\delta}}}}{\pdiam{\tau_k}{\tilde{\phi}_k}^{2(d_\cS + 1)}} \leq H_k \leq \frac{C_H (1 + \gamma)^{2(d_\cS + 1)} \log{\br{\frac{T}{\delta}}}}{\pdiam{\tau_k}{\tilde{\phi}_k}^{2(d_\cS + 1)}}.\label{bdd:hk}
    \end{align}
\end{remark}
\section{Simulation Experiments}\label{app:sim}
We perform simulations on the following environments.
\begin{enumerate}
    \item \texttt{Continuous RiverSwim}: This environment models an agent who is swimming in a river~\citep{strehl2008analysis}.~Though the original MDP is discrete, we use a continuous version of it.~The state denotes the location of the agent in the river in a single dimension, and the action captures the movement of the agent. The state and action spaces are $[0,6]$ and $[0,1]$, respectively. The state of the system evolves as follows:
    \begin{align*}
        s_{t+1} =
        \begin{cases}
             \min\{\max\{0, s_t - \frac{1}{2}(1 + \frac{w_t}{2})\}, 6\} &\mbox{ w.p. } \frac{2(1-a_t)}{5}\\
             s_t &\mbox{ w.p. } 0.2\\
             \min\{\max\{0, s_t + \frac{1}{2}(1 + \frac{w_t}{2})\}, 6\} &\mbox{ w.p. } \frac{2(1+a_t)}{5},
        \end{cases}
    \end{align*}
    where $\{w_t\}$ is a $0$-mean i.i.d. Gaussian random sequence.~The reward function is given by 
    \nal{
r(s,a) = 0.005(((s-6)/6)^4 + ((a-1)/2)^4) + 0.5((s/6)^4 + ((a+1)/2)^4).
    }
    \item \texttt{Truncated LQ System}: The state of an LQ~\citep{abbasi2011regret} system evolves as follows: \nal{s_{t+1} = A s_t + B a_t + w_t,} where $A, B$ are matrices of appropriate dimensions, and $w_t$ is i.i.d. Gaussian noise. The reward at time $t$ is $- s_t^\top P s_t - a_t^\top Q a_t$.~We clip the state vector since our framework allows only compact state-action spaces. More specifically, we ensure that the state value for each coordinate lies within the interval $[c_{\ell},c_u]$, and restrict the action space to be $[-1,1]^{d_\cA}$. Hence, the $i$-th coordinate of the state process evolves as \nal{s_{t+1}(i) = \max{\{\min{\{(A s_t + B a_t + w_t)(i), c_u\}}, c_\ell\}}.} We have used the following two sets of system parameters:
    \begin{enumerate}
        \item \texttt{Truncated LQ-$1$}: 
            \begin{align*}
                A = \begin{bmatrix}
                    -0.2 & -0.07\\
                    0.6 & 0.07
                \end{bmatrix}, \quad
                &B = \begin{bmatrix}
                    0.07 & 0.09\\
                    -0.03 & -0.1
                \end{bmatrix},
            \end{align*}
            $P = 0.4~I_{2}$\footnote{$I_n$ denotes identity matrix of size $n \times n$.}, $Q = 0.6~I_{2}$ and mean and standard deviation of $w_t$ are $0$ and $0.05$, respectively.~We consider $c_u = -c_\ell = 4$.
        \item \texttt{Truncated LQ-$2$}: 
            \begin{align*}
                A = \begin{bmatrix}
                    -0.2 & -0.07\\
                    0.6 & 0.07
                \end{bmatrix}, \quad
                &B = \begin{bmatrix}
                    0.1 & -0.01 & 0.12 & 0.08\\
                    0.02 & -0.1 & 0.3 & 0.001
                \end{bmatrix}.
            \end{align*}
            Values of $P$, $Q$, $c_u$, $c_\ell$ and mean and standard deviation of $w_t$ are the same as \texttt{Truncated LQ-$1$}.
    \end{enumerate}
        
    \item \texttt{Non-linear System}: We consider a non-linear system~\citep{kakade2020information} where the state evolves as \nal{s_{t+1}(i) = \max{\{\min{\{(A f(s_t) + B g(a_t) + w_t)(i), c_u\}}, c_\ell\},}} where $f$ and $g$ are non-linear functions, $A, B$ are matrices of appropriate dimensions, and $w_t$ is noise sequence. This system can be viewed as a generalization of the LQ control system in which the dynamics are linear in the feature vectors corresponding to state-action values. The feature maps $f(\cdot),g(\cdot)$ can be non-linear functions.~The reward function is a function of the state and the actions. We have set the values for the matrices $A, B$, $P$, $Q$, $c_u$ and $c_\ell$ to be the same as that of \texttt{Truncated LQ-$1$}. We set 
    \begin{align*}
        f(s)(i) = 0.5 s(i) + 0.5 s(i)^2,~\mbox{for } i \in \flbr{1, 2}, \mbox{ and }
        g(a) = a^2,
    \end{align*}
    where $v(i)$ denotes the $i$-th element of vector $v$. Similar to the \texttt{LQ system}, we consider the action space to be $[-1,1]^{d_\cA}$.
\end{enumerate}

\subsection{Choosing Hyperparameters}
Since $L_r$~(Assumption~\ref{assum:lip}), $c_a$~\eqref{def:ca}, $C_\eta$, $C_H$~\eqref{def:CH} may not be known, we instead provide their estimates\slash appropriate upper-bounds to~\algo~in lieu of these parameters.~Our theoretical upper-bounds on regret continue to hold, we simply replace these parameters with the chosen upper-bounds.~In addition to these~\algo~we pass $\delta$ and $\gamma$ as hyperparameters to~\algo.~A brief description of these quantities are as follows:
\begin{enumerate}
    \item $L_r$: We assume the knowledge of an upper-bound on $L_r$, the Lipschitz constant for the reward function~(Assumption~\ref{assum:lip}).
    \item $c_a$: \algo~activates a cell $\zeta$ if $N_t(\zeta) \geq \frac{c_a \log{\br{\frac{T}{\delta}}}}{\diamc{\zeta}^{d_\cS+2}}$~\eqref{Nmin}, and deactivates $\zeta$ if $N_t(\zeta) \geq \frac{c_a 2^{d_\cS+2} \log{\br{\frac{T}{\delta}}}}{\diamc{\zeta}^{d_\cS+2}}$~\eqref{Nmax}.
     \item $C_\eta$: Recall from Section~\ref{sec:algo} that if $\zeta$ is an active cell at time $t$, then its confidence radius $\eta_t(\zeta)$ satisfies $\eta_t(\zeta) \leq C_\eta ~\diamc{\zeta}$, where $C_\eta = 3(1 + L_p) + C_p$.~In order to avoid computing $\eta_t(\zeta)$, we use $C_\eta~ \diamc{\zeta}$ as a substitute for  $\eta_t(\zeta)$, and choose $C_\eta$ as a hyperparameter for ZoRL
    \item $C_H$: $C_H$ is the multiplicative constant associated with the episode duration that satisfies \eqref{def:CH}.
    \item $\delta$: $\delta \in (0,1)$ is the probability parameter.
    \item $\gamma$: $\gamma > 0$ is the accuracy parameter for \epe~subroutine that is used by \algo~in order to compute the proxy diameter of the chosen policy in an episode.
\end{enumerate}
The values of the following three hyperparameters are kept unchanged across four experiments: $L_r = 0.001$, $\delta = 0.1$ and $\gamma = 0.05$. Values of the rest of the parameters are reported in Table~\ref{tab:hyp_param}.

\begin{table}
    \centering
    \begin{tabular}{|c|c|c|c|}
        \hline
        Experiments & $C_a$ & $C_\eta$ & $C_H$\\
        \hline
        Truncated LQ-$1$ & 0.2 & 1 & 0.1\\
        \hline
        Truncated LQ-$2$ & 0.1 & 1 & 0.001\\
        \hline
        Continuous RiverSwim & 0.1 & 1 & 0.001\\
        \hline
        Non-linear System & 1 & 5 & 0.1\\
        \hline
    \end{tabular}
    \caption{ZoRL hyper-parameters.}
    \label{tab:hyp_param}
\end{table}

\subsection{Comparison with PZRL-MF and PZRL-MB}
\begin{figure}[ht]
    \centering
    \begin{subfigure}[b]{0.49\textwidth}
        \centering
        \includegraphics[width=\textwidth]{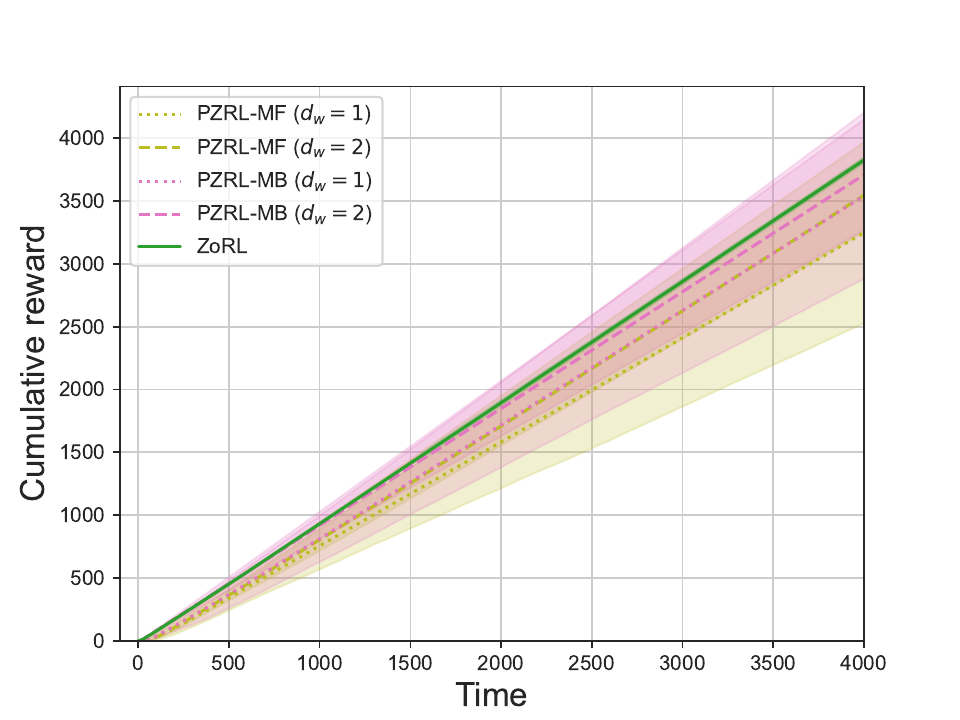}
        \caption{Continuous RiverSwim}
        \label{fig:rswim}
    \end{subfigure}
    \hfill
    \begin{subfigure}[b]{0.49\textwidth}
        \centering
        \includegraphics[width=\textwidth]{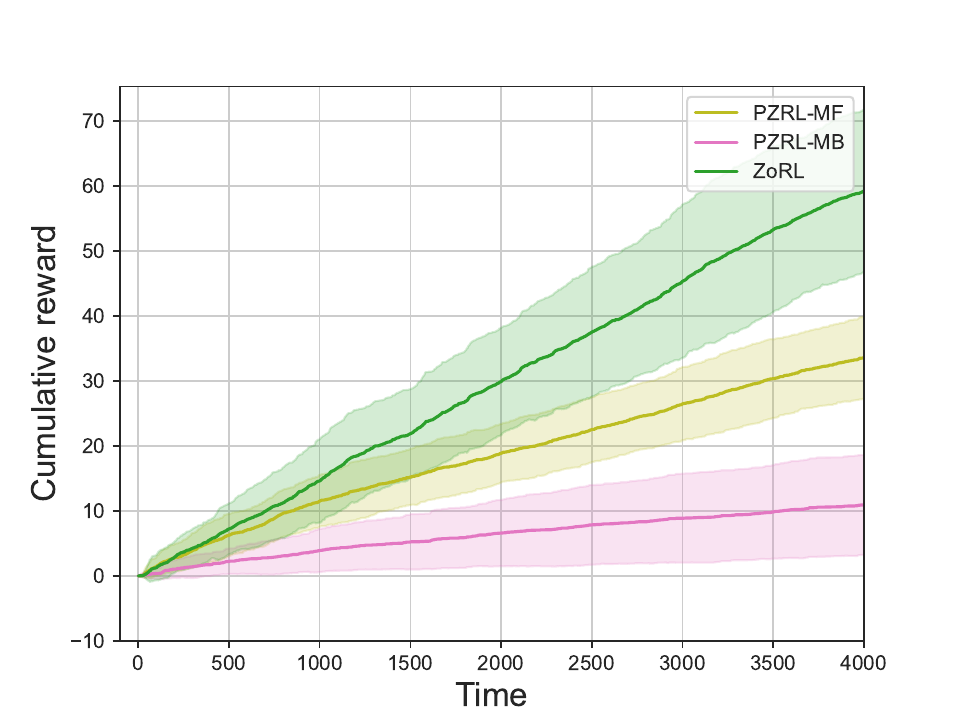} 
        \caption{Continuous GridWorld}         
        \label{fig:gworld}
    \end{subfigure}
    \caption{Comparison with PZRL-MF and PZRL-MB.}
    \label{fig:pzrl_compare}
\end{figure}
\citet{kar2024policy} allows the agent to play policies from a parametric class.~The latest version of~\citet{kar2024policy} proposes two new algorithms PZRL-MF and PZRL-MB\footnote{These replace the PZRL-H algorithm, that has been proposed in an earlier version of the same paper.}. Due to paucity of time, we could not compare PZRL-MF and PZRL-MB with~\algo~on the four environments discussed in Section~\ref{sec:sim}. However, here we compare PZRL-MB and PZRL-MF with \algo~on the following two environments: (i) \texttt{Continuous RiverSwim} and (ii) \texttt{Continuous GridWorld}. 

\texttt{Continuous RiverSwim:}~This environment has been discussed above. We use the following two policy parameterization schemes for PZRL-MF and PZRL-MB.
\begin{enumerate}
    \item $\phi(s;w) = w \cdot s$, $w \in [-1,1]$.
    \item $\phi(s;w) = w(1) + w(2)s^2$, $w = (w(1),w(2)) \in [-1,1]^2$.
\end{enumerate}

\texttt{Continuous GridWorld}: In GridWorld environment~\citep{sutton2018reinforcement}, the agent moves around a compact space, and the space contains a designated reward-yielding region such that the agent earns a reward of $1$ whenever it stays inside the reward-yielding region, and earns no reward otherwise.~We design a continuous version of the same environment; the reward-yielding region is taken to be a circle of radius $0.1$ units whose center is $[0.8,0.8]$. The state space is $[0,1]^2$ and the action space is $[0,2\pi]$. The state of the system evolves as follows:
\begin{align*}
    y_{t+1} &= s_t + \beta \begin{bmatrix} \cos{a_t}\\ \sin{a_t} \end{bmatrix} + w_t, \mbox{ and}\\
    s_t(i) &= (0\vee y_t(i)) \wedge 1, \mbox{ for } i=1,2,~\forall t \in \{0\}\cup \bN,
\end{align*}
where $w_t$ is a zero-mean i.i.d. Gaussian noise, and $\beta > 0$ is the step-size.~The standard deviation of $w_t$ is set to $0.1$, and we use a step size $\beta = 0.2$. For this environment, we parametrize the policies as follows: $\phi(s;w) = w(0) + s(0) w(1) + s(1) w(2)$, where $w \in [0,1]^3$ and $s \in [0,1]^2$.

We plot the cumulative rewards incurred by PZRL-MF, PZRL-MB, and \algo, averaged over $50$ runs for both the systems in Figure~\ref{fig:pzrl_compare}.

\textbf{Computing resources.} We have conducted experiments on a $11$-th Gen Intel Core-i7, $2.5$GHz CPU processor with $16$GB RAM using Python-$3$ and PyTorch library.
\section{Auxiliary Results}\label{app:aux_res}

In this section, we derive some useful properties of the algorithm that are used in the proof of regret upper-bound.~The first lemma shows that for any active cell $\zeta$ at time $t$, the quantity $\frac{1}{N_t(\zeta)}{\sum_{i=1}^{N_t(\zeta)}{\diamc{\zeta_{t_i}}}}$ is bounded above by $3~\diamc{\zeta}$.~We use this in concentration inequality for the transition kernel estimate.

\begin{lemma}\label{lem:avg_diam}
    For all $t \in [T-1]$ and $\zeta \in \cP_t$, let $t_i$ denote the time instance when $\zeta$ or any of its ancestor was visited by \algo~for the $i$-th time. Then
    \begin{align*}
        \frac{1}{N_t(\zeta)} \sum_{i=1}^{N_t(\zeta)}{\diamc{\zeta_{t_i}}} \leq 3~ \diamc{\zeta}.
    \end{align*}
\end{lemma}
\begin{proof}
    By the activation rule~\eqref{def:activationrule}, a cell $\zeta\up$ can be played at most $N_{\max}(\zeta\up) - N_{\min}(\zeta\up) = \tilde{c}_a 2^{2\ell(\zeta\up)} + \frac{\tilde{c}_a}{3} \ind{\zeta\up = \cS \times \cA}$ times while being active, where $\tilde{c}_a = 3 c_a d\inv \log{\br{\frac{T}{\eps \delta}}}~ \eps^{-d_\cS}$. We can write,
    \begin{align*}
        \frac{1}{N_t(\zeta)} \sum_{i=1}^{N_t(\zeta)}{\diamc{\zeta_{t_i}}} &= \frac{1}{N_t(\zeta)} \sum_{i=1}^{N_{\min}(\zeta)}{\diamc{\zeta_{t_i}}} + \frac{1}{N_t(\zeta)} \sum_{i=N_{\min}(\zeta)+1}^{N_t(\zeta)}{\diamc{\zeta_{t_i}}}\\
        &= \frac{\tilde{c}_a \sqrt{d}}{3 N_t(\zeta)} + \frac{\tilde{c}_a \sqrt{d}}{N_t(\zeta)} \sum_{\ell = 0}^{\ell(\zeta) - 1}{2^\ell} +  \frac{N_t(\zeta) - N_{\min}(\zeta) - 1}{N_t(\zeta)} \diamc{\zeta} \\
        &< \frac{\tilde{c}_a \sqrt{d}}{N_t(\zeta)} 2^{\ell(\zeta)} + \frac{N_t(\zeta) - N_{\min}(\zeta) - 1}{N_t(\zeta)} \diamc{\zeta} \\
        &= \frac{3 N_{\min}(\zeta)}{N_t(\zeta)} \diamc{\zeta} + \frac{N_t(\zeta) - N_{\min}(\zeta) - 1}{N_t(\zeta)} \diamc{\zeta}\\
        &=\frac{(N_t(\zeta) + 2 N_{\min}(\zeta) - 1)~ \diamc{\zeta}}{N_t(\zeta)}\\
        &\leq 3~ \diamc{\zeta},
    \end{align*}
    where the last step is due to the fact that $N_{\min}(\zeta) \leq N_t(\zeta)$.
\end{proof}
Next, we show that under Assumption~\ref{assum:bdd_der}, the total variation norm between $\bar{\wp}_{\cS \times \cA \to \cS_t,p}(z,\cdot)$ and $\bar{\wp}_{\cS \times \cA \to \cS\uc{\ell},p}(z,\cdot)$ is bounded above by the discretization width of the partition $\cQ\uc{\ell}$.~We use this result in Lemma~\ref{lem:conc_ineq}.
\begin{lemma}\label{lem:disc_dist}
    Let us fix any state-action pair $z$ and time $t$. Let $\ell = \ell(q_t\inv(z))$. Recall distributions $\bar{\wp}_{\cS \times \cA \to \cS_t,p}(z,\cdot)$ and $\bar{\wp}_{\cS \times \cA \to \cS\uc{\ell},p}(z,\cdot)$ from Lemma~\ref{lem:conc_ineq}. Under Assumption~\ref{assum:bdd_der}, we have that
    \begin{align*}
        \norm{\bar{\wp}_{\cS \times \cA \to \cS\uc{\ell},p}(z,\cdot) - \bar{\wp}_{\cS \times \cA \to \cS_t,p}(z,\cdot)}_{TV} \leq C_p \sqrt{d}~ 2^{-\ell}
    \end{align*}
    for every $z \in \cS \times \cA$.
\end{lemma}
\begin{proof}
    Recall that $\cS_t$ is the set of representative points of $\cQ_t$ and that $\cQ^{(\ell)}$ is a coarser partition of $\cS$ than $\cQ_t$. Let us fix $\xi \in \cQ\uc{\ell}$, and let us denote the Radon-Nikodym derivative of the distribution $p(z, \cdot)$ by $f$. Let $\bar{f} = p(z,\xi)/\lambda(\xi)$.~We have,
    \begin{align*}
        \sup_{B \subseteq \xi}{\abs{\bar{\wp}_{\cS \times \cA \to \cS\uc{\ell},p}(z, B) - p(z, B)}} &\leq \int_{\xi}{(f - \bar{f}) \ind{f \geq \bar{f}} d\lambda}\\
        &\leq \int_{\xi}{(\bar{f} + C_p \sqrt{d} \eps) \ind{f \geq \bar{f}} d\lambda} - \int_{\xi}{\bar{f} \ind{f \geq \bar{f}} d\lambda}\\
        &\leq C_p \sqrt{d} \eps \times \eps^{d_\cS},
    \end{align*}
    where $\eps = 2^{-\ell}$. Hence, by Assumption~\ref{assum:bdd_der}, we have that for every $z \in \cS \times \cA$ and for every $\xi \in \cQ\uc{\ell}$,
    \begin{align*}
        \sup_{B \subseteq \xi}{\abs{\bar{\wp}_{\cS \times \cA \to \cS\uc{\ell},p}(z,B) - \bar{\wp}_{\cS \times \cA \to \cS_t,p}(z, B)}} &\leq \sup_{B \subseteq \xi}{\abs{\bar{\wp}_{\cS \times \cA \to \cS\uc{\ell},p}(z,B) - p(z, B)}}\\
        &\leq C_p \sqrt{d} \eps \times \eps^{d_\cS}.
    \end{align*}
   As $\cQ\uc{\ell}$ is coarser than $\cQ$, it follows that
   \begin{align*}
        \norm{\bar{\wp}_{\cS \times \cA \to \cS\uc{\ell},p}(z,\cdot) - \bar{\wp}_{\cS \times \cA \to \cS_t,p}(z,\cdot)}_{TV} &\leq \sum_{\xi \in \cQ\uc{\ell}}{\sup_{B \subseteq \xi}{\abs{\bar{\wp}_{\cS \times \cA \to \cS\uc{\ell},p}(z,B) - p(z, B)}}}\\
        &\leq C_p \sqrt{d} \eps \times \eps^{d_\cS} \times \eps^{-d_\cS}\\
        &\leq C_p \sqrt{d} \eps.
    \end{align*}
    Hence, we have proven the claim.
\end{proof}
\section{Useful Results}
\label{app:use_res}
\subsection{Concentration Inequalities}
\begin{lemma}[Azuma-Hoeffding inequality]\label{lem:ah_ineq}
    Let $X_1, X_2, \ldots$ be a martingale difference sequence with $|X_i| \leq c \forall i$. Then for all $\epsilon > 0$ and $n \in \bN$,
    \begin{align}
        \bP\left\{\sum_{i = 1}^{n}{X_i} \geq \epsilon\right\} \leq e^{-\frac{\epsilon^2}{2nc^2}}
    \end{align}
\end{lemma}
The following inequality is Proposition $A.6.6$ of \cite{van1996weak}.
\begin{lemma}[Bretagnolle-Huber-Carol inequality]\label{lem:bhl_ineq}
    If the random vector $\paren{X_1, X_2, \ldots, X_n}$ is multinomially distributed with parameters $N$ and $\paren{p_1, p_2, \ldots, p_n}$, then for $\eps > 0$
    \begin{align}
        \bP\paren{\sum_{i=1}^{n}{\abs{X_i - N p_i}} \geq 2\sqrt{N} \eps} \leq 2^n e^{-2\eps^2}.
    \end{align}
    Alternatively, for $\delta > 0$
    \begin{align}
        \bP\paren{\sum_{i=1}^{n}{\abs{\frac{X_i}{N} - p_i}} < \sqrt{\frac{2n}{N} \log{\br{\frac{2}{\delta^\frac{1}{n}}}}}} \geq 1 - \delta.
    \end{align}
\end{lemma}

The following is essentially Theorem~1 of~\cite{abbasi2011improved}.
\begin{thm}[Self-Normalized Tail Inequality for Vector-Valued Martingales] \label{thm:self_norm}
    Let $\{\cF_t\}_{t=0}^{\infty}$ be a filtration. Let $\{\eta_t\}_{t=1}^{\infty}$ be a real-valued stochastic process such that $\eta_t$ is $\cF_t$ measurable and $\eta_t$ is conditionally $R$ sub-Gaussian for some $R>0$, i.e., 
    \begin{align*}
        \bE\left[ \exp(\lambda \eta_t) | \cF_{t-1}  \right] \le \exp\left( \lambda^2 R^2 \slash 2  \right), \forall \lambda\in \bR.
    \end{align*}
    Let $\{X_t\}_{t=1}^{\infty}$ be an $\bR^{d}$ valued stochastic process such that $X_t$ is $\cF_{t-1}$ measurable. Assume that $V$ is a $d\times d$ positive definite matrix. For any $t\ge 0$, define
    \begin{align*}
        \bar{V}_t := V + \sum_{s=1}^{t} X_s X^\top_s,
    \end{align*}
    and
    \begin{align*}
        S_t := \sum_{s=1}^{t} \eta_s X_s.
    \end{align*}
    Then, for any $\delta>0$, with a probability at least $1-\delta$, for all $t\ge 0$,
    \begin{align*}
        \|S_t\|^{2}_{\bar{V}^{-1}_t} \le 2 R^{2} \log{\br{\frac{\det(\bar{V}_t)^{1\slash 2} \det(V)^{-1\slash 2}}{\delta}}}.
    \end{align*}
\end{thm}

\begin{cor}[Self-Normalized Tail Inequality for Martingales] \label{cor:self_norm_vec}
	Let $\{\cF_i\}_{i=0}^{\infty}$ be a filtration. Let $\{\eta_i\}_{i=1}^{\infty}$ be a $\{\cF_i\}_{i=0}^{\infty}$ measurable stochastic process and $\eta_t$ is conditionaly $R$ sub-Gaussian for some $R > 0$. Let $\{ X_i \}_{i=1}^{\infty}$ be a $\{0,1\}$-valued $\cF_{i-1}$ measurable stochastic process.
	
	Then, for any $\delta>0$, with a probability at least $1-\delta$, for all $k \geq 0$,
	\begin{align*}
		\left|\sum_{i=1}^{k}{\eta_i X_i}\right| \leq R \sqrt{2 \left(1 + \sum_{i=1}^{k}{X_i}\right) \log{\br{\frac{1 + \sum_{i=1}^{k}{X_i}}{\delta}}}} .
	\end{align*}
\end{cor}
\begin{proof}
	Taking $V = 1$, we have that $\bar{V}_t = 1 + \sum_{s = 1}^{t}{X_s}$. The claim follows from Theorem~\ref{thm:self_norm}.
\end{proof}

\subsection{Other Useful Results}

\begin{lemma}\label{lem:bdd_epi_tool}
    Consider the following function $f(x)$ such that $0 < a_0 \leq \frac{a_1}{4}$,
    \begin{align*}
        f(x) = a_0 x - \sqrt{a_1 x} - 1.
    \end{align*}
    Then for all $x \geq 1.5\frac{a_1}{a_0^2}$, $f(x) \geq 0$.
\end{lemma}
\begin{proof}
    See that $f(x) \geq 0$ for all $x \geq \paren{\frac{\sqrt{a_1} + \sqrt{a_1 + 4 a_0}}{2 a_0}}^2$. Since $a_1 \leq 4 a_0$, we have that for all $x \geq 1.5 \frac{a_1}{a_0^2}$ $f(x) \geq 0$.
\end{proof}

\begin{lemma}\label{lem:bdd_dotdifLv}
    Let $\mu_1$ and $\mu_2$ be two probability measures on $Z$ and let $v$ be an $\bR$-valued bounded function on $Z$. Then, the following holds.
    \begin{align*}
        \abs{\int_{Z}{(\mu_1 - \mu_2)(z) v(z) dz}} \leq \frac{1}{2}\norm{\mu_1 - \mu_2}_{TV} \spn{v}.&
    \end{align*}
\end{lemma}
\begin{proof}
    Denote $\lm(\cdot) := \mu_1(\cdot) - \mu_2(\cdot)$. Now let $Z_+,Z_- \subset Z$ be such that $\lm(B) \geq 0$ for every $B \subseteq Z_+$ and $\lm(B) < 0$ for every $B \subseteq Z_-$. We have that
    \begin{align}
        \lm(Z) = \lm(Z_+) + \lm(Z_-) = 0.
    \end{align}
    Also,
    \begin{align}
        \lm(Z_+) - \lm(Z_-) = \norm{\mu_1 - \mu_2}_{TV}.
    \end{align}
    Combining the above two, we get that 
    \begin{align}
        \lm(Z_+) = \frac{1}{2}\norm{\mu_1 - \mu_2}_{TV}.
    \end{align}
    Now,
    \begin{align*}
        \abs{\int_Z{\lm(z) v(z) dz}} &= \abs{\int_{Z_+}{\lm(z) v(z) dz} + \int_{Z_-}{\lm(z) v(z) dz}}\\
        &\leq \abs{\lm(Z_+) \sup_{z \in Z}{v(z)} + \lm(Z_-) \inf_{z \in Z}{v(z)}} \\
        &= \abs{\lm(Z_+) \sup_{z \in Z}{v(z)} - \lm(Z_+) \inf_{z \in Z}{v(z)} + \lm(Z_+) \inf_{z \in Z}{v(z)} + \lm(Z_-) \inf_{z \in Z}{v(z)}} \\
        &= \lm(Z_+) \paren{\sup_{z \in Z}{v(z)} - \inf_{z \in Z}{v(z)}} \\
        &= \frac{1}{2}\norm{\mu_1 - \mu_2}_{TV} \spn{v}.
    \end{align*}
    Hence, we have proven the lemma.
\end{proof}

\begin{lemma}\label{lem:diff_kern_comp}
    Let $\te_1$ and $\te_2$ be two transition probability kernels of two Markov chains with common state space $\cS$. Let $\max_{s \in \cS}{\norm{\te_1(s,\cdot) - \te_2(s,\cdot)}_{TV}} \leq c$. Then,
    \begin{align*}
        \norm{\te\uc{m}_1(s,\cdot) - \te\uc{m}_2(s,\cdot)}_{TV} \leq m\cdot c,~\forall m \in \bN.
    \end{align*}
    where $\te\uc{m}_i$ is the $m$-step transition kernel of the Markov chain with one-step transition kernel $\te_i$ for $i = 1, 2$.
\end{lemma}
\begin{proof}
    We shall prove this using mathematical induction. The base case is given. Let us assume that,
    \begin{align*}
        \norm{\te\uc{i}_1(s,\cdot) - \te\uc{i}_2(s,\cdot)}_{TV} \leq i\cdot c,~\forall i = 1, 2, \ldots, m-1.
    \end{align*}
    See that
    \begin{align*}
        \norm{\te\uc{m}_1(s,\cdot) - \te\uc{m}_2(s,\cdot)}_{TV} &= \left\lVert \int_{\cS}{\te\uc{m-1}_1(s,s\up) \te_1(s\up,\cdot) ds\up} - \int_{\cS}{\te\uc{m-1}_2(s,s\up) \te_1(s\up,\cdot) ds\up}\right.\\
        &\qquad \left. + \int_{\cS}{\te\uc{m-1}_2(s,s\up) \te_1(s\up,\cdot) ds\up} - \int_{\cS}{\te\uc{m-1}_2(s,s\up) \te_2(s\up,\cdot) ds\up}\right\rVert_{TV} \\
        &\leq 2\sup_{A\in \cB_\cS}{\int_{\cS}{\br{\te\uc{m-1}_1(s,s\up) - \te\uc{m-1}_2(s,s\up)}\te_1(s\up,A) ds\up}} \\
        &\quad + 2\sup_{A\in \cB_\cS}{\int_{\cS}{\te\uc{m-1}_2(s,s\up) \br{\te_1(s\up,A) - \te_2(s\up,A)} ds\up}} \\
        &\leq \norm{\te\uc{m-1}_1(s,\cdot) - \te\uc{m-1}_2(s,\cdot)}_{TV} \sup_{A\in \cB_\cS}{\spn{\te_1(\cdot,A)}} \\
        &\quad + \int_{\cS}{\te\uc{m-1}_2(s,s\up) \norm{\te_1(s\up,\cdot) - \te_2(s\up,\cdot)}_{TV} ds\up} \\
        &\leq \norm{\te\uc{m-1}_1(s,\cdot) - \te\uc{m-1}_2(s,\cdot)}_{TV} + \max_{s\up \in \cS}{\norm{\te_1(s\up,\cdot) - \te_2(s\up,\cdot)}_{TV}},
    \end{align*}
    where the first inequality follows from triangle inequality and from the definition of total variation distance, the second inequality follows from Lemma~\ref{lem:bdd_dotdifLv} and by taking the supremum inside integration.~This concludes the proof of the lemma.
\end{proof}

\end{document}